%% file: ICML-LN-Camera.tex
\documentclass{article}

\usepackage{microtype}
\usepackage{graphicx}
\usepackage{subfigure}
\usepackage{booktabs} 
\usepackage[utf8]{inputenc} 
\usepackage[T1]{fontenc}    
\usepackage{hyperref}       
\usepackage{url}            
\usepackage{booktabs}       
\usepackage{amsfonts}       
\usepackage{nicefrac}       
\usepackage{microtype}      
\usepackage{xcolor}         
\usepackage{natbib}

\usepackage{amsthm,amsmath,amssymb}
\usepackage{mathrsfs}
\usepackage{enumerate}
\usepackage{subfigure}
\usepackage{wrapfig}
\usepackage{float}
\usepackage{bbding}

\usepackage{multirow} 

\usepackage[misc]{ifsym} 

\newtheorem{theorem}{Theorem}
\newtheorem{lemma}{Lemma}
\newtheorem{definition}{Definition}
\newtheorem{corollary}{Corollary}

\def\ie{\emph{i.e.}}

\def\SM{Appendix}

\newtheorem{proposition}{Proposition}

\newcommand{\REVISE}[1]{#1}

\usepackage{import}
\subimport{tool/}{symbol.tex}
\usepackage{tikz}
\usetikzlibrary{graphs}
\usetikzlibrary{positioning}

\usepackage[accepted]{icml2024}

\icmltitlerunning{On the Nonlinearity of Layer Normalization}
\begin{document}
	
	\twocolumn[
	\icmltitle{On the Nonlinearity of Layer Normalization}
	
	
	\icmlsetsymbol{equal}{*}
	
	\begin{icmlauthorlist}
		\icmlauthor{Yunhao Ni}{buaa}
		\icmlauthor{Yuxin Guo}{buaa}
        \icmlauthor{Junlong Jia}{buaa}
        \icmlauthor{Lei Huang*}{buaa}
	\end{icmlauthorlist}
	
	\icmlaffiliation{buaa}{SKLCCSE, Institute of Artificial Intelligence,  Beihang University, Beijing, China}
	
	\icmlcorrespondingauthor{Lei Huang}{huangleiAI@buaa.edu.cn}
	
	\icmlkeywords{Machine Learning, ICML}
	
	\vskip 0.3in
	]
	
	
	
	\printAffiliationsAndNotice{}  
		\begin{abstract}
	Layer normalization (LN) is a ubiquitous technique in deep learning but our theoretical understanding to it remains elusive. This paper investigates a new theoretical direction for LN, regarding to its nonlinearity and representation capacity. We investigate the representation capacity of a network with layerwise composition of linear and LN transformations, referred to as LN-Net. We theoretically show that, given $m$ samples with any label assignment, an LN-Net with only 3 neurons in each layer and $O(m)$ LN layers can correctly classify them. We further show the lower bound of the VC dimension of
    an LN-Net. The nonlinearity of LN can be amplified by group partition, which is also theoretically demonstrated with mild assumption and empirically supported by our experiments. Based on our analyses, we consider to design neural architecture by exploiting and amplifying the nonlinearity of LN, and the effectiveness is supported by our experiments. 
	\end{abstract}
	\vspace{-0.3in}
		\section{Introduction}
	\label{introduction}
	 Layer normalization (LN)~\cite{2016_LN_Ba} is a ubiquitous technique in deep learning, enabling varies neural networks to train effectively. It was initially proposed to address the train-inference inconsistency problem of Batch Normalization (BN)~\cite{2015_ICML_Ioffe} applied in the recurrent neural networks for Natural Language Processing (NLP) tasks. It then became the key component of Transformer~\cite{2017_NIPS_Vaswani} and its variants~\cite{2019_ACL_Dai,2020_ICML_Xiong,2021_ICLR_Dosovitskiy},  spreading from NLP~\cite{Radford2018GPT1,2019_ACL_Devlin,2020_JMLR_Raffel} to Computer Vision (CV)~\cite{2021_ICLR_Dosovitskiy,2020_ECCV_Carion,2022_CVPR_Cheng} communities.
	  LN has got its firm position~\cite{2023_TPAMI_Huang} in the evolution of neural architectures and is currently a basic layer in almost all the foundation models~\cite{Brown2020GPT3,2022_NIPS_Flamingo,2023_ICCV_Kirillov}.  
	 
	 While LN is extensively used in practice, our theoretical understanding to it remains elusive.  
	 One main theoretical work for LN is its scale-invariant property, which is initially discussed in~\cite{2016_LN_Ba} to illustrate its ability in stabilizing training and is further extended in~\cite{2018_NeurIPS_Hoffer,2019_ICLR_Arora,2020_ICLR_Li} to consider its potential affects in optimization dynamics. 
	 Different from the previous work focusing on theoretical analyses of LN from the perspective of optimization, this paper investigates a new theoretical direction for LN, regarding to its nonlinearity and representation capacity. 

       We mathematically demonstrate that LN is a nonlinear transformation. We highlight that LN might be a nonlinear transformation by intuition, but there is no work, to our best knowledge, demonstrating it.  Our demonstration is based on the defined lower bound named LSSR (Definition~\ref{def:LSSR}). The LSSR will not be broken under any linear transformation by definition, but we show that a linear neural network combined with LN can break the LSSR. Therefore, LN has nonlinearity. We also show that an LN-Net, which is a layerwise composition of linear and LN transformations, has nonlinearity.
       

       One interesting question is that how powerful the nonlinear of an LN-Net is in theory. 
       We theoretically show that, given $m$ samples with any label assignment, an LN-Net with only 3 neurons in each layer and $O(m)$ LN layers can correctly classify them. We further show  the lower bound of the VC dimension of an LN-Net. In particular, given an LN-Net with width only 3 neurons in each layer and $L$ LN layers, its VC dimension is lower bounded by $L+2$. 
  These results show that LN-Net has great representation capacity in theory, implying the possibility that a network with linear and LN layer only can work well in practice.
 

       We further investigate how to amplify and exploit the nonlinearity of LN. We find that Group based LN (LN-G)---which divides neurons of a layer into groups and perform LN in each group in parallel---has stronger nonlinearity than the naive LN countpart. This is also theoretically demonstrated with mild assumption and empirically supported by our comprehensive experiments. 
       We also consider practical scenario, where we replace LN with LN-G on Transformer and ViT, since we believe the amplified nonlinearity can benefits the models. The preliminary results show the potentiality of this design in neural architecture.

 	\vspace{-0.1in}
 \section{Preliminary and Notation}
 \label{sec_relatedWork}
 We use a lowercase letter $x \in \R$ to denote a scalar,  boldface lowercase letter $\rvx \in \R^{d}$  for a vector and boldface uppercase letter for a matrix $\mX \in \R^{d \times m}$, where $\R$ is the set of real-valued numbers, and $d, m$ are positive integers.
 \paragraph{Neural Network.} Given the input $\rvx$, a classical neural network $f_{\theta}(\rvx)$ is typically represented as a layer-wise linear\footnote{We follow the convention in deep learning community, and do not differentiate between the linear and affine transformation.} and nonlinear transformation:
 	\begin{eqnarray}
 		\label{eqn:Linear}
 		\vh^{(l)}&= &\mW^{(l)} \rvx^{(l-1)}+\vb^{(l)},   \\
 		\label{eqn:Nonlinear}
 		\rvx^{(l)}&=& \phi(\vh^{(l)}), ~~~ l=1,..., L,
 	\end{eqnarray}
where $\mathbf{\theta}=\{ (\mW^{(l)}, \vb^{(l)}), l=1,\cdots,L \}$ are learnable parameters, $\rvx^{(0)} = \rvx$, $\mW^{(l)} \in \R^{d_{l} \times d_{l-1}}$, $\vb^{(l)}\in \R^{d_{l}}$ and $d_{l}$ indicates the number of neurons in the $l$-th layer. We  set  $\rvx^{(L)}=\vh^{(L)} $ as the output of the network $f_{\theta}(\rvx)$ to simplify denotations.
 A neural network without nonlinear transformation $\phi(\cdot)$ (Eqn.~\ref{eqn:Nonlinear}) is referred to as a \textit{linear neural network}, which is still a linear transformation in native. 

 \paragraph{Layer Normalization.} Layer Normalization (LN) is an essential layer in modern deep neural networks mainly for stabilizing training. Given a single sample of layer input $\rvx=[x^{(1)}, x^{(2)},\cdots, x^{(d)}] \in \mathbb{R}^{d}$ with $d$ neurons in a neural network, LN standardizes $\rvx$ within the neurons as~\footnote{LN usually uses extra learnable scale and shift parameters~\cite{2015_ICML_Ioffe}, and we omit them for simplifying discussion as they are affine transformation in native}:
 	\begin{equation}
 		\label{eqn:LN}
 		\hat{x}^{(j)}=LN(x^{(j)})= \frac{x^{(j)} - \mu}{\sigma}, ~~j=1,2, \cdots, d,
 	\end{equation}
where $\mu=\frac{1}{d}  \sum\limits_{i=1}^{d}  x^{(j)}$ and $\sigma = \sqrt{\frac{1}{d}   \sum\limits_{i=1}^{d} (x^{(j)}-\mu)^2} $ are the  mean and variance for each sample, respectively. The standardization operation can be viewed as a combination of centering and scaling operations.  Centering projects $ \rvx $ onto the hyperplane $ \{ \rvx\in\R^d: x^{(1)} + \cdots + x^{(d)} = 0 \}$, by $ \tilde\rvx = (\mI - \frac1d \bm{1}_d \bm{1}_d^\top) \rvx $.
 Scaling projects $ \tilde\rvx $ onto the sphere $ \{ \rvx\in\R^d: [x^{(1)}]^2 + \cdots + [x^{(d)}]^2=d \} $, by $ \hat\rvx = \sqrt{d}\tilde\rvx/\|\tilde\rvx\|_2 $. We thus also call scaling as Spherical Projection (SP), from the geometric perspective.  \REVISE{Note that SP is the  only operation for normalization in RMSNorm~\cite{2019_NIPS_Zhang}.}
    

    
\paragraph{Sum of Squares.} Sum of Squares (SS)~\cite{fisher1970statistical}  is a statistical concept that measures the variability or dispersion within a set of data.
Denote $ m $ samples from class $ c $ as $ \rvx_{c1}, \cdots, \rvx_{cm} \in \R^d $, represented as a matrix $ \mX_c = [\rvx_{c1}, \cdots, \rvx_{cm}] $, then SS of $ \mX_c $ is defined as 
	\begin{equation}
		SS(\mX_c) = \sum_{i=1}^m \left\Vert \rvx_{ci} - \bar{\rvx}_{c} \right\Vert^2,
	\end{equation}
	where $ \bar{\rvx}_{c} = (\rvx_{c1} + \cdots + \rvx_{cm}) / m $.

 	\section{The Existence of Nonlinearity in LN}
    \vspace{-0.05in}
	
	In this section, we define Sum of Squares Ratio (SSR) and its linear invariant lower bound named LSSR. We then show that LN can break the boundary of SSR and plays a role in nonlinear representation. 

	\vspace{-0.05in}
	\subsection{Linear Invariant Lower Bound}

	We take binary classification for simplifying discussion. Let $ \mX_c = [\rvx_{c1}, \cdots, \rvx_{cm}]$ represents $ m $ samples\footnote{We use the same number ($m$) of samples  in each class for simplifying notation, and our subsequent definition and conclusion are also apply to different number for different classes.} in $ \R^d $ from the corresponding class $ c \in \{1,2\} $, and $[\mX_1,\mX_2] \in \R^{d\times 2m}$ represents all the samples together.
	
	\begin{definition}{(SSR.)}
		\label{def:SSR}
		Given $SS([\mX_{1}, \mX_{2}]) \neq 0$, the Sum of Squares Ratio (SSR) between $ \mX_{1} $ and $ \mX_{2} $ is defined as
		\begin{equation}
			SSR(\mX_{1},\mX_{2}) = \frac{SS(\mX_{1}) + SS(\mX_{2})}{SS([\mX_{1}, \mX_{2}])}. 
		\end{equation}
	\end{definition}
	 It is easy to demonstrate that $ SSR(\mX_1,\mX_2) \in [0,1] $. 
	 SSR can be an indicator to show how easy the samples in the Euclidean space from different classes can be separated. I.e., the smaller SSR is, the more easily $ \mX_1 $ and $ \mX_2 $ are to be separated with Eulcidean distance as a measurement in most cases. Based on SSR, we further define its lower bound under any linear transformation as follows.
		\begin{definition}{(LSSR.)}
        \label{def:LSSR}
	    The Linear SSR (LSSR) between $ \mX_1 $ and $ \mX_2 $ is defined as
		\begin{equation}
			LSSR(\mX_1,\mX_2) = \inf_{\varphi\in \sD_{\varphi}(d)} SSR(\varphi(\mX_1),\varphi(\mX_2)),
		\end{equation}
		where $ \sD_\varphi(d) $ is the set of all linear functions defined on $ \R^d $.
	\end{definition}
	
	By definition, LSSR is the lower bound of SSR under any linear transformation. 	LSSR can be an indicator to show how easy the samples from different classes can be linearly separated. We provide illustrative examples in \textit{\SM~\ref{section:LSSR-illu}} for details. 
    In the following proposition, we show a linear neural network can not break LSSR.
	
	\begin{proposition}
		\label{prop:invariance}
		Given $ \mX_1,\mX_2 \in \R^{d_0\times m} $ and a linear neural network represented as $ \tilde{\varphi} = \varphi_1 \circ \cdots \circ \varphi_L $, where  $ \varphi_l : \R^{d_{l-1}} \to \R^{d_{l}}, (l = 1, \cdots, L) $ are all linear transformations as shown in Eqn.~\ref{eqn:Linear}, we have that
		\begin{equation}
			SSR(\tilde{\varphi} (\mX_1), \tilde{\varphi} (\mX_2)) \ge LSSR(\mX_1,\mX_2).
		\end{equation}
	\end{proposition}
 
    Proposition \ref{prop:invariance} is easily proved by the definition of LSSR, since we have $ \tilde\varphi \in \sD_{\varphi}(d_0)$. Proposition \ref{prop:invariance} implies that the SSR will not break the lower bound if we use an arbitrary \textit{linear neural network} as a representation transformation over the samples. One interesting question is that whether a \textit{linear neural network} combined with LN can break the lower bound of SSR. If Yes, we can show that LN has nonlinearity.
    
	
	
	\subsection{Break the Lower Bound of SSR with LN}
	\label{section:decrease}
	Here, we focus on the \textit{linear neural network} combined with LN. To state more precisely, we denote LN-Net as follows. 
		\begin{definition}{(LN-Net.)}
		\label{def:LN-Net}
		The LN-Net $f_{\theta}(\rvx)$ is defined as layer-wise composition of linear and LN transformation:
	\begin{eqnarray}
	\vh^{(l)}&= &\mW^{(l)} \rvx^{(l-1)}+\vb^{(l)}, ~~ l=1,..., L,  \\ \nonumber
	\rvx^{(l)}&=& LN(\vh^{(l)}), ~~~ l=1,..., L-1,  \nonumber
	\end{eqnarray}
where $\mathbf{\theta}=\{ (\mW^{(l)}, \vb^{(l)}), l=1,...,L \}$ are learnable parameters, $\rvx^{(0)} = \rvx$ and $ LN(\cdot) $ denotes the LN operation. We set $\rvx^{(L)}=\vh^{(L)} $ as the output of the network $f_{\theta}(\rvx)$ to simplify denotations.
	\end{definition}

    We first provide a tractable method to calculate LSSR, stated by the following proposition.
    \begin{proposition}
		\label{prop:LSSR}
		Given $ \mX_1, \mX_2 \in \R^{d \times m} $, we denote $ \mM = \sum\limits_{c=1}^{2} \sum\limits_{i=1}^m (\rvx_{ci}-\bar{\rvx}_{c})(\rvx_{ci}-\bar{\rvx}_{c})^\top $, and $ \mN = \sum\limits_{c=1}^{2} \sum\limits_{i=1}^m (\rvx_{ci}-\bar{\rvx})(\rvx_{ci}-\bar{\rvx})^\top $, where $\bar\rvx = (\bar\rvx_1 + \bar\rvx_2)/2 $. Supposing that $ \mN $ is reversible, we have  
		\begin{equation}
			LSSR(\mX_1, \mX_2) = \lambda^*, 
		\end{equation}
		and correspondingly, 
		\begin{equation}
            \label{eqn:9}
			LSSR(\mX_1, \mX_2) = SSR((\vu^*)^\top\mX_1, (\vu^*)^\top\mX_2), 
		\end{equation}
	where  $ \lambda^* $ and  $\vu^* $ are the minimum eigenvalue and corresponding eigenvector of $ \mN^{-1} \mM $.
	\end{proposition}
    The proof of Proposition \ref{prop:LSSR} are shown in \textit{\SM~\ref{section:Prop2}}. 
    Based on Proposition \ref{prop:LSSR}, we further define $f_{SSR}(t)$ as
    \begin{equation}
    	f_{SSR}(t) = \begin{cases}
    		LSSR(\mX_1, \mX_2), &t=0, \\
    		SSR(\bar\psi(t;\mX_1), \bar\psi(t;\mX_2)), &t\ne0, 
    	\end{cases}
    \end{equation}
    where $ \bar\psi(t;\rvx_{ci}) =  \bm1^\top\bar\varphi(t;\rvx_{ci}) / \| \bar\varphi(t;\rvx_{ci}) \|_2 $, $ \bar\varphi(t;\rvx_{ci}) = [(\vu^*)^\top\rvx_{ci}t, 1]^\top $ and $ t\in \R $.
      We point out that $ f_{SSR}(t) $ is derivable at $ t=0 $, and $ f_{SSR}'(0) $ is only decided by $ \mX_1 $ and $ \mX_2 $, which is proved in \textit{\SM~\ref{section:LSSRbreak}}.

    Based on the definition of $ f_{SSR}(t) $, we show that LN-Net can decrease the LSSR as stated by the following theorem.
    	\begin{theorem}
		\label{thm:break}
		Let $ \psi = \varphi_1 \circ LN(\cdot) \circ \varphi_2 $, performing over the input $ \mX_1, \mX_2 \in \R^{d \times m} $. If $ f_{SSR}'(0) \ne 0 $ , we can always find suitable linear functions $ \varphi_1$ and $\varphi_2 $, such that
		\begin{equation}
			SSR(\psi(\mX_1), \psi(\mX_2)) < LSSR(\mX_1, \mX_2). 
		\end{equation}
	\end{theorem}	
    The proof of Theorem \ref{thm:break} requires complicated derivation. Please refer to \textit{\SM~\ref{section:LSSRbreak}} for details. Note that LN-Net is a more general form of $\psi$ in Theorem \ref{thm:break}, which implies that LN-Net can break the lower bound of SSR. 
 
    Based on Theorem \ref{thm:break}, we can obtain the following statement. We deduce that LN is a nonlinear transformation. 
    \begin{corollary}
	       LN is a nonlinear transformation.
	   \end{corollary}

  \begin{proof}
	We assume that $ LN(\cdot) $ is a linear transformation. We thus have LN-Net is also a linear transformation. Based on Proposition \ref{prop:invariance}, we have LN-Net can not break LSSR. This contradicts Theorem \ref{thm:break}. Therefore, $ LN(\cdot) $ must be a nonlinear transformation. 
\end{proof}

	\paragraph{Summary.} In this section, we mathematically show that LN is a nonlinear transformation, and LN-Net is a network with nonlinearity. One interesting question is that how powerful the nonlinearity of an LN-Net is in theory. We will discuss about it in the following section. 
	
	\section{Capacity of a Network with LN}
    \label{section:capacity}
	
	In this section, we apply LN-Net to classify $ m $ samples with any label assignment. To prove the existence of such LN-Net, we propose Projection Merge Algorithm (PMA) and Parallelization Breaking Algorithm (PBA) to help find the parameters of the LN-Net. 

	\subsection{LN for Xor Classification}
	\label{section:xor}

    To understand PMA intuitively, we use Spherical Projection (SP) rather than LN at the beginning. But we replace SP with LN and linear layers back in the end, according to the lemma as follows. 
	\begin{lemma}
		\label{lemma:equivalence}
		Denote $ LN(\cdot) $ as the LN operation in $ \R^d (d\ge3) $, and $ SP(\cdot) $ as the SP operation\footnote{If there are no special instructions, we denote SP projects the sample on to the unit circle, namely $ \rvx \mapsto \rvx/\|\rvx\|_2 $. } in $ \R^{d-1} $. We can find some linear transformations $ \hat \varphi_1 $ and $ \hat \varphi_2 $, such that 
		\begin{equation}
			SP(\cdot) = \hat \varphi_1 \circ LN(\cdot) \circ \hat \varphi_2. 
		\end{equation} 
	\end{lemma}
	
	The proof of Lemma \ref{lemma:equivalence} is shown in \textit{\SM~\ref{section:LSSRbreak}}. And we can easily obtain the following corollary. 
    \begin{corollary}
        \label{coro:SP is LN-Net}
        $ SP(\cdot) $ can be represented by an LN-Net. 
    \end{corollary}
    Taking xor classification as an example, we primarily show how we use LN-Net to classify linearly inseparable samples. 

    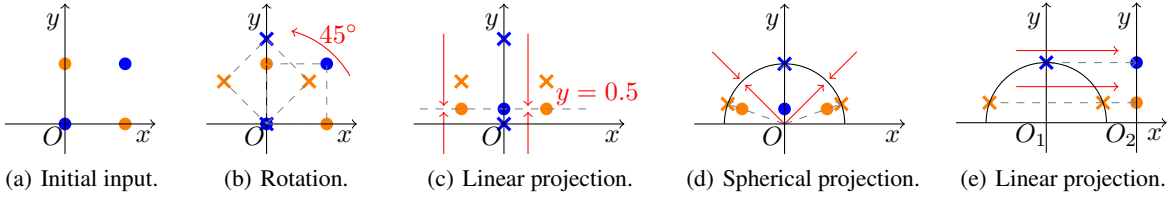
\begin{figure*}[t]
 \vspace{-0.1in}
		\centering
		\subfigure[Initial input.]{
			\label{fig:1a}
			\begin{tikzpicture}[scale=0.8]
				\fill[blue, shift={(0,0)}] (0,0) circle (3pt);
				\fill[orange, shift={(0,1)}] (0,0) circle (3pt);
				\fill[orange, shift={(1,0)}] (0,0) circle (3pt);
				\fill[blue, shift={(1,1)}] (0,0) circle (3pt);
				
				\draw [->, black] (-1,0) -- (1.5,0);
				\draw [->, black] (0,-0.5) -- (0,2);
                \node[] at (-0.2,1.7) {$y$};
                \node[] at (1.3,-0.2) {$x$};
                \node[] at (-0.2,-0.2) {$O$};
                
			\end{tikzpicture}
		}
		\quad
        \subfigure[Rotation.]{
			\label{fig:1b}
			\begin{tikzpicture}[scale=0.8]
				\fill[blue, shift={(0,0)}] (0,0) circle (3pt);
				\fill[orange, shift={(0,1)}] (0,0) circle (3pt);
				\fill[orange, shift={(1,0)}] (0,0) circle (3pt);
				\fill[blue, shift={(1,1)}] (0,0) circle (3pt);

                \draw[blue, shift={(0,0)}, very thick] (-3pt,-3pt) -- (3pt,3pt) (-3pt,3pt) -- (3pt,-3pt);
				\draw[orange, shift={(-0.707,0.707)}, very thick] (-3pt,-3pt) -- (3pt,3pt) (-3pt,3pt) -- (3pt,-3pt);
                \draw[orange, shift={(0.707,0.707)}, very thick] (-3pt,-3pt) -- (3pt,3pt) (-3pt,3pt) -- (3pt,-3pt);
                \draw[blue, shift={(0,1.414)}, very thick] ((-3pt,-3pt) -- (3pt,3pt) (-3pt,3pt) -- (3pt,-3pt);

                \draw [->, red] (1.386,0.8) arc [start angle=30,end angle=75, radius=1.6];

				\draw [->, black] (-1,0) -- (1.5,0);
				\draw [->, black] (0,-0.5) -- (0,2);
                \draw [dashed, gray, very thin] (0,0) -- (1,0) -- (1,1) -- (0,1) -- (0,0);
                \draw [dashed, gray, very thin] (0,0) -- (0.707,0.707) -- (0,1.414) -- (-0.707,0.707) -- (0,0);
                \node [red] at (1.2,1.5) {$ 45^\circ $};
                \node[] at (-0.2,1.7) {$y$};
                \node[] at (1.3,-0.2) {$x$};
                \node[] at (-0.2,-0.2) {$O$};
			\end{tikzpicture}
		}
		\quad
        \subfigure[Linear projection.]{
			\label{fig:1c}
			\begin{tikzpicture}[scale=0.8]
				\fill[blue, shift={(0,0.25)}] (0,0) circle (3pt);
				\fill[orange, shift={(-0.707,0)}] (0,0.25) circle (3pt);
				\fill[orange, shift={(0.707,0)}] (0,0.25) circle (3pt);

                \draw[blue, shift={(0,0)}, very thick] (-3pt,-3pt) -- (3pt,3pt) (-3pt,3pt) -- (3pt,-3pt);
				\draw[orange, shift={(-0.707,0.707)}, very thick] (-3pt,-3pt) -- (3pt,3pt) (-3pt,3pt) -- (3pt,-3pt);
                \draw[orange, shift={(0.707,0.707)}, very thick] (-3pt,-3pt) -- (3pt,3pt) (-3pt,3pt) -- (3pt,-3pt);
                \draw[blue, shift={(0,1.414)}, very thick] (-3pt,-3pt) -- (3pt,3pt) (-3pt,3pt) -- (3pt,-3pt);

                \draw [->, red] (-1,1.5) -- (-1,0.3);
                \draw [->, red] (-1,-0.5) -- (-1,0.2);
                \draw [gray, dashed] (-1.4,0.25) -- (1.4,0.25);
                \draw [->, red] (0.4,1.5) -- (0.4,0.3);
                \draw [->, red] (0.4,-0.5) -- (0.4,0.2);
				
				\draw [->, black] (-1.5,0) -- (2,0);
				\draw [->, black] (0,-0.5) -- (0,2);
                \node[] at (-0.2,1.7) {$y$};
                \node[] at (1.8,-0.2) {$x$};
                \node[] at (-0.2,-0.2) {$O$};
                \node[red] at (1.55,0.5) {$y=0.5$};
			\end{tikzpicture}
		}
		\quad
        \subfigure[Spherical projection.]{
			\label{fig:1d}
			\begin{tikzpicture}[scale=0.8]
				\fill[blue, shift={(0,0.25)}] (0,0) circle (3pt);
				\fill[orange, shift={(-0.707,0.25)}] (0,0) circle (3pt);
				\fill[orange, shift={(0.707,0.25)}] (0,0) circle (3pt);

                \draw[blue, shift={(0,1)}, very thick] (-3pt,-3pt) -- (3pt,3pt) (-3pt,3pt) -- (3pt,-3pt);
				\draw[orange, shift={(-0.943,0.333)}, very thick] (-3pt,-3pt) -- (-3pt,-3pt) -- (3pt,3pt) (-3pt,3pt) -- (3pt,-3pt);
                \draw[orange, shift={(0.943,0.333)}, very thick] (-3pt,-3pt) -- (3pt,3pt) (-3pt,3pt) -- (3pt,-3pt);

                \draw [] (1,0) arc [start angle=0,end angle=180, radius=1];
                \draw [->, red] (0,0) -- (0.65,0.65); 
                \draw [->, red] (1.2,1.2) -- (0.75,0.75);
                \draw [->, red] (0,0) -- (-0.65,0.65); 
                \draw [->, red] (-1.2,1.2) -- (-0.75,0.75);

                \draw [->, gray, dashed] (0,0) -- (0.943,0.333);
                \draw [->, gray, dashed] (0,0) -- (-0.943,0.333);
                \draw [->, gray, dashed] (0,0) -- (0,1);
				
				\draw [->, black] (-1.5,0) -- (2,0);
				\draw [->, black] (0,-0.5) -- (0,2);
                \node[] at (-0.2,1.7) {$y$};
                \node[] at (1.8,-0.2) {$x$};
                \node[] at (-0.2,-0.2) {$O$};
			\end{tikzpicture}
		}
        \quad
        \subfigure[Linear projection.]{
			\label{fig:1e}
			\begin{tikzpicture}[scale=0.8]
				\fill[blue, shift={(1.5,1)}] (0,0) circle (3pt);
				\fill[orange, shift={(1.5,0.333)}] (0,0) circle (3pt);

                \draw[blue, shift={(0,1)}, very thick] (-3pt,-3pt) -- (3pt,3pt) (-3pt,3pt) -- (3pt,-3pt);
				\draw[orange, shift={(-0.943,0.333)}, very thick] (-3pt,-3pt) -- (-3pt,-3pt) -- (3pt,3pt) (-3pt,3pt) -- (3pt,-3pt);
                \draw[orange, shift={(0.943,0.333)}, very thick] (-3pt,-3pt) -- (3pt,3pt) (-3pt,3pt) -- (3pt,-3pt);

                \draw [<-, dashed, gray] (1.5,0.333) -- (-0.943,0.333); 
                \draw [->, dashed, gray] (0,1) -- (1.5,1); 
                \draw [->, red] (-0.5,1.2) -- (1.2,1.2); 
                \draw [->, red] (-0.5,0.6) -- (1.2,0.6);
				
				\draw [->, black] (-1.5,0) -- (2,0);
				\draw [->, black] (0,-0.5) -- (0,2);
                \draw [->, black] (1.5,-0.5) -- (1.5,2);
                \draw [] (1,0) arc [start angle=0,end angle=180, radius=1];
                \node[] at (-0.2,1.7) {$y$};
                \node[] at (1.3,1.7) {$y$};
                \node[] at (1.8,-0.2) {$x$};
                \node[] at (-0.25,-0.2) {$O_1$};
                \node[] at (1.25,-0.2) {$O_2$};
			\end{tikzpicture}
		}
        \vspace{-0.15in}
		\caption{Solution to the Xor Classification. To begin with, we rotate them by $45^\circ$, as shown in Figure \ref{fig:1b}. Then we vertically project them onto $y=0.5$, as shown in Figure \ref{fig:1c}. Next, we spherically project them onto the circle $ x^2 + y^2 = 1 $, as shown in Figure \ref{fig:1d}. Finally, we \REVISE{horizontally} project them onto \REVISE{$x=0$}, as shown in Figure \ref{fig:1e}. Now we have classified the two classes.}
		\label{fig:1}
	\end{figure*}
	
	As shown in Figure \ref{fig:1a}, $ (0,0), (1,1) $ and $ (0,1), (1,0) $ belong to different classes. Obviously, the two classes are not linearly separable. We can classify them with SP and linear transformations only, please refer to the demonstration in Figure \ref{fig:1} for details.   
 
    By Lemma \ref{lemma:equivalence}, replace SP with LN-Net. Therefore, we can construct an LN-Net according to the operations in Figure \ref{fig:1}, and then classify the xor samples. 
 
    More generally, we discuss binary classification in Section \ref{section:binary} and multi-class classification in Section \ref{section:multi-class}. 
	
	\subsection{LN for Binary Classification}
	\label{section:binary}
 
	\begin{theorem}
		\label{thm:network}
        \REVISE{Given $m$ samples with any binary label assignment in $\{0,1\}$, there always exists an LN-Net with only $3$ neurons per layer and $O(m)$ LN layers can correctly classify them.}
	\end{theorem} 

    To prove Theorem \ref{thm:network}, we represent the LN-Net with SP and linear layers. Then we design an algorithm to help compute the parameters according to the input. We hence get an LN-Net with proper parameters to classify the samples. The proof is shown as follows. 

    We represent an LN-Net as
	\begin{equation}
		\label{eqn:14}
		f_{\theta}(\cdot) = \varphi_1 \circ LN(\cdot) \circ \varphi_2 \circ \cdots \circ \varphi_{L-1} \circ LN(\cdot) \circ \varphi_L,
	\end{equation}
	where $ \varphi_1, \cdots, \varphi_L $ denote the linear layers, and $ LN(\cdot) $ denotes the LN layers. For convenience, we replace LN with SP temporarily. 

    \begin{proposition}
        The LN-Net $ f_\theta(\cdot) $ in \Eqn\ref{eqn:14} can be represented by SP and linear layers equivalently. 
    \end{proposition}

    \begin{proof}
        Since each $ LN(\cdot) $ acts on $ \R^3 $, by Lemma \ref{lemma:equivalence}, we can construct a $2$-dimensional $ SP(\cdot) = \hat \varphi_1 \circ LN(\cdot) \circ \hat \varphi_2 $. Define each $ \varphi_l $ in Theorem \ref{thm:network} as
	\begin{equation}
		\label{eqn:15}
		\varphi_l = 
		\begin{cases}
			\varphi_l^{(1)} \circ \varphi_l^{(2)} \circ \hat \varphi_1, &l=1,\\
			\hat \varphi_2 \circ \varphi_l^{(1)} \circ \varphi_l^{(2)} \circ \hat \varphi_1, &1< l < L,\\
			\hat \varphi_2 \circ \varphi_l^{(1)}, &l=L,
		\end{cases}
	\end{equation}
	where $ \varphi_{l}^{(1)} $ and $ \varphi_{l}^{(2)} $ are both linear functions. By \Eqn\ref{eqn:15} and Lemma \ref{lemma:equivalence}, we can rewrite $ f_\theta(\cdot) $ as
	\begin{equation}
        \label{eqn:16}
		\tilde f_\theta(\cdot) = \varphi_1^{(1)} \circ \varphi_1^{(2)} \circ SP(\cdot) \circ \varphi_2^{(1)} \circ \cdots \circ SP(\cdot) \circ \varphi_L^{(1)}, 
	\end{equation}
    namely, $ f_\theta(\cdot) $ can be represented by SP and linear layers equivalently. 
    \end{proof}



    Hereafter, we consider to compute the parameters of $\tilde f_\theta(\cdot)$. Specifically, for each layer, we denote
    \begin{equation}
        \label{eqn:17}
        \begin{cases}
            \varphi_l^{(1)}: \mX^{(l-1)} \mapsto \mP^{(l)}, &1\le l \le L; \\
            \varphi_l^{(2)}: \mP^{(l)} \mapsto \mH^{(l)}, &1\le l \le L-1; \\
            SP(\cdot): \mH^{(l)} \mapsto \mX^{(l)}, &1\le l \le L-1.
        \end{cases}
    \end{equation}
    Besides, the input of $ \tilde f_\theta(\cdot) $ is $ \mX^{(0)} = [\rvx_1^{(0)}, \cdots, \rvx_m^{(0)}] $, and the output is $ \mP^{(L)} $. Now we construct $ \tilde f_\theta(\cdot) $ step by step. 
    
    We denote that for each $ \mP^{(l)}, (l=1,\cdots,L) $, these points are on the $x$-axis, namely $ \vp_k^{(l)} = [p_k^{(l)},0]^\top, (k=1,\cdots,m) $. To get $ \mP^{(1)}$, we apply $ \varphi_1^{(1)} $ for initialization as below. 

    \begin{proposition}
        \label{prop:initial}
        For any input $ \mX^{(0)} $, we can find some $ \vu $, such that 
        \begin{equation}
            \label{eqn:18}
            \varphi_1^{(1)}: \rvx_k^{(0)} \mapsto \vp_k^{(1)} = [\vu^\top\rvx_k^{(0)}, 0]^\top, 
        \end{equation}
        where $ \vp_i^{(1)} \ne \vp_j^{(1)} $ if $ \rvx_i^{(0)} \ne \rvx_j^{(0)} $. 
    \end{proposition}

    Proposition \ref{prop:initial} parameterizes $ \varphi_1^{(1)} $ and initializes $ \mP^{(1)} $ onto the $x$-axis, without merging different points\footnote{In this paper, we claim that $ \vp_i^{(l)} $ and $ \vp_j^{(l)} $ are "different points" means $ \vp_i^{(l)} \ne \vp_j^{(l)} $ rather than $ i \ne j $, for each hidden layer (applies to $\rvx$ and $\vh$ as well).}. Please refer to \textit{\SM~\ref{section:proofofAlgorithms}} for the proof. 
    
    As for other linear functions, the suitable parameters are generated from the Projection Merge Algorithm, as shown in Algorithm \ref{alg:PMA}. 
	
	\begin{algorithm}[h]
		\begin{algorithmic}[1]
			\INPUT The initial input $ \mP^{(1)} $.
			\OUTPUT The final output $ \mP^{(L)} $.
			\STATE $ l \gets 1 $;
			\STATE $ \sP \gets \{\vp_1^{(l)}, \vp_2^{(l)}, \cdots, \vp_m^{(l)}\} $;
			\WHILE{$ \sP \ne \emptyset $} 
			\STATE $ i \gets \arg\min\limits_{k}\{p_k^{(l)}:\vp_k^{(l)} \in \sP\} $; 
			\STATE $ \sJ_i \gets \{\vp_j^{(l)} \in \sP :\vp_j^{(l)} \ne \vp_i^{(l)}, y_j = y_i\} $;
			\IF{$ \sJ_i \ne \emptyset $}
			\STATE $ j \gets \arg\min\limits_{k} \{p_k^{(l)}: \vp_k^{(l)} \in \sJ_i\} $;
			\FOR{$ k \gets 1 $ to $ m $}
			\STATE $ \vh_k^{(l)} \gets \vp_k^{(l)} - \mat {p_i^{(l)} + p_j^{(l)} \\ p_i^{(l)} - p_j^{(l)}}/2 $; 
			\STATE $ \rvx_k^{(l)} \gets \vh_k^{(l)}/\|\vh_k^{(l)}\| $; 
			\STATE $ \vp_k^{(l+1)} \gets \mat{0 & 1 \\ 0 & 0} \rvx_k^{(l)} $;
			\ENDFOR
			\STATE $ l \gets l + 1 $;
			\STATE $ \sP \gets \{\vp_1^{(l)}, \vp_2^{(l)}, \cdots, \vp_m^{(l)}\} $;
			\ELSE
			\STATE remove $ \vp_j^{(l)} $ from $ \sP $, as long as $ \vp_j^{(l)} = \vp_i^{(l)} $;
			\ENDIF
			\ENDWHILE
			\STATE {\bf return} $ \mP^{(l)} $;
		\end{algorithmic}
		\caption{Projection Merge Algorithm}
		\label{alg:PMA}
	\end{algorithm}

    In Algorithm \ref{alg:PMA}, $ \mP^{(L)} $ is the output, as well as that of $ \tilde f_\theta(\cdot) $. Factually, by Algorithm \ref{alg:PMA}, we get each $ \mP^{(l)} $ in a recursive way. For the case $ \sJ_i \ne \emptyset $, we take $5$ points as an example to show how we get $\mP^{(l+1)}$ from $ \mP^{(l)} $ in Figure \ref{fig:2}. 

    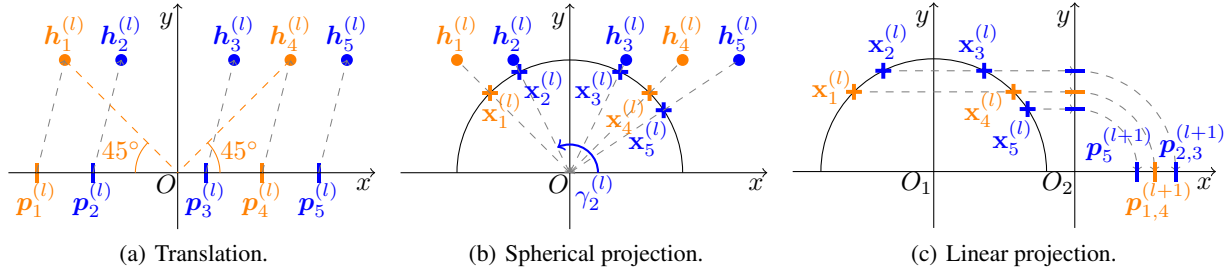
\begin{figure*}[t]
  \vspace{-0.1in}
		\centering
        \subfigure[Translation.]{
        \label{fig:2a}
        \begin{tikzpicture}[scale=0.75]
			
			\fill[blue, shift={(-1.5,0)}]  (-1pt,-5pt) rectangle         (+1pt,+5pt);
		      \fill[blue, shift={(0.5,0)}] (-1pt,-5pt) rectangle  (+1pt,+5pt);
            \fill[blue, shift={(2.5,0)}]  (-1pt,-5pt) rectangle (+1pt,+5pt);
			\fill[orange, shift={(-2.5,0)}] (-1pt,-5pt) rectangle (+1pt,+5pt);
			\fill[orange, shift={(1.5,0)}] (-1pt,-5pt) rectangle (+1pt,+5pt);
			
			\fill[blue, shift={(-1,2)}] (0,0) circle (3pt);
			\fill[blue, shift={(1,2)}] (0,0) circle (3pt);
			\fill[orange, shift={(-2,2)}] (0,0) circle (3pt);
			\fill[orange, shift={(2,2)}] (0,0) circle (3pt);
			\fill[blue, shift={(3,2)}] (0,0) circle (3pt);
			
			\draw [->, black] (-3,0) -- (3.5,0);
			\draw [->, black] (0,-1) -- (0,3);
            \node[] at (-0.2,2.7) {$y$};
            \node[] at (3.3,-0.2) {$x$};
            \node[] at (-0.2,-0.2) {$O$};
			
			\draw[->, dashed, gray, very thin] (-2.5,0) -- (-2,2);
            \draw[->, dashed, gray, very thin] (-1.5,0) -- (-1,2);
            \draw[->, dashed, gray, very thin] (0.5,0) -- (1,2);
            \draw[->, dashed, gray, very thin] (1.5,0) -- (2,2);
            \draw[->, dashed, gray, very thin] (2.5,0) -- (3,2);

            \draw [orange, dashed] (-2,2) -- (0,0);
            \draw [orange] (-0.75,0) arc [start angle=180, end angle=135, radius=0.75];
            \draw [orange, dashed] (2,2) -- (0,0);
            \draw [orange] (0.75,0) arc [start angle=0, end angle=45, radius=0.75];
            \node [orange] at (1.1,0.4) {$45^\circ$};
            \node [orange] at (-1.0,0.4) {$45^\circ$};

            \node [orange] at (-2.5,0) [below] {$\vp_1^{(l)}$};
			\node [orange] at (1.5,0) [below] {$\vp_4^{(l)}$};
			\node [blue] at (-1.5,0) [below] {$\vp_2^{(l)}$};
			\node [blue] at (0.5,0) [below] {$\vp_3^{(l)}$};
			\node [blue] at (2.5,0) [below] {$\vp_5^{(l)}$};
			\node [orange] at (-2,2) [above] {$\vh_1^{(l)}$};
			\node [orange] at (2,2) [above] {$\vh_4^{(l)}$};
			\node [blue] at (-1,2) [above] {$\vh_2^{(l)}$};
			\node [blue] at (1,2) [above] {$\vh_3^{(l)}$};
			\node [blue] at (3,2) [above] {$\vh_5^{(l)}$};
		\end{tikzpicture}
        }
        \subfigure[Spherical projection.]{
        \label{fig:2b}
        \begin{tikzpicture}[scale=0.75]	
			\fill[blue, shift={(-1,2)}] (0,0) circle (3pt);
			\fill[blue, shift={(1,2)}] (0,0) circle (3pt);
			\fill[orange, shift={(-2,2)}] (0,0) circle (3pt);
			\fill[orange, shift={(2,2)}] (0,0) circle (3pt);
			\fill[blue, shift={(3,2)}] (0,0) circle (3pt);
			
			\draw [->, black] (-3,0) -- (3.5,0);
			\draw [->, black] (0,-1) -- (0,3);
            \node[] at (-0.2,2.7) {$y$};
            \node[] at (3.3,-0.2) {$x$};
            \node[] at (-0.2,-0.2) {$O$};
			\draw [] (2,0) arc [start angle=0,end angle=180, radius=2];

            \draw [->, blue, thick] (0.5,0) arc [start angle=0,end angle=116.56, radius=0.5];
            \node [blue] at (0.45,-0.35) {$\gamma_2^{(l)}$};

			\fill[blue, shift={(-0.894,1.788)}] (-4pt,-1pt) rectangle (+4pt,+1pt) (-1pt,-4pt) rectangle (+1pt,+4pt);
			\fill[blue, shift={(0.894,1.788)}] (-4pt,-1pt) rectangle (+4pt,+1pt) (-1pt,-4pt) rectangle (+1pt,+4pt);
			\fill[orange, shift={(-1.414,1.414)}] (-4pt,-1pt) rectangle (+4pt,+1pt) (-1pt,-4pt) rectangle (+1pt,+4pt);
			\fill[orange, shift={(1.414,1.414)}] (-4pt,-1pt) rectangle (+4pt,+1pt) (-1pt,-4pt) rectangle (+1pt,+4pt);
			\fill[blue, shift={(1.664,1.109)}] (-4pt,-1pt) rectangle (+4pt,+1pt) (-1pt,-4pt) rectangle (+1pt,+4pt);
			
			\draw[<-, dashed, gray, very thin] (0,0) -- (-2,2);
            \draw[<-, dashed, gray, very thin] (0,0) -- (-1,2);
            \draw[<-, dashed, gray, very thin] (0,0) -- (1,2);
            \draw[<-, dashed, gray, very thin] (0,0) -- (2,2);
            \draw[<-, dashed, gray, very thin] (0,0) -- (3,2);
			
			\node [orange] at (-2,2) [above] {$\vh_1^{(l)}$};
			\node [orange] at (2,2) [above] {$\vh_4^{(l)}$};
			\node [blue] at (-1,2) [above] {$\vh_2^{(l)}$};
			\node [blue] at (1,2) [above] {$\vh_3^{(l)}$};
			\node [blue] at (3,2) [above] {$\vh_5^{(l)}$};
			
			\node [orange] at (-1.2,1.1) {$ \rvx_1^{(l)} $};
			\node [blue] at (-0.45,1.5) {$ \rvx_2^{(l)} $};
			\node [blue] at (0.45,1.5) {$ \rvx_3^{(l)} $};
			\node [orange] at (1.0,1.0) {$ \rvx_4^{(l)} $};
			\node [blue] at (1.4,0.6) {$ \rvx_5^{(l)} $};
		\end{tikzpicture}
        }
        \subfigure[Linear projection.]{
        \label{fig:2c}
        \begin{tikzpicture}[scale=0.75]
			
			\draw [->, black] (-2.5,0) -- (5,0);
			\draw [->, black] (0,-1) -- (0,3);
            \node[] at (-0.2,2.7) {$y$};
            \node[] at (2.3,2.7) {$y$};
            \node[] at (4.8,-0.2) {$x$};
            \node[] at (-0.3,-0.2) {$O_1$};
            \node[] at (2.2,-0.2) {$O_2$};
			\draw [] (2,0) arc [start angle=0,end angle=180, radius=2];
            \draw [->, black] (2.5,-1) -- (2.5,3);
            \draw [->, gray, dashed] (1.664,1.109) -- (2.5,1.109); 
            \draw [->, gray, dashed] (-0.894,1.788) -- (2.5,1.788); 
            \draw [->, gray, dashed] (-1.414,1.414) -- (2.5,1.414); 
            \draw [->, gray, dashed] (2.5,1.414) arc [start angle=90,end angle=0, radius=1.414];
            \draw [->, gray, dashed] (2.5,1.109) arc [start angle=90,end angle=0, radius=1.109];
            \draw [->, gray, dashed] (2.5,1.788) arc [start angle=90,end angle=0, radius=1.788];
            
            \fill[blue, shift={(3.609,0)}]  (-1pt,-5pt) rectangle         (+1pt,+5pt);
            \fill[blue, shift={(4.288,0)}]  (-1pt,-5pt) rectangle (+1pt,+5pt);
			\fill[orange, shift={(3.914,0)}] (-1pt,-5pt) rectangle (+1pt,+5pt);
			\fill[blue, shift={(-0.894,1.788)}] (-4pt,-1pt) rectangle (+4pt,+1pt) (-1pt,-4pt) rectangle (+1pt,+4pt);
			\fill[blue, shift={(0.894,1.788)}] (-4pt,-1pt) rectangle (+4pt,+1pt) (-1pt,-4pt) rectangle (+1pt,+4pt);
			\fill[orange, shift={(-1.414,1.414)}] (-4pt,-1pt) rectangle (+4pt,+1pt) (-1pt,-4pt) rectangle (+1pt,+4pt);
			\fill[orange, shift={(1.414,1.414)}] (-4pt,-1pt) rectangle (+4pt,+1pt) (-1pt,-4pt) rectangle (+1pt,+4pt);
			\fill[blue, shift={(1.664,1.109)}] (-4pt,-1pt) rectangle (+4pt,+1pt) (-1pt,-4pt) rectangle (+1pt,+4pt);

            \fill[blue, shift={(2.5,1.788)}] (-5pt,-1pt) rectangle (+5pt,+1pt);
			\fill[blue, shift={(2.5,1.109)}] (-5pt,-1pt) rectangle (+5pt,+1pt);
			\fill[orange, shift={(2.5,1.414)}] (-5pt,-1pt) rectangle (+5pt,+1pt);
			
			\node [orange] at (-1.8,1.5) {$ \rvx_1^{(l)} $};
            \node [blue] at (-0.8,2.3) {$ \rvx_2^{(l)} $};
			\node [blue] at (0.8,2.3) {$ \rvx_3^{(l)} $};
			\node [orange] at (0.9,1.1) {$ \rvx_4^{(l)} $};
			\node [blue] at (1.4,0.6) {$ \rvx_5^{(l)} $};
            \node [orange] at (4,-0.03) [below] {$\vp_{1,4}^{(l+1)}$};
			\node [blue] at (4.6,0) [above] {$\vp_{2,3}^{(l+1)}$};
			\node [blue] at (3.3,0) [above] {$\vp_5^{(l+1)}$};
		\end{tikzpicture}
        }
        \vspace{-0.15in}
		\caption{Get $ \mP^{(l + 1)} $ from $ \mP^{(l)} $ geometrically. In Figure \ref{fig:2a}, $ \mP^{(l)} $ is shown as the bars on the $ x $-axis. At first, find the leftmost point, namely $ \vp_1^{(l)} $. Then we find another point with the same label as $ \vp_1^{(l)} $, but right of $ \vp_1^{(l)} $, choose the leftmost one, namely $ \vp_4^{(l)} $. Afterwards, shift all the points up by $ (p_4^{(l)} - p_1^{(l)})/2 $, and left by $ (p_4^{(l)} + p_1^{(l)})/2 $, then we get $ \mH^{(l)} $, as shown in Figure \ref{fig:2a}. Next, spherically project $ \mH^{(l)} $ onto the unit circle and get $ \mX^{(l)} $, shown as '+'s in Figure \ref{fig:2b}. Finally merge the points in $ \mX^{(l)} $ by their ordinates, as the new abscissas of $ \mP^{(l+1)} $, and take $ 0 $ as the new ordinates of $ \mP^{(l+1)} $, as shown in Figure \ref{fig:2c}. Now, we have $ \mP^{(l+1)} $. }
		\label{fig:2}
	\end{figure*}
    
    As for the case $ \sJ_i = \emptyset $, it indicates that all points with the same label as $ \vp_i^{(l)} $ are merged together. Therefore, we remove them from $ \sP $, and choose the leftmost point from the remaining $ \sP $, until $ \sP = \emptyset $. 

    Based above, we give the properties of each layer as follows. 
    \begin{proposition}
        \label{prop:layer property}
		For each layer, $ \varphi_l^{(1)} (2\le l \le L) $ only merges points with the same label. Nevertheless, $ \varphi_1^{(1)}, SP(\cdot) $ and $ \varphi_l^{(2)} (1\le l \le L-1) $ do not merge any points. 
    \end{proposition}

    Please refer to \textit{\SM~\ref{section:proofofAlgorithms}} for the proof of Proposition \ref{prop:layer property}. 

    By Proposition \ref{prop:layer property}, we figure out that Algorithm \ref{alg:PBA} will only merge points with the same label. Besides, we find that from $ \mP^{(l)} $ to $ \mP^{(l+1)} $, the number of different points will decrease at least $1$. Since the input is $ m $ different points from two classes, we merge at most $ m-2 $ times by Algorithm \ref{alg:PMA}, we thus have $ L-1 \le m-2 $. 
    
    By Algorithm \ref{alg:PMA}, we can construct other linear functions with exact parameters as follows. 
    \begin{equation}
        \label{eqn:19}
        \begin{cases}
            \varphi_l^{(1)}: \rvx_k^{(l-1)} \mapsto \mat{0 & 1 \\ 0 & 0}\rvx_k^{(l-1)}, & 1 < l \le L, \\
            \varphi_l^{(2)}: \vp_k^{(l)} \mapsto \vp_k^{(l)} - \mat {p_i^{(l)} + p_j^{(l)} \\ p_i^{(l)} - p_j^{(l)}}/2, & 1\le l < L.
        \end{cases}
    \end{equation}

    Therefore, $ \tilde f_\theta(\cdot) $ with the parameters in \Eqn\ref{eqn:18} and \Eqn\ref{eqn:19} can classify the samples $ \mX^{(0)} $. Besides, the LN-Net in \Eqn\ref{eqn:14} with depth\footnote{We denote the number of LNs as the depth of an LN-Net. } $ L-1 = O(m) $ can also classify the $m$ samples. We hence have proved Theorem \ref{thm:network}. 

    Our results above are based on an LN-Net with $3$ neurons each layer. 
    Furthermore, we can generalize PMA for a wider neural network, but it is much more complex. Please refer to \textit{\SM ~\ref{section:proofofAlgorithms}} for more details. 

   Based on Theorem \ref{thm:network}, we can easily obtain the following corollary related to VC dimension~\cite{1998_NIPS_Bartlett} of an LN-Net. 
    \begin{corollary}
        Given an LN-Net $f_\theta(\cdot)$ with width 3 and depth $L$, its VC dimension $VCdim(f_\theta(\cdot))$ is lower bounded by $L+2$.  
    \end{corollary}

	
	\subsection{LN for Multi-class Classification}
	\label{section:multi-class}

   \begin{theorem}
		\label{thm:network-multi}
        \REVISE{Given $m$ samples with any binary label assignment, there always exists an LN-Net with only $3$ neurons per layer and $O(m)$ LN layers can correctly classify them.}
	\end{theorem} 

    Applying Algorithm \ref{alg:PMA} for a multi-class classification may confuse two samples with different labels. We thus introduce Parallelization Breaking Algorithm to avoid such confusion. Besides, we can also construct an LN-Net to classify the samples. The detailed analysis and proof are as below. 

    To begin with, we are concerned about whether Algorithm \ref{alg:PMA} applies to multi-class classification---the answer is Not. Based on Figure \ref{fig:2c}, we recolor $ \rvx_3^{(l)} $ red, as shown in Figure \ref{fig:3}. When we merge $ \rvx_1^{(l)} $ and $\rvx_4^{(l)}$, $ \rvx_2^{(l)} $ and $\rvx_3^{(l)}$ will be merged in the meanwhile. In other words, the algorithm will confuse them to be in the same class. Proposition \ref{prop:confusion} indicates the necessary condition for such confusion. 

    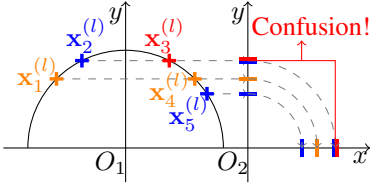
\begin{figure}[h]
		\centering
		\begin{tikzpicture}[scale=0.65]
			
			\draw [->, black] (-2.5,0) -- (5,0);
			\draw [->, black] (0,-0.7) -- (0,3);
            \node[] at (-0.2,2.7) {$y$};
            \node[] at (2.3,2.7) {$y$};
            \node[] at (4.8,-0.2) {$x$};
            \node[] at (-0.3,-0.3) {$O_1$};
            \node[] at (2.2,-0.3) {$O_2$};
			\draw [] (2,0) arc [start angle=0,end angle=180, radius=2];
            \draw [->, black] (2.5,-0.7) -- (2.5,3);
            \draw [->, gray, dashed] (1.664,1.109) -- (2.5,1.109); 
            \draw [->, gray, dashed] (-0.894,1.788) -- (2.5,1.788); 
            \draw [->, gray, dashed] (-1.414,1.414) -- (2.5,1.414); 
			
			\fill[blue, shift={(-0.894,1.788)}] (-4pt,-1pt) rectangle (+4pt,+1pt) (-1pt,-4pt) rectangle (+1pt,+4pt);
			\fill[red, shift={(0.894,1.788)}] (-4pt,-1pt) rectangle (+4pt,+1pt) (-1pt,-4pt) rectangle (+1pt,+4pt);
			\fill[orange, shift={(-1.414,1.414)}] (-4pt,-1pt) rectangle (+4pt,+1pt) (-1pt,-4pt) rectangle (+1pt,+4pt);
			\fill[orange, shift={(1.414,1.414)}] (-4pt,-1pt) rectangle (+4pt,+1pt) (-1pt,-4pt) rectangle (+1pt,+4pt);
			\fill[blue, shift={(1.664,1.109)}] (-4pt,-1pt) rectangle (+4pt,+1pt) (-1pt,-4pt) rectangle (+1pt,+4pt);

            \fill[blue, shift={(2.5,1.788)}] (-5pt,-2pt) rectangle (5pt,0pt);
            \fill[red, shift={(2.5,1.788)}] (-5pt,0pt) rectangle (+5pt,2pt);
			\fill[blue, shift={(2.5,1.109)}] (-5pt,-1pt) rectangle (+5pt,+1pt);
			\fill[orange, shift={(2.5,1.414)}] (-5pt,-1pt) rectangle (+5pt,+1pt);

            \fill[blue, shift={(3.609,0)}]  (-1pt,-5pt) rectangle         (+1pt,+5pt);
            \fill[blue, shift={(4.288,0)}]  (-2pt,-5pt) rectangle (0pt,+5pt);
            \fill[red, shift={(4.288,0)}]  (0pt,-5pt) rectangle (2pt,+5pt);
			\fill[orange, shift={(3.914,0)}] (-1pt,-5pt) rectangle (+1pt,+5pt);

            \draw [->, gray, dashed] (2.5,1.414) arc [start angle=90,end angle=0, radius=1.414];
            \draw [->, gray, dashed] (2.5,1.109) arc [start angle=90,end angle=0, radius=1.109];
            \draw [->, gray, dashed] (2.5,1.788) arc [start angle=90,end angle=0, radius=1.788];
            \draw [->, red] (2.5,1.788) -- (3.6,1.788) -- (3.6,2.2);
            \node [red,very thick] at (3.8,2.5) {Confusion!}; 
            \draw [red] (4.288,0) -- (4.288,1.788) -- (3.6,1.788);
			
			\node [orange] at (-1.8,1.5) {$ \rvx_1^{(l)} $};
			\node [blue] at (-0.8,2.3) {$ \rvx_2^{(l)} $};
			\node [red] at (0.8,2.3) {$ \rvx_3^{(l)} $};
            \node [orange] at (0.9,1.2) {$ \rvx_4^{(l)} $};
			\node [blue] at (1.3,0.7) {$ \rvx_5^{(l)} $}; 
		\end{tikzpicture}
        \vspace{-0.1in}
		\caption{The case of confusion in the merging process.}
		\label{fig:3}
	\end{figure}

    \begin{proposition}
        \label{prop:confusion}
        Confusion refers to merging two points with different labels. If confusion happens when we project $ \mX^{(l+1)} $ onto the $y$-axis, there must be a parallelogram\footnote{The parallelogram may be degenerate. Given four points $ \rvx_1, \rvx_2, \rvx_3, \rvx_4 $, if the sum of two points is the same with that of the other two, we regard they form a parallelogram.} consisting of four different points in $ \mP^{(l)} $. 
    \end{proposition}

    In reverse, if there is no parallelograms in $ \mP^{(l)} $, confusion will never happen when applying Algorithm \ref{alg:PMA}. Please refer to \textit{\SM~\ref{section:proofofAlgorithms}} for the proof of Proposition \ref{prop:confusion}. 
    
    To avoid such confusion, we propose Parallelization Breaking Algorithm (PBA) as follows. 

    \begin{algorithm}[h]
		\begin{algorithmic}[1]
			\INPUT $ \mP^{(l)} $, $ \vu_l $ (got by Proposition \ref{prop:PBA}).
			\OUTPUT $ \hat\mP^{(l)} $.
            \FOR{$k \gets 1 $ to $m$}
                \STATE $ \tilde\vp_k^{(l)} = SP(\vp_k^{(l)} +[0,1]^\top) $;
                \STATE $ \hat\vp_k^{(l)} = [\vu_l^\top \tilde \vp_k^{(l)}, 0]^
                \top $; 
            \ENDFOR
			\STATE {\bf return} $ \hat\mP^{(l)} $;
		\end{algorithmic}
		\caption{Parallelization Breaking Algorithm}
		\label{alg:PBA}
	\end{algorithm}

    \begin{proposition}
        \label{prop:PBA}
        We can always find $ \vu_l \in \R^2 $ for Algorithm \ref{alg:PBA}, such that there is no parallelograms in $ \hat\mP^{(l)} $, and no points merged in the algorithm. 
    \end{proposition}

    Please refer to \textit{\SM}~\ref{section:proofofAlgorithms} for the proof of Proposition \ref{prop:PBA}.




    PBA helps us transform $ \mP^{(l)} $ to $ \hat\mP^{(l)} $, based on which confusion will never happen. For multi-class classification, we insert PBA between $ \varphi_l^{(1)} $ and $ \varphi_l^{(2)} $ in \Eqn\ref{eqn:16}, then given $ m $ samples with any label assignment, $ \tilde f_\theta(\cdot) $ with PBA can classify them. Based above, we replace SP with LN and linear layers in $ \tilde f_\theta(\cdot) $ with PBA, and then merge the adjacent linear layers. We figure out $ \tilde f_\theta(\cdot) $ with PBA is also an LN-Net. We point out that the depth of this LN-Net is no more than $ 2m $. We hence have proved Theorem \ref{thm:network-multi}.

    \paragraph{Summary.} In this section, we show that LN-Net also has powerful capacity in theory. Our theoretical results show that an LN-Net with width $3$ and depth $O(m)$ is able to classify given $m$ samples with any label assignment. We see an LN-Net performing over $3$ neurons can introduce nonlinearity. One question is that whether the nonlinearity of an LN-Net with $d>3$ neurons can be amplified, if we group neurons and perform LN in each group in parallel? We answer it in the following section.   
	
	\section{Amplify and Exploit the Nonlinearity of LN}

      \subsection{Comparison of Nonlinearity}
      \label{section:5.1}
    
In this part, we first define a measurement over the Hessian matrix to compare the magnitude of the nonlinearity. We then show the Group based LN (LN-G)\footnote{We use the new defined term LN-G rather than Group Normalization (GN)~\cite{2018_ECCV_Wu}, considering that: 1) GN is defined on the convolutional input $\mX \in \mathbb{R}^{d \times h \times w}$ but not on the input $\rx \in \mathbb{R}^d$; 2) Given the sequential input (e.g., text) $\mX \in \mathbb{R}^{d \times T}$ in Transformer/ViT, GN will share statistics over $T$ by definition while LN-G will have no inter-sequence dependence and use separate statistics over T, like LN.}---which divides neurons of a layer into groups and perform LN in each group in parallel---has stronger nonlinearity than the naive LN countpart.  

    \paragraph{Hessian of Linear Function.}
	Given a twice differential function $ f(\rvx) :\R^d \to \R $, we focus on its Hessian Matrix $ \nabla^2 f(\rvx) $. If $ f(\rvx) $ is a linear function, we have $ \nabla^2 f(\rvx) \equiv \mO $. 
	More generally, suppose that $ \varphi: \R^d \to \R^d $ is a linear transformation, we define $ \varphi(\rvx) = \mat{\varphi_1(\rvx), \cdots, \varphi_d(\rvx)}^\top $, and each $ \varphi_i(\rvx): \R^d \to \R $ is a linear function, namely each Hessian matrix $ \nabla^2_{\rvx}\varphi_i(\rvx) = \mO $. 

    \paragraph{Measurement of Nonlinearity.}
	
  Given a twice differential function\footnote{For $ \vy = f(\rvx) $, we require each $ y_i(i=1,\cdots,d) $ is twice differential about $ \rvx $.} $ f: \R^d \to \R^d $ and $ \vy = f(\rvx) $. Denote $ \vy = [y_1,\cdots,y_d]^\top $ and $ \rvx = [x_1,\cdots,x_d]^\top $. We define $ \mathcal{H}(f;\rvx) $ as an indicator to describe the Hessian information of $ f:\R^d\to\R^d $ as 
    \begin{equation}
        \mathcal{H}(f;\rvx) = \sum_{i=1}^d  \left\| \pf{^2 y_i}{\rvx^2} \right\|_F^2, 
    \end{equation}
    where each $\displaystyle \pf{^2 y_i}{\rvx^2} $ is a Hessian matrix. 
    
    We use the Frobenius norm rather than the operator norm, For easier calculations. Note that $\mathcal{H}(f;\rvx) \ge 0$, and $\mathcal{H}(f;\rvx)=0$ if and only if $f$ is a linear function. We thus\textbf{ assume} that the larger $\mathcal{H}(f;\rvx)$ is, the more nonlinearity $f$ contains.
        
    \paragraph{Amplifying Nonliearity by Group.} Denote $ \psi_G(g;\cdot) $ as Group based LN (LN-G) on $ \R^d $ with group number $g$, and $ \psi_L(\cdot) $ as LN on $ \R^d $. Compare LN with LN-G, the result is shown in Proposition \ref{prop:Hessian}. 
        
        \begin{proposition}   
        \label{prop:Hessian}
            Given $ g\le d/3 $, we have
            \begin{equation}
            \frac{\mathcal{H}( \psi_G(g;\cdot); \rvx )}{\mathcal{H}( \psi_L(\cdot); \rvx )} \ge 1. 
        \end{equation}
        Specifically, when $ g = d/4 $, we figure out that
        \begin{equation}
            \frac{\mathcal{H}( \psi_G(g;\cdot); \rvx )}{\mathcal{H}( \psi_L(\cdot); \rvx )} \ge \frac{d}{8} . 
        \end{equation}
        \end{proposition}

   Proposition \ref{prop:Hessian} shows that LN-G can amplify the nonliearity of LN by using appropriated group number. Compared with LN, when $d$ is larger, LN-G shows more nonlineaity. Please refer to \textit{\SM~\ref{section:proofofhessian}} for the proof.
     \REVISE{Besides, we generalize our discussion about $\mathcal{H}$ to the typical activation function ReLU, please refer to \textit{\SM~\ref{section:proofofhessian}} for more details. }
     
   One limit of the result above is the assumption, that $\mathcal{H}(f;\rvx)$ is a good indicator for measuring nonlinearity, is from the intuition and can not be well verified. In the subsequent experiments, we empirically show that LN-G indeed can amplify the nonlinearity of LN.

    \subsection{Comparison of Representation Capacity by Fitting Random Labels}
        \label{sec:Experiments}
        \vspace{-0.05in}
    In this part, we follow the non-parametric randomization tests fitting random labels~\cite{2017_ICLR_Zhang} to empirically verify the nonlinearity of LN, and to further compare the representational capacity of LN-Net with different groups for LN-G. 
    The experiments are conducted on CIFAR-10 and MNIST with random label assigned (CIFAR-10-RL and MNIST-RL). We evaluate the classification accuracy on the training set after the model is trained, which indicates that the capacity of models in fitting dataset empirically. We only provide essential components of the experimental setup; for more details, please refer to the \textit{\SM~\ref{section:experiments}}.

  
  
  

  
    \paragraph{Verify the Nonlinearity of LN.}  We conduct experiments on linear neural network and LN-Net with 256 neurons in each layer and various depths. We first train sufficiently a linear classifier and obtain the (nearly) upper bound accuracy (18.51 \% on CIFAR-10 -RL and 15.38\% on MNIST-RL). To rule out the influence in optimization difficulty, we train the linear neural network and LN-Net with various configurations, including different learning rates and (with or without) residual connection\footnote{A linear neural network with residual connection is still a linear model.}. We report the best result from all configurations, as shown in Figure~\ref{fig:Res-CIFAR10}. 
    

   We observe that linear neural network cannot break the bound of linear classifier on all datasets, while LN-Net can reach the accuracy of 55.85\% on CIFAR-10-RL and 19.44\% on MNIST-RL, which is much better than the linear classifier. This result also verifies that LN has nonlinearity empirically. Besides, we observe that LN-Net obtains better performance in general as the depth increases (namely more LN layers and greater nonlinearity). 
   We note that an LN-Net without sufficient depth does not break the bound of linear classifier on MNIST-RL. The reasons leading to this phenomenon are likely to be that: 1) MNIST-RL are more difficult to train, compare to CIFAR-10-RL; 2) LN-Nets have a non-convex optimization landscape and we cannot ensure the weight learned to be the optimal point, given fixed training epochs.

    We also conduct experiments with Batch Normalization (BN)~\cite{2015_ICML_Ioffe}, where we replace LN with BN in LN-Net. We find that BN cannot break the bound of linear classifier on all datasets, like linear neural network. This preliminary result is interesting, which shows the potential advantage of LN over BN, in terms of the representation capacity. 
   
   
    

    \begin{figure}[t]
        \vspace{-0.1in}
        \centering
        \hspace{-0.25in}	
        \subfigure[CIFAR-10-RL.]{
            \label{fig:Res-CIFAR10a}
            \begin{minipage}[c]{.43\linewidth}
                \centering
                \includegraphics[width=4.2cm]{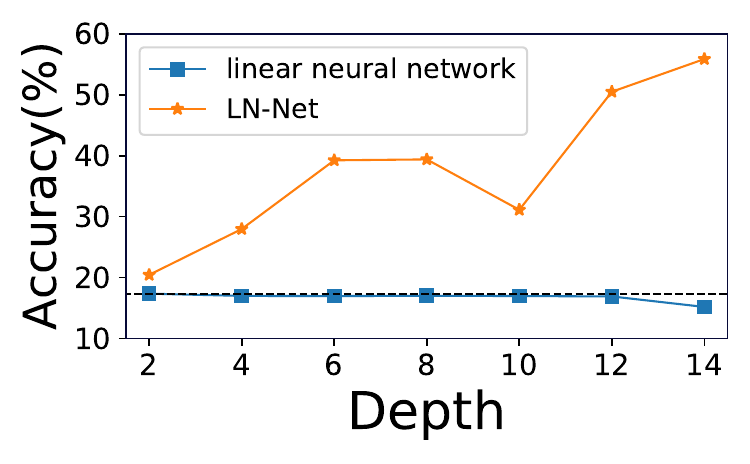}
            \end{minipage}
        }
        \hspace{0.15in}		\subfigure[MNIST-RL.]{
            \begin{minipage}[c]{.43\linewidth}
            \label{fig:Res-CIFAR10b}
                \centering
                \includegraphics[width=4.2cm]{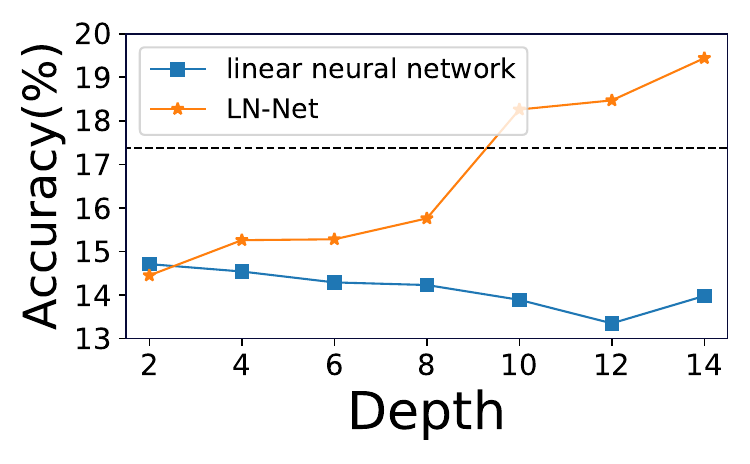}
            \end{minipage}
        }	
        \vspace{-0.1in}
        \caption{Results of linear neural network and LN-Net on fitting random label. The black dashed line represents the upper bound accuracy of linear classifier. (a) Results on CIFAR-10-RL; (b) Results on MNIST-RL.}
        \label{fig:Res-CIFAR10}
        \vspace{-0.16in}
    \end{figure}

    \label{section:6.2}
        \begin{figure*}[t]
        	\centering
        	\hspace{-0.25in}	\subfigure[Accuracy on CIFAR-10-RL.]{
        		\begin{minipage}[c]{.23\linewidth}
        			\centering
        			\includegraphics[width=4.2cm]{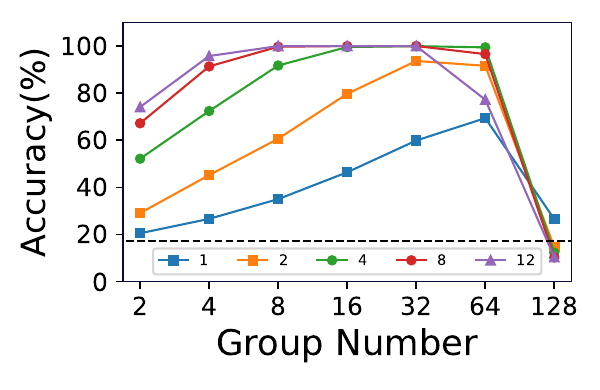}
        		\end{minipage}
        	}
        	\hspace{0in}		\subfigure[Accuracy on MNIST-RL.]{
        		\begin{minipage}[c]{.23\linewidth}
        			\centering
        			\includegraphics[width=4.2cm]{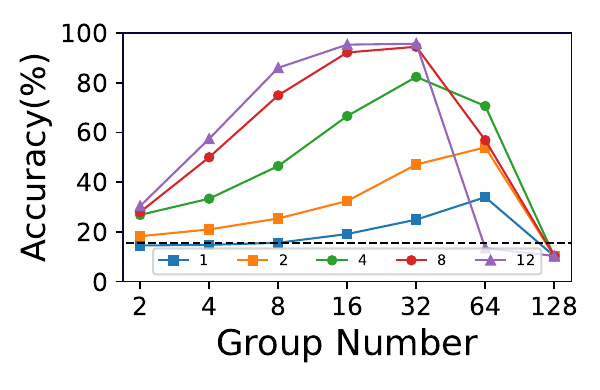}
        		\end{minipage}
        	}	
        	\hspace{0in}	\subfigure[$ \mathcal{H}(f;\rvx)$ on CIFAR-10-RL.]{
        	\begin{minipage}[c]{.23\linewidth}
        		\includegraphics[width=3.69cm]{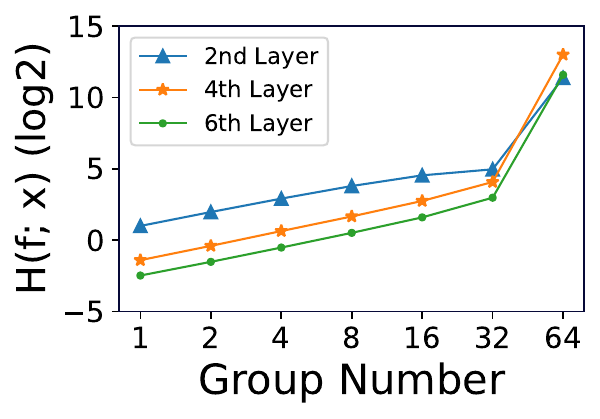}
        	\end{minipage}
                }
                \hspace{-0.15in}	\subfigure[$ \mathcal{H}(f;\rvx)$ on MNIST-RL.]{
        	\begin{minipage}[c]{.23\linewidth}
        		\centering
        		\includegraphics[width=3.69cm]{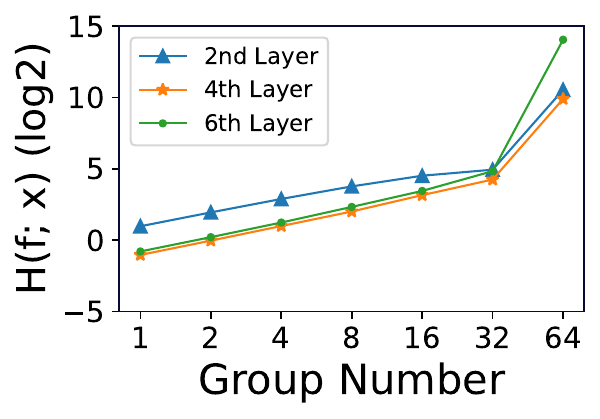}
        	\end{minipage}
                }
        	\vspace{-0.1in}
        	\caption{Results of LN-Net using LN-G. We vary the group number $g$ and show the training accuracy and $ \mathcal{H}(f;\rvx)$. (a) Training accuracy on CIFAR-10-RL; (b) Training accuracy on MNIST-RL; (c)$ \mathcal{H}(f;\rvx)$ on CIFAR-10-RL;(d) $ \mathcal{H}(f;\rvx)$ on MNIST-RL. The black dashed line in (a) and (b) has the same meaning as that in Figure \ref{fig:Res-CIFAR10}.}
        	\label{fig:MLP-GN}
        	\vspace{-0.16in}
        \end{figure*}

%


    \paragraph{Amplifying the Nonlinearity using Group.}
     We conduct experiments on LN-Net  with $d=256$ neurons in each layer and various depths. We replace LN in LN-Net with LN-G and also vary the group number $g$. We train LN-Net with various learning rates and report the best training accuracy on CIFAR-10-RL and MNIST-RL, as shown in Figure~\ref{fig:MLP-GN}.

     We observe that some LN-Net with LN-G (e.g., depth = 8 and $g=32$) can perfectly classify all the random labels on CIFAR-10-RL and MNIST-RL, which suggests that LN-G can amplify the nonlinearity of LN by using group, as stated in Proposition~\ref{prop:Hessian}. We also observe that an LN-Net with appropriate group number (e.g, $g =32$) can obtain better performance, as the depth increases. Besides, an LN-Net has better performance in general with larger group number, along the group number is not too much (relative to the number of neurons). E.g, An LN-Net has significantly degenerated performance when $g=128$, due to $d/g=2<3$ that go against the premise in Proposition~\ref{prop:Hessian}.
     
      \REVISE{We also calculate $ \mathcal{H}(f;\rvx)$ in certain layers and show how  $ \mathcal{H}(f;\rvx)$ varies as the group number increases in Figure~\ref{fig:MLP-GN} (c) and (d).  $ \mathcal{H}(f;\rvx)$ is calculated by averaging over 1000 samples in our experiments. We find  $ \mathcal{H}(f;\rvx)$ increases as the group number of LN-G increases, which matches our theoretical analyses in Section~\ref{section:5.1}.}



    \subsection{Inspiration for Neural Architecture Design}
        \label{section:6.3}
       In this part, we consider designing neural networks in real scenarios, considering that LN-G can amplify the nonlinearity and have great performance in fitting the random label shown in Section~\ref{sec:Experiments}. We conduct experiments on both CNN and Transformer architectures. 
        \subsubsection{CNN without activation function}\label{sec:CNN_EX}
         To validate the representation capacity of LN-G in real scenarios further, we conducted experiments on CIFAR-10 using ResNet~\cite{2015_CVPR_He}.  To exclude the influence of other nonlinearities, we remove all nonlinear activations from the ResNet, and refer the network to ResNet-NA. We set the channel number of each layer to 128 for better ablating the group number of LN-G. We also conduct experiments on more CNNs shown in \textit{\SM~\ref{section:CNN-extension}}
        
                \begin{figure}[t]
        	\centering
        	\hspace{-0.25in}	\subfigure[Training.]{
        		\begin{minipage}[c]{.43\linewidth}
        			\centering
        			\includegraphics[width=4.2cm]{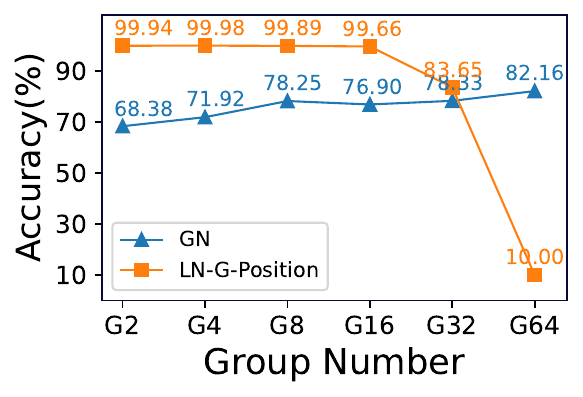}
        		\end{minipage}
        	}
        	\hspace{0.15in}		\subfigure[Test.]{
        		\begin{minipage}[c]{.43\linewidth}
        			\centering
        			\includegraphics[width=4.2cm]{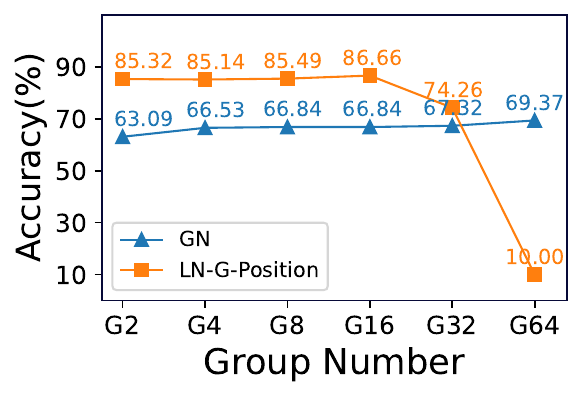}
        		\end{minipage}
        	}	
        	\vspace{-0.12in}
        	\caption{Results of the variants of LN-G (GN and LN-G-Position)  when using different group number. The experiments are conducted on CIFAR-10 dataset using ResNet without ReLU activation. We show (a) the training accuracy and (b) the test accuracy. In the x-axis, G2 refers to a group number of 2.}
        	\label{fig:resnet}
        	\vspace{-0.16in}
        \end{figure}
         
         \paragraph{Investigation of LN-G.}         
         Note that LN-G may have several variants for a convolutional input $\mX \in \mathbb{R}^{c \times h \times w}$, where $c$, $h$ and $w$ indicate the feature mappings' channel, height and width dimensions  respectively. Following the usage of LN on CNNs, LN-G should calculate the mean/variance along all the channel, height and width dimensions, which is  equivalent to Group Normalization (GN)~\cite{2018_ECCV_Wu}. Following the usage of LN on MLP$\&$Transformer, LN-G should calculate the mean/variance along only the channel dimension and use separate statistics over each position (a pair of height and width), and we refer to this method as LN-G-Position. 
         
         We investigate how the group number affects the performance of the variants of LN-G (GN and LN-G-Position).  We vary the group number $g$ ranging in \{2, 4, 8, 16, 32, 64\}. We train a total of 200 epochs using SGD with a mini-batch size of 128, momentum of 0.9 and weight decay of 0.0001.
         The initial learning rate is set to 0.1, and divided by 5 at the 60th, 120th, and 160th epochs. The results are shown in Figure~\ref{fig:resnet}. We find that GN obtains slightly better performance as the group number increases. Note that this observation does not go against the experimental results of LN-G in amplifying the nonlinearity in Section~\ref{sec:Experiments} since the `effective samples' used to calculate the normalization statistics in each group of GN is $\frac{h*w*c}{g}$. We observe that LN-G-Position works particularly well and obtains over $85\%$ test accuracy for multiple group number (Note that there is no ReLU activations.). We also find that LN-G-Position works particularly bad if group number is 64, because the samples used to calculate the normalization statistics in each group of LN-G-Position is $\frac{c}{g}=2$.
         
         
      \begin{table}[t]
      	\caption{Comparison of different normalization methods on CIFAR-10 using ResNet-NA (ResNet without ReLU activation).}
      	\begin{tabular}{c|c|c}
      		\hline
      		Normalization methods   & Train Acc(\%) & Test  Acc(\%) \\ \hline
      		IN          & 10                 & 10                 \\
      		BN             & 36.0               & 39.3               \\
      		LN             & 59.5               & 62.85              \\ 
      		GN  & 82.16              & 69.37              \\
      		LN-G-Position & 99.66              & 86.66              \\ \hline
      	\end{tabular}
      \end{table}  

         \paragraph{Comparison to other Normalization.} 
 		\REVISE{We also conduct experiments to train ResNet-NA by using other normalization methods, including the original Batch Normalization (BN)~\cite{2015_ICML_Ioffe}, Layer Normalization (LN)~\cite{2016_LN_Ba}, Instance Normalization (IN)~\cite{2016_arxiv_Ulyanov}. Besides, we also train ResNet-NA without normalization.  We use the same setting up described in previous experiments. 
          We find that ResNet-NA without normalization is very difficult to train and shows a random guess behavior. Similarly, ResNet-NA with IN is also very difficult to train.  ResNet-NA with BN can be trained normally. However, the performance of the model is relatively low, with only $39.3\%$ test accuracy. ResNet-NA with LN obtains $62.85\%$ test accuracy, which is significantly better than BN. Furthermore, ResNet-NA with LN-G-Position obtains the best performance, e.g., a test accuracy of $86.66\%$ when using a group number 16 for LN-G-Position. We contribute it to the strong nonlinearity of LN-G-Position. }

         \subsubsection{LN-G in Transformers}
         \paragraph{Transformer for Machine Translation.} We conduct experiments to apply LN-G on Transformer~\cite{2017_NIPS_Vaswani} (where LN is the default normalization) for machine translation tasks using \textit{fairseq-py}~\cite{2019_fairseq_Ott}. We evaluate the public  IWSLT14 German-to-English (De-EN) dataset using BLEU (higher is better). We use the hyper-parameters recommended  in \textit{fairseq-py}~\cite{2019_fairseq_Ott} for Transformer and train over 50 epochs with five random seeds. The baseline LN has a BLEU score of $35.01 \pm 0.10$.  LN-G  (replacing all the LNs with LN-G) has a BLEU score of $35.23\pm0.07$.
         \paragraph{ViT for Image Classification.} We conducted experiments by applying LN-G to Tiny-VIT (with the default normalization being LN). We performed classification tests on the test set of the CIFAR-10 dataset, with hyperparameter settings referencing~\cite{how_to_train_vit}. The classification accuracy on the test dataset was $88.81\%$ for LN and $89.26\%$ for LN-G (replacing all the LNs with LN-G).

         These preliminary results show the potentiality of LN-G used for neural architecture design in practice. 

   \section{Related Work} 

Previous theoretical analyses on normalization  are mainly focused on BN, the pioneer work in normalization for deep learning. 
 One main argument is that BN can improve the conditioning of the optimization problem~\cite{2019_ICML_Cai}, either by avoiding the rank collapse of pre-activation matrices~\cite{2020_NIPS_Daneshmand} or by alleviating the pathological sharpness of the landscape~\cite{2018_NIPS_shibani,2019_NIPS_Karakida,2019_ICML_Ghorbani,2022_NIPS_Lv}. The improved conditioning enables large learning rates~\cite{2018_NIPS_Bjorck}, thus improving the generalization~\cite{2019_ICLR_Luo2}.
Another argument is that BN is scale invariant~\cite{2016_LN_Ba}, enabling it to adaptively adjust the learning rate \cite{2019_ICLR_Arora,2019_ICML_Cai,2019_ICLR_Zhang,2020_ICLR_Li}, which stabilizes and further accelerates training. This scale invariant analyses also applies to LN~\cite{2016_LN_Ba,2021_NIPS_Lubana}. Some work address to understanding LN empirically through experiments, showing that the learnable parameters in LN increases the risk of over-fitting~\cite{2019_NIPS_Xu}.

Different from these work, we investigate a new theoretical direction for LN, regarding to its nonlinearity and representation capacity. We note that there are several work~\cite{2021_CVPR_Huang,2021_NIPS_Labatie} investigating the expressive power of normalization empirically by experiments. However, their experiments are conducted on networks with activation functions, while our work focuses on analyzing the representation capacity of a network without activation functions through theory and experiment.

	\section{Conclusion}
			\label{sec:conclusion}
	
        We mathematically demonstrated that LN is a nonlinear transformation. We also theoretically showed the representation capacity of an LN-Net in correctly classifying samples with any label assignment. We demonstrated these results by finely designing algorithms, considering the geometric property of LN. We hope that our techniques will inspire the community to reconsider the analyses of the representation capacity of a network with normalization layer, though it suffers from great challenges~\cite{2023_TPAMI_Huang}.
        
        \vspace{-0.15in}
        \paragraph{Limitation and Future Work.}
      Our results in representation capacity for LN-Net is very loose currently, which is like the initial universal approximation theory in the arbitrary wide shallow neural network~\cite{1989_Universal}. We believe it is interesting to extend our results along the direction as universal approximation theory is extended to the cases of arbitrary depth~\cite{gripenberg2003approximation}, bounded depth and bounded width~\cite{maiorov1999lower}, and the question of minimal possible width~\cite{park2020minimum}. Besides, the effectiveness of group mechanism for LN (\ie, LN-G) is only verified on small-scale networks and datasets, and more results on large-scale networks and datasets are required to support the practicality of LN-G.

\section*{Acknowledgments}
This work was partially supported  by the National Science and Technology Major Project under Grant 2022ZD0116310, National Natural Science Foundation of China (Grant No. 62106012), the Fundamental Research Funds for the Central Universities.

    \section*{Impact Statement}
    This paper presents work whose goal is to advance the field of Machine Learning. There are many potential social consequences of our work, none which feel must be specifically highlighted here. 
     \bibliography{main}
    \bibliographystyle{icml2024}
    	\newpage

	\onecolumn
	\appendix
\clearpage
\renewcommand{\thetable}{\Roman{table}}
\setcounter{table}{0}

\renewcommand{\thefigure}{A\arabic{figure}}
\setcounter{figure}{0}

	
	\section{LSSR as a Linearly Separable Indicator}
        \label{section:LSSR-illu}
	
	\begin{figure*}[h]
		\vspace{-0.1in}
		\centering
		\hspace{-0.05in}	\subfigure[XOR data.]{
			\label{fig:4a}
			\includegraphics[width=4cm]{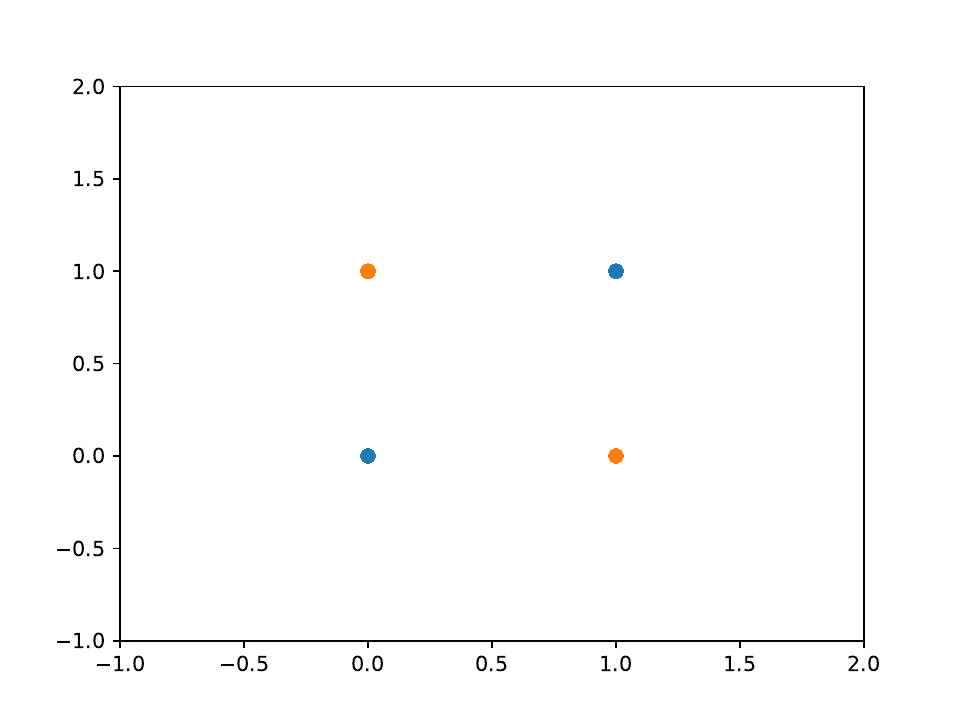}
		}
		\centering
		\hspace{-0.05in}	\subfigure[Inseparable Gaussian data.]{
			\centering
			\label{fig:4b}
			\includegraphics[width=4cm]{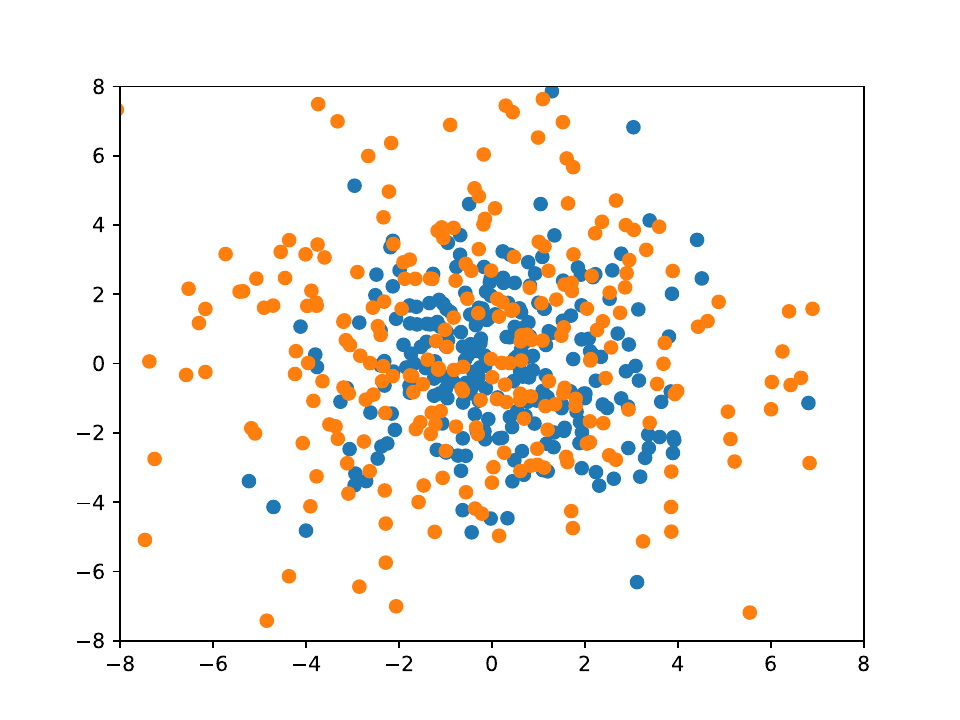}
		}
		\centering
		\hspace{-0.05in}	\subfigure[Separable Gaussian data.]{
			\centering
			\label{fig:4c}
			\includegraphics[width=4cm]{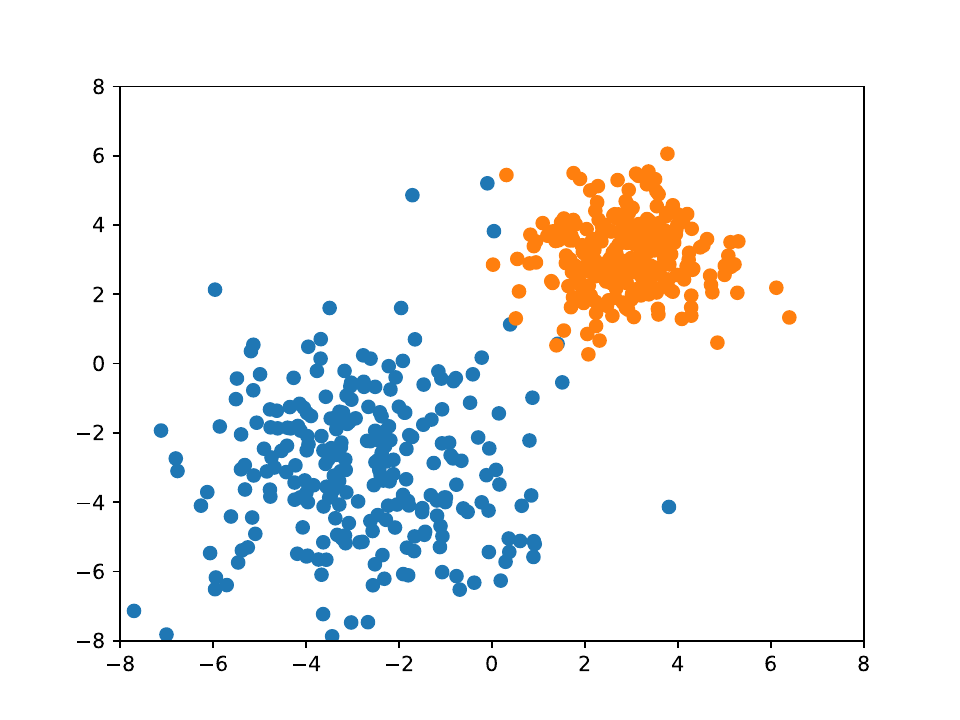}
		}
		\centering
		\hspace{-0.05in}	\subfigure[Parallel Gaussian data.]{
			\centering
			\label{fig:4d}
			\includegraphics[width=4cm]{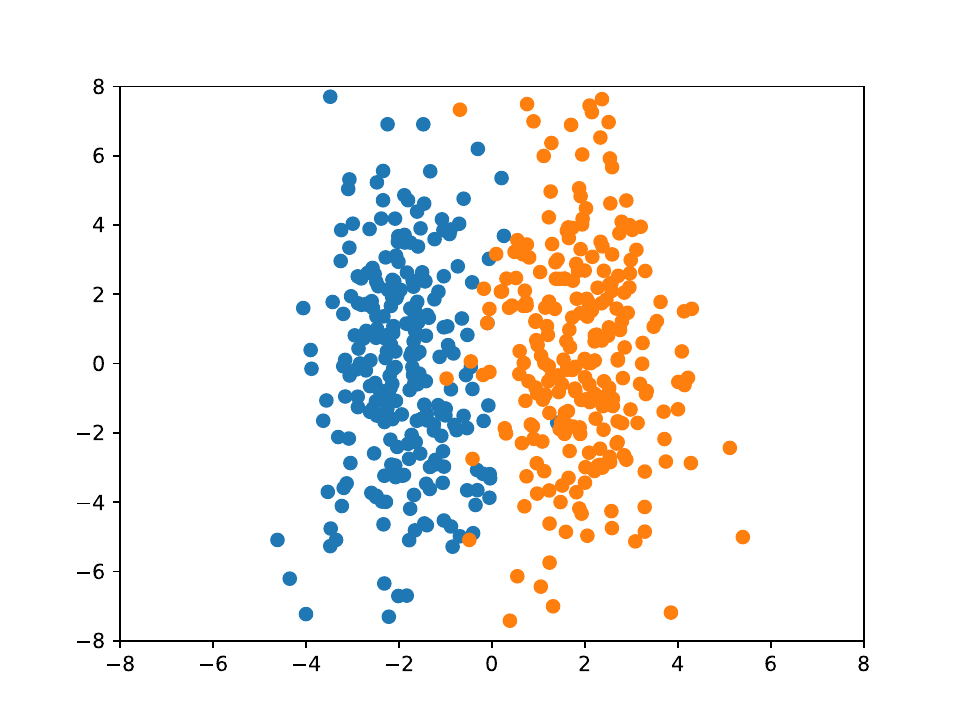}
		}
		\vspace{-0.13in}
		\caption{We randomly sample $256$ points from each class in the four different distributions of data above. The detailed data is shown in Table \ref{tab:1}.} 
		\label{fig:4}
		\vspace{-0.12in}
	\end{figure*}
    To show how SSR and LSSR evaluate the difficulty of separating the samples from different classes linearly, we give four different distributions of data in the figure above and their details in the table below.

    \begin{table}[h]
    \vspace{-0.15in}
            \caption{Detailed data of  Figure \ref{fig:4}. In Figure \ref{fig:4a}, the random variance $ X $ takes values $0$ and $1$ with probabilities $1/2$ each. In the other figures, the sign $ N(\cdot, \cdot) $ denotes the Gaussian distribution.}
	\label{tab:1}
	\centering
	\begin{tabular}{|c|c|c|c|c|}
		\toprule
		Figure & Distribution of $X_1$ & Distribution of $X_2$ & SSR & LSSR \\
		\midrule
		Figure \ref{fig:4a} & $ X_1 = \mat{X\\X} $ & $ X_2 = \mat{X \\ 1-X}$ & 0.9963 & 0.9929 \\
        \midrule
		Figure \ref{fig:4b} & $ X_1 \sim N\left(\mat{0\\0}, \mat{4&0\\0&4} \right) $ & $ X_2 \sim N\left(\mat{0\\0}, \mat{9&0\\0&9} \right) $ & 0.9929 & 0.9859 \\
        \midrule
		Figure \ref{fig:4c} & $ X_1 \sim N\left(\mat{-3\\-3}, \mat{4&0\\0&4} \right) $ & $ X_2 \sim N\left(\mat{3\\3}, \mat{1&0\\0&1} \right) $ & 0.2304 & 0.1312 \\
        \midrule
        Figure \ref{fig:4d} & $ X_1 \sim N\left(\mat{-2\\0}, \mat{1&0\\0&9} \right) $ & $ X_2 \sim N\left(\mat{2\\0}, \mat{1&0\\0&9} \right) $ & 0.7536 & 0.2157 \\
		\bottomrule
	\end{tabular}
        \vspace{-0.1in}
    \end{table}
    According to Figure \ref{fig:4} and Table \ref{tab:1}, we have several conclusions below. In Figure \ref{fig:4a} and Figure \ref{fig:4b}, the classes are hard to be linearly separated, whose SSR and LSSR are both near $1$. In Figure \ref{fig:4c}, the classes are easy to be linearly separated, whose SSR and LSSR are both near $0$. However, in Figure \ref{fig:4d}, the classes are easy to be linearly separated, but harder to be separated if focused on the Euclidean distance. As a result, its SSR is larger, but its LSSR is near $0$. We hence conclude that---LSSR is a better indicator than SSR in judging how two classes are linearly separable.
	
	\section{Proofs of Proposition \ref{prop:LSSR}}
    \label{section:Prop2}

	\paragraph{Proposition 2.}\textit{Given $ \mX_1, \mX_2 \in \R^{d \times m} $, we denote $ \mM = \sum\limits_{c=1}^{2} \sum\limits_{i=1}^m (\rvx_{ci}-\bar{\rvx}_{c})(\rvx_{ci}-\bar{\rvx}_{c})^\top $, and $ \mN = \sum\limits_{c=1}^{2} \sum\limits_{i=1}^m (\rvx_{ci}-\bar{\rvx})(\rvx_{ci}-\bar{\rvx})^\top $, where $\bar\rvx = (\bar\rvx_1 + \bar\rvx_2)/2 $. Supposing that $ \mN $ is reversible, we have  
		\begin{equation}
			LSSR(\mX_1, \mX_2) = \lambda^*, 
		\end{equation}
		and correspondingly, 
		\begin{equation}
			LSSR(\mX_1, \mX_2) = SSR((\vu^*)^\top\mX_1, (\vu^*)^\top\mX_2), 
		\end{equation}
	where  $ \lambda^* $ and  $\vu^* $ are the minimum eigenvalue and corresponding eigenvector of $ \mN^{-1} \mM $.}
	
	Since the definition of LSSR comes from a lower bound, we prove Proposition \ref{prop:LSSR} from solving the optimization problem as follows. 
    \begin{equation}
        (P_{LSSR})\begin{cases}
            \min\limits_{\varphi} \quad &LSSR(\mX_1,\mX_2) = SSR(\varphi(\mX_1),\varphi(\mX_2)), \\
            \text{s.t.} \quad &\varphi(\rvx) = \mW\rvx + \vb, \\
            & \mW \in \R^{n \times d}, \vb \in \R^n, n\in \sN^*, \\
            & SS(\varphi(\mX_1),\varphi(\mX_2)) \ne 0. 
        \end{cases}
    \end{equation}
    To solve this, we first propose four lemmas, and then use them to prove Proposition \ref{prop:LSSR}. Furthermore, we give the optimal $\mW$ as a corollary. 

    \subsection{Required Lemmas for the Proof}
	
	\begin{lemma}
		\label{lemma:21}
		The bias $ \vb \in \R^n $ does not affect SSR, as well as LSSR. 
	\end{lemma}
	
	\begin{proof} By the definition of SS, we obtain
		\begin{equation}
			\begin{aligned}
				SS(\varphi(\mX_1)) &= SS(\mW\mX_1+\vb\bm1^\top) \\
				&= \sum_{i=1}^m\left\Vert \mW\rvx_{1i} + \vb - \frac1m\sum_{i=1}^m (\mW\rvx_{1i}+\vb) \right\Vert^2\\
				&= \sum_{i=1}^m\left\Vert \mW\rvx_{1i} - \frac1m\sum_{i=1}^m \mW\rvx_{1i} \right\Vert^2\\
				&= \sum_{i=1}^m\left\Vert \mW\rvx_{1i} - \overline{\mW\rvx_1} \right\Vert^2\\
				&= SS(\mW\mX_1).
			\end{aligned}
		\end{equation}
		Similarly, we have
        \begin{equation}
            SS(\varphi(\mX_2)) = SS(\mW\mX_2) , 
        \end{equation}
        and
        \begin{equation}
            SS([\varphi(\mX_1), \varphi(\mX_2)]) = SS([\mW\mX_1, \mW\mX_2]) .
        \end{equation}
        Since SSR is defined with SS, the conclusion also holds for SSR, namely
		\begin{equation}
			SSR(\varphi(\mX_1),\varphi(\mX_2)) = SSR (\mW\mX_1, \mW\mX_2),  
		\end{equation}
    where the bias $\vb$ is not included. 
	\end{proof}

    \begin{lemma}
	\label{lemma:2}
	Suppose the eigenvalue decomposition of $ \mW^\top\mW $ as
	\begin{equation}
        \label{eqn:25}
		\mW^\top\mW = \mU \Lambda \mU^\top = \sum_{i=1}^d \lambda_i \vu_i \vu_i^\top, 
	\end{equation}
	where $ \mU=[\vu_1,\cdots,\vu_d] $ is an orthogonal matrix, and $ \Lambda=diag\{\lambda_1,\cdots,\lambda_d\} $ is a positive semi-definite and diagonal matrix. We consider to minimize $SSR (\mW\mX_1, \mW\mX_2)$ over  $ \Lambda $ with a fixed $ \mU $, as:
		\begin{equation}
			\min_{\Lambda} SSR(\mW\mX_1, \mW\mX_2) = \min_{1\le j\le d} SSR(\vu_j^\top\mX_1,  \vu_j^\top\mX_2).
		\end{equation}
		The optimal solution is that 
		\begin{equation}
			\lambda_j \begin{cases}
				\ge0, &j \in\arg\min\limits_{1\le j\le d}SSR(\vu_j^\top\mX_1, \vu_j^\top\mX_2),\\
				=0, &\text{otherwise},
			\end{cases}
		\end{equation}
		for $ j=1,\cdots,d$, and $ \lambda_1,\cdots,\lambda_d $ are not all zeros. 
	\end{lemma}
	
	\begin{proof} 
        By Lemma \ref{lemma:21}, we find that
	\begin{equation}
		LSSR(\mX_1,\mX_2) = \min_{\mW} SSR(\mW\mX_1,\mW\mX_2). 
	\end{equation} 
	Besides, we figure out that
	\begin{equation}
        \label{eqn:29}
		\begin{aligned}
			SS(\mW\mX_c) 
			= \sum_{i=1}^m\left\Vert \mW\rvx_{ci} - \mW\bar\rvx_c \right\Vert_2^2
			= \sum_{i=1}^m (\rvx_{ci}-\bar\rvx_c)^\top\mW^\top\mW(\rvx_{ci}-\bar\rvx_c), 
		\end{aligned}
	\end{equation}
    where $ \mX_c = \mX_1, \mX_2 $, or even\footnote{In this case, we choose $ \bar \rvx $ as $\bar \rvx_c $ in Eqn.\ref{eqn:29} } $ [\mX_1,\mX_2] $.

	Based on the eigenvalue decomposition, we obtain that
		\begin{equation}
			\begin{aligned}
				SS(\mW\mX_1) 
				=& \sum_{i=1}^m (\rvx_{1i}-\bar\rvx_1)^\top\mW^\top\mW(\rvx_{1i}-\bar\rvx_1) \\
				=& \sum_{i=1}^m \sum_{j=1}^d \lambda_j (\rvx_{1i}-\bar\rvx_1)^\top \vu_j \vu_j^\top (\rvx_{1i}-\bar\rvx_1)\\
				=& \sum_{j=1}^d \lambda_j \sum_{i=1}^m (\vu_j^\top \rvx_{1i} - \vu_j^\top\bar\rvx_1)^2\\
				=&\sum_{j=1}^d \lambda_j SS(\vu_j^\top\mX_1).
			\end{aligned}
		\end{equation}
		The term $ SS(\vu_j^\top\mX_1) $ can be regarded as that we put a linear transformation $ \vu_j^\top $ on $ \mX_1 $, and then calculate its SS. 
		
		Similarly, we have that
        \begin{equation}
            SS(\mW\mX_2) = \sum\limits_{j=1}^d \lambda_j SS(\vu_j^\top \mX_2) , 
        \end{equation}
        and 
        \begin{equation}
            SS([\mW\mX_1, \mW\mX_2]) = \sum\limits_{j=1}^d \lambda_j SS([\vu_j^\top\mX_1, \vu_j^\top\mX_2]) . 
        \end{equation}
        Therefore, we obtain
		\begin{equation}
			\begin{aligned}
				SSR(\mW\mX_1,\mW\mX_2) = \frac{\sum\limits_{j=1}^d \lambda_j [SS(\vu_j^\top\mX_1) + SS(\vu_j^\top \mX_2)]} {\sum\limits_{j=1}^d \lambda_j SS([\vu_j^\top\mX_1, \vu_j^\top \mX_2])}
			\end{aligned}.
		\end{equation}
		By the definition of $\mM$ and $\mN$ in Proposition \ref{prop:LSSR}, we obtain that
		\begin{equation}
            \label{eqn:32}
			\begin{aligned}
			SS(\vu_j^\top\mX_1) + SS(\vu_j^\top \mX_2) &= \sum_{i=1}^m [(\vu_j^\top \rvx_{1i} - \vu_j^\top\bar\rvx_1)^2 + (\vu_j^\top \rvx_{2i} - \vu_j^\top\bar\rvx_2)^2] \\
				& = \sum_{i=1}^m \vu_j^\top[(\rvx_{1i} - \bar\rvx_1)(\rvx_{1i} - \bar\rvx_1)^\top + (\rvx_{2i} - \bar\rvx_2)(\rvx_{2i} - \bar\rvx_2)^\top]\vu_j \\
				&= \vu_j^\top \mM \vu_j, 
			\end{aligned}
		\end{equation}
		and similarly, we have
		\begin{equation}
            \label{eqn:33}
			SS([\vu_j^\top\mX_1, \vu_j^\top \mX_2]) = \vu_j^\top\mN\vu_j. 
		\end{equation}
		
		By the hypothesis in Definition \ref{def:LSSR}, we figure out that $ \lambda_j(j=1,\cdots,d) $ are not all zeros, otherwise $ SS([\mW \mX_1, \mW \mX_2]) = 0 $. Besides, by the hypothesis in Proposition \ref{prop:LSSR}, $ \mN $ is reversible. We point out that $ \mN $ is also a positive semi-definite matrix.
  
        Furthermore, $ \mN $ is a positive definite matrix. When $ \vu_j \ne \bm0 $, we find that
        \begin{equation}
            SS([\vu_j^\top\mX_1, \vu_j^\top \mX_2]) = \vu_j^\top\mN\vu_j > 0 . 
        \end{equation}
		
		Let $ \eta_j = \lambda_j SS([\vu_j^\top\mX_1, \vu_j^\top \mX_2]) $, we thus have $ \eta_1+\cdots+\eta_d > 0 $. We obtain
		\begin{equation}
			\begin{aligned}
				SSR(\mW\mX_1,\mW\mX_2)
				=& \frac{1}{\eta_1+\cdots+\eta_d} \sum_{j=1}^d \frac {\eta_j [SS(\vu_j^\top\mX_1) + SS(\vu_j^\top \mX_2)]}  {SS([\vu_j^\top\mX_1, \vu_j^\top \mX_2])}\\
				=& \sum_{j=1}^d \frac{\eta_j}{\eta_1+\cdots+\eta_d} SSR(\vu_j^\top\mX_1, \vu_j^\top\mX_2)\\
				\ge& \sum_{j=1}^d \frac{\eta_j}{\eta_1+\cdots+\eta_d} \min_{1\le k\le d}SSR(\vu_k^\top\mX_1, \vu_k^\top\mX_2)\\
                =& \min_{1\le k\le d} SSR(\vu_k^\top\mX_1,  \vu_k^\top\mX_2)\\
				=& \min_{1\le j\le d} SSR(\vu_j^\top\mX_1,  \vu_j^\top\mX_2).
			\end{aligned}
		\end{equation}
		We figure out that the equation holds, if and only if
		\begin{equation}
			\eta_j \begin{cases}
				\ge0, &j \in\arg\min\limits_{1\le j\le d}SSR(\vu_j^\top\mX_1, \vu_j^\top\mX_2);\\
				=0, &\text{otherwise}.
			\end{cases}
		\end{equation}
		Here, $ j=1,\cdots,d $, and $ \eta_1,\cdots,\eta_d $ are not all zeros. 
		
		Since $ \lambda_j = \eta_j / SS([\vu_j^\top\mX_1, \vu_j^\top \mX_2]) $ and $ SS([\vu_j^\top\mX_1, \vu_j^\top \mX_2]) > 0 $, we thus have
		\begin{equation}
			\lambda_j \begin{cases}
				\ge0, &j \in\arg\min\limits_{1\le j\le d}SSR(\vu_j^\top\mX_1, \vu_j^\top\mX_2),\\
				=0, &\text{otherwise},
			\end{cases}
		\end{equation}
		holds for $ j=1,\cdots,d$, and $ \lambda_1,\cdots,\lambda_d $ are not all zeros. 
	\end{proof}
	
	\begin{lemma}
	    \label{lemma:3}
        Given $ \sD_\vv = \{ \vv:\vv^\top\mN\vv = 1\} $ and $ \sD_{\vu} = \{\vu:\vu^\top\vu=1\}$ and the map $ \psi : \sD_\vv \to \sD_\vu $, where $ \vu = \psi(\vv) = \vv/(\vv^\top\vv)^{\frac12} $, we have that $ \psi $ is a bijection. 
	\end{lemma}

    \begin{proof}
        For $\mN$ is a positive definite matrix , and $ \vv^\top\mN\vv = 1 $, we have $ \vv\ne\bm0 $. Given $ \vu = \psi(\vv) $, we obtain 
        \begin{equation}
            \vu^\top\vu = \vv^\top\vv/(\vv^\top\vv) =1, 
        \end{equation} 
        for each $ \vv $ in $ \sD_\vv $. 
        
        Therefore, $ \psi $ is a reflection from $ \sD_\vv $ to $ \sD_\vu $. Besides, we find that
        \begin{equation}
            \vu^\top\mN\vu = \vv^\top\mN\vv/(\vv^\top\vv) = 1/(\vv^\top\vv). 
        \end{equation}
        By the definition of $\psi$, we hence have 
        \begin{equation}
            \vu/(\vu^\top\mN\vu)^{\frac12} = \psi(\vv)(\vv^\top\vv)^{\frac12} = \vv . 
        \end{equation}
        Therefore, for each $\vu$, we obtain that
        \begin{equation}
            \vv = \psi^{-1}(\vu) = \vu/(\vu^\top\mN\vu)^{\frac12} , 
        \end{equation}
        namely we find $\psi^{-1}$ as the inverse mapping of $\psi$. 
        
        As a result, $ \psi $ is a bijection. 
    \end{proof}

    \begin{lemma}
        \label{lemma:4}
        Let the optimization problem be
        \begin{equation}
            (P_\vv)\begin{cases}
            \min\limits_\vv\quad &\displaystyle f(\vv) = \frac{\vv^\top \mM\vv}{\vv^\top \mN\vv}, \\
            s.t. \quad &\vv^\top\mN\vv = 1,
        \end{cases}
        \end{equation}
        where $\mM$ and $\mN$ are defined in Proposition \ref{prop:LSSR}. We have that the optimal value is the minimal eigenvalue of $ \mN^{-1} \mM $, namely $ \lambda^* $. And the optimal solution is the eigenvector which belongs to $ \lambda^* $.
    \end{lemma}

    \begin{proof}
         To get the minimum, we use the Lagrange multiplier method: 
		\begin{equation}
			L(\vv,\alpha) = \vv^\top \mM\vv - \alpha(\vv^\top \mN\vv-1).
		\end{equation}
		We figure out that the KKT conditions are
        \begin{equation}
            \label{eqn:44}
            \begin{cases}
                \displaystyle\frac{\partial L}{\partial \vv} = 2\mM\vv - 2\alpha\mN\vv = \bm0,\\ 
                \vv^\top \mN\vv-1 = 0.
            \end{cases}
        \end{equation}  
        We hence have $ \mN^{-1} \mM\vv = \alpha\vv $, namely $\alpha$ is an eigenvalue of $\mN^{-1} \mM$, and $ \vv $ is the corresponding eigenvector. Based above, we find that
        \begin{equation}
            \vv^\top\mM\vv = \vv^\top\mN\mN^{-1}\mM\vv =  \vv^\top\mN(\alpha\vv) = \alpha. 
        \end{equation}
        Furthermore, the minimum of $ L(\vv,\alpha) $ is the minimum $\alpha$, namely the minimum eigenvalue of $\mN^{-1}\mM$.
        
        We hence have
		\begin{equation}
			LSSR(\mX_1, \mX_2) = \lambda_{min}(\mN^{-1}\mM)=\lambda^*, 
		\end{equation}
        and the optimal solution is the eigenvector which belongs to $ \lambda^* $.
    \end{proof}

    \subsection{Proof of Proposition \ref{prop:LSSR}}

    Based on the four lemmas above, now we give the proof of Proposition \ref{prop:LSSR}. 

    \begin{proof}
        By Lemma \ref{lemma:21} and Lemma \ref{lemma:2}, given $ \mW^\top\mW = \mU\Lambda\mU^\top $, we have 
        \begin{equation}
            \label{eqn:38}
            \begin{aligned}
                LSSR(\mX_1, \mX_2) &= \min_{\mW}SSR (\mW\mX_1, \mW\mX_2) \\ 
                &= \min_{\mU,\Lambda}SSR (\mW\mX_1, \mW\mX_2) \\
                &= \min_{\mU} \min_{\Lambda} SSR (\mW\mX_1, \mW\mX_2) \\
                &= \min_{\mU}\min\limits_{1\le j\le d} SSR(\vu_j^\top\mX_1,  \vu_j^\top\mX_2).
            \end{aligned}
        \end{equation}
        
        According to \Eqn\ref{eqn:32} and \Eqn\ref{eqn:33}, we define the function $ f(\vu) = SSR(\vu^\top\mX_1,  \vu^\top\mX_2) $, namely
        \begin{equation}
            f(\vu) = \frac{\vu^\top \mM\vu}{\vu^\top \mN\vu}.
        \end{equation}
        
        By \Eqn\ref{eqn:38}, there is some $\mU$, and $ j^* = \arg \min \limits_{1\le j\le d} SSR(\vu_j^\top\mX_1,  \vu_j^\top\mX_2) $, such that 
        \begin{equation}
            LSSR(\mX_1,\mX_2) = \min \limits_{1\le j\le d} SSR(\vu_j^\top\mX_1,  \vu_j^\top\mX_2) = f(\vu_{j^*}). 
        \end{equation} 
        
        Consider the optimization problem 
        \begin{equation}
            (P_{\vu})\begin{cases}
            \min\limits_\vu\quad &f(\vu) = \displaystyle\frac{\vu^\top \mM\vu}{\vu^\top \mN\vu}, \\
            s.t. \quad &\vu^\top\vu = 1.
            \end{cases}
        \end{equation}

        We denote one of the optimal solutions as $ \bar\vu $. Obviously, $ \vu_{j^*} $ is one of the feasible solutions above, we thus have 
    \begin{equation}
        \label{eqn:42}
        LSSR(\mX_1,\mX_2) = f(\vu_{j^*}) \ge f(\bar\vu).
    \end{equation} 

    \REVISE{We remind that $ \lambda^* $ and  $\vu^* $ are the minimum eigenvalue and corresponding eigenvector of $ \mN^{-1} \mM $. }On one hand, let $ \vu_0 = \vu^* / \|\vu^*\|_2 $ and $ \vv_0 = \psi^{-1}(\vu_0) $ ($\psi$ is defined the same as that in Lemma \ref{lemma:3}), namely $ \vv_0 = \vu_0/(\vu_0^\top \mN \vu_0)^{\frac12} $. We first point out that for $k\ne0$, we have
    \begin{equation}
        \label{eqn:43} 
        \begin{aligned}
            f(k\vu) &= \frac{(k\vu)^\top \mM(k\vu)}{(k\vu)^\top \mN(k\vu)} \\
            &= \frac{\vu^\top \mM\vu}{\vu^\top \mN\vu} \\
            &= f(\vu). 
        \end{aligned}
    \end{equation}
    Since $1/||\vu^*||_2\ne0$ and $1/ (\vu_0^\top \mN \vu_0)^{\frac12}\ne0$, we obtain
    \begin{equation}
        \label{eqn:56}
        \begin{aligned}
            f(\vv_0) = f(\vu_0)&= f(\vu^*) \\
            &= \frac{(\vu^*)^\top \mM(\vu^*)}{(\vu^*)^\top \mN(\vu^*)} \\
            &= \frac{(\vu^*)^\top \mN \mN^{-1} \mM(\vu^*)}{(\vu^*)^\top \mN(\vu^*)} \\
            &= \frac{(\vu^*)^\top \mN (\lambda^*\vu^*)}{(\vu^*)^\top \mN(\vu^*)} \\
            &= \lambda^*, 
        \end{aligned}
    \end{equation}
   
    where $ \lambda^* $ is the minimal eigenvalue of $\mN^{-1}\mM$, as shown in Proposition \ref{prop:LSSR}. 

    Therefore, by Lemma \ref{lemma:4}, $\vv_0$ is the optimal solution of $(P_\vv)$. Furthermore, by Lemma \ref{lemma:3}, since $\psi$ is a bijection between $\sD_\vv$ and $\sD_\vu$ and $f(\psi(\vv)) = f(\vv)$, we have that $ \vu_0 = \psi(\vv_0) $ is also the optimal solution of $(P_\vu)$. We hence have $ f(\vu^*) = f(\vu_0) = f(\bar\vu) $. By \Eqn\ref{eqn:42}, we obtain
    \begin{equation}
        \label{eqn:51} 
        LSSR(\mX_1,\mX_2) \ge f(\bar\vu) = f(\vu^*) = SSR((\vu^*)^\top\mX_1, (\vu^*)^\top\mX_2). 
    \end{equation}
    
    On the other hand, the definition of LSSR denotes the lower bound of SSR, we hence have
    \begin{equation}
        \label{eqn:45}
        LSSR(\mX_1,\mX_2) \le SSR((\vu^*)^\top\mX_1, (\vu^*)^\top\mX_2).
    \end{equation} 
    By \Eqn\ref{eqn:56}, \Eqn\ref{eqn:51} and \Eqn\ref{eqn:45}, we obtain 
    \begin{equation}
        LSSR(\mX_1,\mX_2) = SSR((\vu^*)^\top\mX_1, (\vu^*)^\top\mX_2) = \lambda^*.
    \end{equation}
    \end{proof}
    
    \subsection{Corollaries of Proposition \ref{prop:LSSR}}
 
    Suppose $ \lambda^* $ is the minimal eigenvalue of $ \mN^{-1}\mM $, and $ \vu^* $ is its unique linearly independent eigenvector, we give the result in Corollary \ref{prop:4}. If $\lambda^*$ has more than one linearly independent eigenvectors, we give the result in Corollary \ref{prop:5}.
	
	\begin{corollary}
		\label{prop:4}
		Suppose $ \lambda^* $ is the minimal eigenvalue of $ \mN^{-1}\mM $, and $ \vu^* $ is its unique linearly independent eigenvector, we have that the optimal $ \mW $ satisfies that
		\begin{equation}
			\mW^\top\mW = C \vu^*(\vu^*)^\top, C>0.
		\end{equation}
	\end{corollary}
	
	\begin{proof}
		By Lemma \ref{lemma:21}, we can only consider the eigenvalues and eigenvectors of $ \mW^\top\mW $. By \Eqn\ref{eqn:25}, $\vu_j$ will affect $ \mW^\top\mW $ only when the corresponding eigenvalue $ \lambda_j \ne 0 $. Furthermore, by Lemma \ref{lemma:2}, when $ \lambda_j \ne 0 $, we figure out that
        \begin{equation}
            j \in \arg \min \limits_{1\le j\le d}SSR(\vu_j^\top\mX_1, \vu_j^\top\mX_2) . 
        \end{equation}
        Since $ \vu_j $ is a unit vector, it must be one of the optimal solutions of $ (P_\vu) $. We hence have $ f(\vu_j) = \lambda^* $. By \Eqn\ref{eqn:43} and Lemma \ref{lemma:3}, we have
        \begin{equation}
            f(\psi^{-1}(\vu_j)) = f(\vu_j) = \lambda^* , 
        \end{equation} 
        and $ \psi^{-1}(\vu_j) $ is one of the optimal solutions of $ (P_\vv) $. By Lemma \ref{lemma:4}, we have that $ \psi^{-1}(\vu_j) $ must satisfy the KKT conditions in \Eqn\ref{eqn:44}, namely
        \begin{equation}
            \mN^{-1} \mM \psi^{-1}(\vu_j) = \lambda^* \psi^{-1}(\vu_j) .
        \end{equation}
        Therefore, $ \psi^{-1}(\vu_j) $ is the minimal eigenvector of $\mM^{-1}\mN$. For $ \psi^{-1}(\vu_j) = \vu_j/(\vu_j^\top\mN\vu_j)^{\frac12} $, we obtain $ \vu_j $ is also the eigenvector of $ \lambda^* $. For $ \vu^* $ is the unique linearly independent eigenvector of $ \lambda^* $, we figure out that
        \begin{equation}
            \label{eqn:49}
            \vu_j = \alpha_j \vu^*,\alpha_j\ne0. 
        \end{equation}
        To be reminded, \Eqn\ref{eqn:49} only holds when $ \lambda_j \ne 0 $. Therefore, we have
        \begin{equation}
            \begin{aligned}
                \mW^\top\mW &= \sum_{j=1}^d \lambda_j \vu_j\vu_j^\top \\
                &= \sum_{\lambda_j \ne 0} \lambda_j \alpha_j^2 \vu^*(\vu^*)^\top \\
                &= C \vu^*(\vu^*)^\top,
            \end{aligned}
        \end{equation}
        where $ C = \sum\limits_{\lambda_j \ne 0} \lambda_j \alpha_j^2 $. 

        By Lemma \ref{lemma:2}, we have $ \lambda_1,\cdots,\lambda_j\ge 0 $ are not all zeros. Besides, we figure out that $\alpha_j$ can be any non-zero real number. We thus have $C$ can be any non-zero real number, to demonstrate $\mW^\top\mW$. 
	\end{proof}
	
	\begin{corollary}
		\label{prop:5}
		Suppose that the minimal eigenvalue of $ \mN^{-1}\mM $, namely $\lambda^*$, has $ k $ linearly independent eigenvectors $ \vv_1,\vv_2,\cdots,\vv_k $. We denote $\mV = [\vv_1,\cdots,\vv_k] $, then we have that the optimal $ \mW $ satisfies that 
		\begin{equation}
			\mW^\top\mW = \mV\mC\mV^\top,
		\end{equation}
		where $\mC$ is a $k$-order semi-positive definite and non-zero matrix.
	\end{corollary}
	
	\begin{proof}
		Suppose $\lambda_j$ is an eigenvalue of $\mW^\top\mW$, and its eigenvector is $\vu_j$. We can identify that $ j \in \arg \min \limits_{1\le j\le d}SSR(\vu_j^\top\mX_1, \vu_j^\top\mX_2) $, if $\lambda_j\ne0 $. Similarly to the proof of Corollary \ref{prop:4}, $\vu_j$ must be an eigenvector of $\mN^{-1}\mM$, and the corresponding eigenvalue is $\lambda^*$. Accordingly, $\vu_j$ is a linear combination of all the linearly independent eigenvectors of $\lambda^* $, namely
        \begin{equation}
            \label{eqn:57}
            \vu_j = \alpha_{j1}\vv_1 + \alpha_{j2}\vv_2 + \cdots + \alpha_{jk}\vv_k = \mV\bm\alpha_j
        \end{equation}
        where $ \bm\alpha_j = [\alpha_{j1}, \cdots, \alpha_{jk}]^\top $, and $ \bm\alpha_j \ne \bm0 $.
        
        We remind that \Eqn\ref{eqn:57} only holds when $ \lambda_j \ne 0 $. We thus have
		\begin{equation}
			\label{eqn:k eigen W}
			\begin{aligned}
				\mW^\top\mW &= \sum_{j=1}^d\lambda_j\vu_j\vu_j^\top \\
				&= \sum_{\lambda_j>0} \lambda_j\vu_j \vu_j^\top \\
				&= \sum_{\lambda_j>0} \lambda_j (\mV\bm\alpha _j)(\mV\bm\alpha _j)^\top \\
				&= \mV\left(\sum_{\lambda_j>0} \lambda_j \bm\alpha_j \bm\alpha_j^\top \right)\mV^\top \\
                &= \mV\mC\mV^\top,
			\end{aligned}
		\end{equation}
		where $\mC = \sum\limits_{\lambda_j>0} \lambda_j \bm\alpha_j \bm\alpha_j^\top $. 
  
        By Lemma \ref{lemma:2}, we have $ \lambda_1,\cdots,\lambda_j\ge 0 $ are not all zeros. Besides, we figure out that $\bm\alpha_j$ can be any non-zero vector. We thus have $\mC$ is any $k$-order semi-positive definite and non-zero matrix, to demonstrate $\mW^\top\mW$. 
	\end{proof}

    \section{Proof Related to Breaking LSSR}
    \label{section:LSSRbreak}
    In this section, we prove Theorem \ref{thm:break} from the perspective of Taylor's expansion. 


   We have defined $f_{SSR}(t)$ as
    \begin{equation}
    	f_{SSR}(t) = \begin{cases}
    		LSSR(\mX_1, \mX_2), &t=0, \\
    		SSR(\bar\psi(t;\mX_1), \bar\psi(t;\mX_2)), &t\ne0, 
    	\end{cases}
    \end{equation}
    where $ \bar\psi(t;\rvx_{ci}) =  \bm1^\top\bar\varphi(t;\rvx_{ci}) / \| \bar\varphi(t;\rvx_{ci}) \|_2 $, $ \bar\varphi(t;\rvx_{ci}) = [(\vu^*)^\top\rvx_{ci}t, 1]^\top $ and $ t\in \R $. We remind Theorem \ref{thm:break} as below. 

    \paragraph{Theorem 1.} 
    \textit{Let $ \psi = \varphi_1 \circ LN(\cdot) \circ \varphi_2 $, performing over the input $ \mX_1, \mX_2 \in \R^{d \times m} $. If $ f_{SSR}'(0) \ne 0 $ , we can always find suitable linear functions $ \varphi_1$ and $\varphi_2 $, such that
	\begin{equation}
		SSR(\psi(\mX_1), \psi(\mX_2)) < LSSR(\mX_1, \mX_2). 
	\end{equation}}

    We prove Lemma \ref{lemma:equivalence} and show three extra lemmas before the formal proof of Theorem \ref{thm:break}. 

    \subsection{Proof of Lemma \ref{lemma:equivalence}}

    \paragraph{Lemma 1.}\textit{Denote $ LN(\cdot) $ as the LN operation in $ \R^d (d\ge3) $, and $ SP(\cdot) $ as the SP operation\footnote{If there are no special instructions, we denote SP projects the sample on to the unit circle, namely $ \rvx \mapsto \rvx/\|\rvx\|_2 $. } in $ \R^{d-1} $. We can find some linear transformations $ \hat \varphi_1 $ and $ \hat \varphi_2 $, such that 
		\begin{equation}
			SP(\cdot) = \hat \varphi_1 \circ LN(\cdot) \circ \hat \varphi_2. 
		\end{equation} }
  
    We denote that $ SP(\cdot) $ is defined on $ \R^{d-1} $, as
    \begin{equation}
        SP(\rvx) = \rvx/\|\rvx\|_2 .
    \end{equation} 
    While $ LN(\cdot) $ is defined on $ \R^{d} $, where
    \begin{equation}
        LN(\rvx) = \sqrt{d}\ (\rvx - \frac1d \bm 1 \bm 1^\top \rvx)/\|\rvx - \frac1d \bm 1 \bm 1^\top \rvx\|_2 . 
    \end{equation} 

    Before the proof of Lemma \ref{lemma:equivalence}, we propose Lemma \ref{lemma:5} as follows. 
	
	\begin{lemma}
		\label{lemma:5}
		There is some orthogonal matrix $ \mQ \in \R^{d \times d} $, such that $ \vz = \mQ\mat{\rvx\\0} \in \{\vz \in \R^d:z^{(1)} + \cdots + z^{(d)} = 0\} $ (namely $ \vz $ is centralized), for $ \rvx \in \R^{d-1} $,
	\end{lemma}
	
	\begin{proof}
		Suppose $ \mQ = \{q_{ij}\}_{d\times d} = [\vq_1, \vq_2, ...,\vq_d] $. We take $ \vq_d = \frac{1}{\sqrt{d}} \bm1 $ specially, and $ \vq_1, \cdots, \vq_{d-1} $ can be calculated by Schmidt orthogonalization. 
		
		Given $ \rvx = [x^{(1)}, \cdots, x^{(d-1)}]^\top \in \mathbb R^{d-1}$, we have
        \begin{equation}
            \vz = \mQ \mat{ \rvx \\ 0 } = [z^{(1)}, \cdots, z^{(d)}]^\top .
        \end{equation}
        Since $ \mQ $ is an orthogonal matrix, we have
        \begin{equation}
            \vq_i^\top \vq_d = \frac{1}{\sqrt{d}}\sum\limits_{k=1}^{d} q_{ki}=0 , 
        \end{equation} 
        for $ i=1, \cdots, d-1 $. Furthermore, we obtain that 
		\begin{equation}
			\begin{aligned}
				\sum_{k=1}^d z^{(k)} &= \sum_{k=1}^d\left(\sum_{i=1}^{d} q_{ki}x^{(i)}\right)\\
                &= \sum_{k=1}^d\left(q_{kd}\cdot0 +  \sum_{i=1}^{d-1} q_{ki}x^{(i)}\right)\\
				&= \sum_{i=1}^{d-1}\left( \sum_{k=1}^d q_{ki}\right)x^{(i)} \\
				&= 0,
			\end{aligned}
		\end{equation}
		which shows that $ \vz \in \{\vz \in \R^d:z^{(1)} + \cdots + z^{(d)} = 0\} $, namely $ \vz $ is centralized. 
	\end{proof}
	
	Now we can design $ \hat\varphi_1 $ and $ \hat\varphi_2 $ based on $ \mQ $ in Lemma \ref{lemma:5}, and then prove Lemma \ref{lemma:equivalence}. 
	
	\begin{proof}
		Based above, we obtain
        \begin{equation}
            \|\vz\|_2 = \left\Vert \mQ \mat{\rvx \\ 0} \right\Vert_2 = \left\Vert \mat{\rvx \\ 0} \right\Vert_2 = \Vert \rvx \Vert_2 .
        \end{equation}
        By Lemma \ref{lemma:5}, we have $ \bm1^\top \vz = 0 $, and $ \vz = \vz - \frac1d\bm1\bm1^\top\vz $. We hence find that
        \begin{equation}
            LN(\vz) = \sqrt{d} (\vz - \frac1d\bm1 \bm1^\top\vz) / \|\vz - \frac1d\bm1 \bm1^\top\vz\|_2 = \sqrt{d}\ \vz/\|\vz\|_2 . 
        \end{equation}
        Let $ \mI_{d} $ denotes the identity matrix in $ \R^{d\times d} $. We thus have
		\begin{equation}
			\begin{aligned}
				\frac{1}{\sqrt{d}}\mat{\mI_{d-1}&\bm0}\mQ^\top LN(\mQ\mat{\mI_{d-1} & \bm0}^\top \rvx) 
				=&\frac{1}{\sqrt{d}}\mat{\mI_{d-1}&\bm0}\mQ^\top LN(\vz) \\
				=&\frac{1}{\sqrt{d}}\mat{\mI_{d-1}&\bm0}\mQ^\top \sqrt{d}\ \vz/\|\vz\|_2 \\
				=&\sqrt{d}\cdot \frac{1}{\sqrt{d}}\mat{\mI_{d-1}&\bm0} \mQ^\top \mQ \mat{\rvx\\0} / \|\rvx\|_2 \\
				=&\rvx / \|\rvx\|_2 \\
				=&SP(\rvx).
			\end{aligned}
		\end{equation}
	
	Let $ \hat\varphi_1(\rvx) = \mQ \mat{\mI_{d-1} & \bm0}^\top \rvx $, and $ \hat\varphi_2(\rvx) = \frac{1}{\sqrt{d}}\mat{\mI_{d-1} & \bm0} \mQ^\top \rvx $. We observe that
	\begin{equation}
		SP(\cdot) = \hat\varphi_1 \circ LN(\cdot) \circ \hat\varphi_2. 
	\end{equation}
    \end{proof}

    \subsection{Extra Lemmas for the Proof}

    Let $ x_{ci} = (\vu^*)^\top\rvx_{ci}, (i=1,\cdots,m;c=1,2) $. We define the mean $ \bar x_{c}$, the variance $\sigma_{c}^2$ and the third-order central moment $ \overline{(x_{c} - \bar x_{c})^3} $ with the equations below:
    \begin{equation}
        \begin{aligned}  
        &\bar x_{c} = \frac1m \sum \limits_{i=1}^m x_{ci}, \\
        &\sigma_{c}^2 = \frac1m \sum\limits_{i=1}^m (x_{ci}-\bar x_{c})^2, \\
        &\overline{(x_{c} - \bar x_{c})^3} = \frac1m \sum\limits_{i=1}^m (x_{ci}-\bar x_{c})^3.
        \end{aligned}
    \end{equation}
 
	Based on $ x_{ci} = (\vu^*)^\top\rvx_{ci} $, we design an linear transformation $ \varphi(t;\cdot) : \R \to \R^2 $, where $ t \in \R $ is a parameter: 
	\begin{equation}
        \label{eqn:55}
		\varphi(t;x_{ci}) = t\cdot\mat{1\\0} x_{ci} + \mat{0\\1} = \mat{x_{ci}t \\ 1}.
	\end{equation}
 
    \begin{lemma}
        \label{lemma:6}
        Let $ \hat \mX = SP(\varphi (t; (\vu^*)^\top \mX)) $, and $ \vv = [1,1]^\top $. Besides, we define three statistics about $ (\vu^*)^\top\mX$:
        \begin{equation}
        \begin{cases}
            T_1 = (\bar x_{1}-\bar x_{2})^2 [\overline{(x_{1}-\bar x_{1})^3} + \overline{(x_{2} - \bar x_{2})^3}], \\
            T_2 = (\bar x_{1} - \bar x_{2}) (\sigma_{1}^2-\sigma_{2}^2) [(\bar x_{1} - \bar x_{2})^2 - (\sigma_{1}^2 + \sigma_{2}^2)], \\
            T_3 = [2\sigma_1^2 + 2 \sigma_2^2 + (\bar x_1 - \bar x_2 )^2]^2. 
        \end{cases}
    \end{equation}
    We figure out that when $t\to0$, we have
    \begin{equation}
        SSR(\vv^\top\hat\mX_1, \vv^\top\hat\mX_2) = LSSR(\mX_1,\mX_2) - \frac{2(T_1+T_2)}{T_3}t + o(t). 
		\end{equation}
    \end{lemma}

    \begin{proof}
        Denote $ \hat\rvx_{ci} = SP(\varphi(t;x_{ci}))= [\hat x_{ci}^{(1)},\hat x_{ci}^{(2)}]^\top $. By Newton's binomial expansion, we obtain that
		\begin{equation}
			\begin{aligned}
				\frac {1} {\Vert \varphi(t; x_{ci}) \Vert_2}
				=& \frac {1} {\sqrt{1+(x_{ci}t)^2}} \\
				=& (1 + x_{ci}^2t^2)^{-\frac12}\\
				=& 1 - \frac12(x_{ci}^2t^2) + \frac38 (x_{ci}^2t^2)^2 + o((t^2)^2)\\
				=& 1 - \frac12x_{ci}^2t^2 + o(t^3).
			\end{aligned}
		\end{equation}
		We thus have 
		\begin{equation}
			\hat x_{ci}^{(1)} = \frac {x_{ci}t} {\sqrt{1+(x_{ci}t)^2}} = x_{ci}t-\frac12x_{ci}^3t^3 + o(t^3), 
		\end{equation}
		and 
        \begin{equation}
            \hat x_{ci}^{(2)} = \frac {1} {\sqrt{1+(x_{ci}t)^2}} = 1 - \frac12x_{ci}^2t^2 + o(t^3). 
        \end{equation} 
        Let $\vv =[1,1]^\top$. Then we have
        \begin{equation}
            \vv^\top\hat\rvx_{ci} = 1 + x_{ci}t - \frac12x_{ci}^2t^2 - \frac12x_{ci}^3t^3 +o(t^3) . 
        \end{equation}
        
        We denote that $ a_0=1, a_1=1, a_2=-\frac12 $ and $ a_3=-\frac12 $, therefore 
        \begin{equation}
            \vv^\top\hat\rvx_{ci} = \sum\limits_{s=0}^{3} a_sx_{ci}^st^s+o(t^3).  
        \end{equation}
        We hence obtain that
		\begin{equation}
			\begin{aligned}
				SS(\vv^\top\hat \mX_{c}) &= \sum_{i=1}^m(\vv^\top\hat\rvx_{ci}-\overline{\vv^\top\hat\rvx_c})^2\\
				&=\frac1m \sum_{i=1}^m \sum_{j=1}^m \vv^\top \hat\rvx_{ci} (\vv^\top\hat\rvx_{ci} - \vv^\top\hat\rvx_{cj})\\
				&=\frac1m \sum_{i=1}^m \sum_{j=1}^m \left[\left(\sum_{r=0}^{3} a_rx_{ci}^rt^r+o(t^3)\right)\cdot\left(\sum_{s=0}^{3} a_s(x_{ci}^s-x_{cj}^s)t^s+o(t^3)\right)\right]. 
			\end{aligned}
		\end{equation}
		For $ s+r >3 $, we put the multiplicative term into $ o(t^3) $. Accordingly, we only consider the case $ s+r\le3 $. 
		
		For $ r=0,1,2,3; s=0 $, we have \begin{equation}
			x_{ci}^r(x_{ci}^s-x_{cj}^s) = 0. 
		\end{equation}
		For $ r=0; s = 1,2,3 $, we have
		\begin{equation}
			\sum_{i=1}^m \sum_{j=1}^m x_{ci}^rt^r\cdot (x_{ci}^s-x_{cj}^s)t^s = \sum_{i=1}^m \sum_{j=1}^m (x_{ci}^s-x_{cj}^s)t^s = 0.
		\end{equation}
		For $ r=1, s=1 $, we have 
		\begin{equation}
			\sum_{i=1}^m \sum_{j=1}^m x_{ci}(x_{ci}-x_{cj}) = m \sum_{i=1}^m (x_{ci}-\bar x_c) = m\sigma_c^2. 
		\end{equation}
        For $ r=2, s=1 $, we observe
		\begin{equation}
			\begin{aligned}
				\sum_{i=1}^m \sum_{j=1}^m x_{ci}^2 (x_{ci}-x_{cj}) &= m \sum_{i=1}^m x_{ci}^2 (x_{ci}-\bar x_c) \\
				&= m \sum_{i=1}^m [(x_{ci}^2 - 2x_{ci}\bar x_c +\bar x_c^2) (x_{ci}-\bar x_c) + 2x_{ci}\bar x_c(x_{ci}-\bar x_c) - \bar x_c^2(x_{ci}-\bar x_c)]\\
				&= m^2 [\overline{(x_c-\bar x_{c})^3} + 2\bar x_c\sigma_c^2].  
			\end{aligned}
		\end{equation}
        And for $ r=1, s=2 $, we obtain
		\begin{equation}
			\begin{aligned}
				\sum_{i=1}^m \sum_{j=1}^m x_{ci} (x_{ci}^2-x_{cj}^2) &= \sum_{i=1}^m \sum_{j=1}^m x_{ci}^3 - x_{ci}x_{cj}^2 \\
				&= \sum_{i=1}^m \sum_{j=1}^m x_{ci}^3 - x_{ci}^2x_{cj} \\
				&= \sum_{i=1}^m \sum_{j=1}^m x_{ci}^2 (x_{ci}-x_{cj}) \\
				&= m^2 [\overline{(x_c-\bar x_{c})^3} + 2\bar x_c\sigma_c^2]. 
			\end{aligned}
		\end{equation}
		
		Therefore, we have that 
		\begin{equation}
            \label{eqn:68}
			\begin{aligned}
				SS(\vv^\top\hat \mX_{c}) &= \frac1m \sum_{i=1}^m \sum_{j=1}^m \left[a_1^2x_{ci}(x_{ci}-x_{cj})t^2 + a_1a_2x_{ci}^2(x_{ci}-x_{cj})t^3 + a_1a_2x_{ci}(x_{ci}^2-x_{cj}^2)t^3 + o(t^3)\right]\\
				&= ma_1^2\sigma_c^2t^2 + 2ma_1a_2[\overline{(x_c-\bar x_{ci})^3} + 2\bar x_c\sigma_c^2]t^3 + o(t^3) \\
				&= m\sigma_c^2t^2 - m[\overline{(x_c-\bar x_{ci})^3} + 2\bar x_c\sigma_c^2]t^3 + o(t^3) \\
				&= \beta_{c2} t^2 +\beta_{c3}t^3 +o(t^3), 
			\end{aligned}
		\end{equation}
		where $ \beta_{c2} = m\sigma_c^2 $, and $ \beta_{c3} = - m [\overline{(x_c - \bar x_{ci})^3} + 2\bar x_c \sigma_c^2] $. 
		
		To simplify the calculation, we define 
        \begin{equation}
            \begin{aligned}
                SS_D(\mX_1, \mX_2) &= SS([\mX_1, \mX_2]) - SS(\mX_1) - SS(\mX_2) \\
                &= \sum\limits_{c=1}^{2} \sum\limits_{i=1}^m (\rvx_{ci}-\bar{\rvx})^\top(\rvx_{ci}-\bar{\rvx}) - \sum\limits_{c=1}^{2} \sum\limits_{i=1}^m (\rvx_{ci}-\bar{\rvx}_{c})^\top(\rvx_{ci}-\bar{\rvx}_{c}) \\
                &= \sum\limits_{c=1}^{2} \sum\limits_{i=1}^m (\rvx_{ci}^\top\rvx_{ci} - \bar\rvx^\top\bar\rvx)  - \sum\limits_{c=1}^{2} \sum\limits_{i=1}^m (\rvx_{ci}^\top\rvx_{ci} - \bar\rvx_c^\top\bar\rvx_c) \\
                & = m \bar\rvx_1^\top\bar\rvx_1 + m \bar\rvx_2^\top\bar\rvx_2 - 2m \bar\rvx^\top\bar\rvx \\
                & = m \bar\rvx_1^\top\bar\rvx_1 + m \bar\rvx_2^\top\bar\rvx_2 - \frac m2 (\bar\rvx_1 + \bar\rvx_2)^\top(\bar\rvx_1 + \bar\rvx_2) \\
                &= \frac m2 \| \bar\rvx_1 - \bar\rvx_2 \|_2^2. 
            \end{aligned}
        \end{equation}
        Similar to Eqn.\ref{eqn:68}, we obtain 
		\begin{equation}
			\begin{aligned}
				SS_D(\vv^\top\hat\mX_1, \vv^\top\hat\mX_2) 
				=&\frac m2 (\overline{\vv^\top \hat\rvx_1} - \overline{\vv^\top \hat\rvx_2})^2\\
				=&\frac m2 [a_1(\bar x_1 - \bar x_2)t + a_2(\overline{x_1^2} - \overline{x_2^2})t^2 +o(t^2)]^2\\
				=&\frac m2 [a_1^2(\bar x_1 - \bar x_2)^2t^2 + 2a_1a_2(\bar x_1 - \bar x_2)(\overline{x_1^2} - \overline{x_2^2})t^3 +o(t^3)]\\
				=&\frac m2(\bar x_1 - \bar x_2)^2t^2 - \frac m2(\bar x_1 - \bar x_2)(\overline{x_1^2} - \overline{x_2^2})t^3 +o(t^3) \\
				=&\beta_2t^2 + \beta_3t^3 + o(t^3), 
			\end{aligned}
		\end{equation}
		where $ \beta_{2} = \frac m2(\bar x_1 - \bar x_2)^2 $, and $ \beta_{3} = - \frac m2(\bar x_1 - \bar x_2)(\overline{x_1^2} - \overline{x_2^2}) $. 
		
		We thus have 
		\begin{equation}
			\begin{aligned}
				SSR(\vv^\top\hat\mX_1,\vv^\top\hat\mX_2)&=\frac{(\beta_{12}+\beta_{22})t^2+(\beta_{13}+\beta_{23})t^3+o(t^3)}{(\beta_{12}+\beta_{22}+\beta_2)t^2+(\beta_{13}+\beta_{23}+\beta_3)t^3+o(t^3)}\\
				&=\frac{(\beta_{12}+\beta_{22})t^2+(\beta_{13}+\beta_{23})t^3+o(t^3)}{(\beta_{12}+\beta_{22}+\beta_2)t^2\left[1+\frac{\beta_{13}+\beta_{23}+\beta_3}{\beta_{12} + \beta_{22} + \beta_2}t + o(t)\right]}\\
				&=\left[\frac{\beta_{12}+\beta_{22}}{\beta_{12}+\beta_{22}+\beta_2}+\frac{\beta_{13}+\beta_{23}}{\beta_{12}+\beta_{22}+\beta_2}t+o(t)\right]\cdot \left[ 1 - \frac{\beta_{13} + \beta_{23} + \beta_3}{\beta_{12} + \beta_{22} + \beta_2}t + o(t)\right]\\
				&=\frac{\beta_{12}+\beta_{22}}{\beta_{12}+\beta_{22}+\beta_2} + \frac{(\beta_{12} + \beta_{22} + \beta_2)(\beta_{13}+\beta_{23}) - (\beta_{12}+\beta_{22})(\beta_{13} + \beta_{23} + \beta_3)}{(\beta_{12}+\beta_{22}+\beta_2)^2} t+o(t)\\
				&=\frac{\beta_{12}+\beta_{22}}{\beta_{12}+\beta_{22}+\beta_2} + \frac{\beta_2(\beta_{13}+\beta_{23}) - \beta_3(\beta_{12}+\beta_{22})}{(\beta_{12}+\beta_{22}+\beta_2)^2} t+o(t). 
			\end{aligned}
		\end{equation}
		We find that 
		\begin{equation}
			\frac{\beta_{12}+\beta_{22}}{\beta_{12}+\beta_{22}+\beta_2} = SSR((\vu^*)^\top\mX_1,(\vu^*)^\top\mX_2) = LSSR(\mX_1,\mX_2). 
		\end{equation}
		On the other hand, we have
		\begin{equation}
			\begin{aligned}
				&\beta_2(\beta_{13}+\beta_{23}) - \beta_3(\beta_{12}+\beta_{22}) \\
				=& -\frac{1}{2}m^2(\bar x_1-\bar x_2)^2[\overline{(x_1-\bar x_{1})^3} + 2\bar x_1\sigma_1^2 + \overline{(x_2-\bar x_{2})^3} + 2\bar x_2\sigma_2^2] + \frac12m^2(\bar x_1 - \bar x_2)(\overline{x_1^2} - \overline{x_2^2})(\sigma_1^2+\sigma_2^2)\\
				=&-\frac12m^2(\bar x_1-\bar x_2)^2[\overline{(x_1-\bar x_{1})^3} +  \overline{(x_2-\bar x_{2})^3}]\\
                &-\frac12m^2(\bar x_1-\bar x_2)^2(2\bar x_1\sigma_1^2 + 2\bar x_2\sigma_2^2) + \frac12m^2(\bar x_1 - \bar x_2)(\overline{x_1^2} - \overline{x_2^2})(\sigma_1^2+\sigma_2^2). 
			\end{aligned}
		\end{equation}
		We figure out that 
		\begin{equation}
			\begin{aligned}
				& (\bar x_1-\bar x_2)^2(2\bar x_1\sigma_1^2 + 2\bar x_2\sigma_2^2) - (\bar x_1 - \bar x_2)(\overline{x_1^2} - \overline{x_2^2}) (\sigma_1^2+\sigma_2^2) \\
				=& (\bar x_1-\bar x_2)[(\bar x_1-\bar x_2)(2\bar x_1\sigma_1^2 + 2\bar x_2\sigma_2^2) - (\bar x_1^2 + \sigma_1^2 - \bar x_2^2 - \sigma_2^2) (\sigma_1^2+\sigma_2^2)] \\
				=& (\bar x_1-\bar x_2)[2\bar x_1(\bar x_1 - \bar x_2)\sigma_1^2 + 2\bar x_2(\bar x_1 - \bar x_2)\sigma_2^2 - (\bar x_1^2-\bar x_2^2)\sigma_1^2 - (\bar x_1^2 - \bar x_2^2)\sigma_2^2 - (\sigma_1^2-\sigma_2^2) (\sigma_1^2+\sigma_2^2)] \\
				=& (\bar x_1-\bar x_2)[(\bar x_1-\bar x_2)^2\sigma_1^2 - (\bar x_1-\bar x_2)^2\sigma_2^2 - (\sigma_1^2-\sigma_2^2) (\sigma_1^2+\sigma_2^2)] \\
				=& (\bar x_1-\bar x_2) (\sigma_1^2-\sigma_2^2)[(\bar x_1-\bar x_2)^2- (\sigma_1^2+\sigma_2^2)]. 
			\end{aligned}
		\end{equation}
		By the definition of $ T_1, T_2 $ and $T_3$, we thus obtain
		\begin{equation}
			\begin{aligned}
				&\beta_2(\beta_{13}+\beta_{23}) - \beta_3(\beta_{12}+\beta_{22}) \\
				=& -\frac12m^2(\bar x_1-\bar x_2)^2[\overline{(x_1-\bar x_{1})^3} +  \overline{(x_2-\bar x_{2})^3}]-\frac12m^2(\bar x_1-\bar x_2) (\sigma_1^2-\sigma_2^2)[(\bar x_1-\bar x_2)^2- (\sigma_1^2+\sigma_2^2)]\\
				=& -\frac12 m^2 T_1 - \frac12 m^2 T_2. 
			\end{aligned}
		\end{equation}
		Moreover, we have
		\begin{equation}
			\begin{aligned}
				(\beta_{12}+\beta_{22}+\beta_2)^2 &= [m\sigma_1^2+m\sigma_2^2+\frac m2(\bar x_1-\bar x_2)^2]^2 \\
				&= \frac14 m^2[2\sigma_1^2+2\sigma_2^2+(\bar x_1-\bar x_2)^2]^2 \\
				&= \frac14 m^2 T_3. 
			\end{aligned}
		\end{equation}
		As a result, we obtain that
		\begin{equation}
			\label{eqn:58}
			SSR(\vv^\top\hat\mX_1, \vv^\top\hat\mX_2) = LSSR(\mX_1,\mX_2) - \frac{2(T_1+T_2)}{T_3}t + o(t). 
		\end{equation}
    \end{proof}

    \begin{lemma}
        \label{lemma:pro2lssr}
        For
            \begin{equation}
    	f_{SSR}(t) = \begin{cases}
    		LSSR(\mX_1, \mX_2), &t=0, \\
    		SSR(\bar\psi(t;\mX_1), \bar\psi(t;\mX_2)), &t\ne0, 
    	\end{cases}
    \end{equation}
    where $ \bar\psi(t;\rvx_{ci}) = \bm1^\top \bar\varphi(t;\rvx_{ci}) / \| \bar\varphi(t;\rvx_{ci}) \|_2 $, $ \bar\varphi(t;\rvx_{ci}) = [(\vu^*)^\top\rvx_{ci}t, 1]^\top $ and $ t\in \R $, we have that $ f_{SSR}(t) $ is derivable around $ t=0 $, and $ f_{SSR}'(0) $ is only decided by $ \mX_1 $ and $ \mX_2 $.
    \end{lemma}

    \begin{proof}
    It is easy to identify that $ \bar\psi(t;\mX_i) = \vv^\top\hat\mX_i $. 
    Therefore, by Lemma \ref{lemma:6}. We have
   \begin{equation}
        \begin{aligned}
        SSR(\bar\psi(t;\mX_1), \bar\psi(t;\mX_2)) &= SSR(\vv^\top\hat\mX_1,\vv^\top\hat\mX_2) \\
        &= LSSR(\mX_1,\mX_2) - \frac{2(T_1+T_2)}{T_3}t + o(t).
        \end{aligned}
    \end{equation}
    We hence obtain
    \begin{equation}
        \begin{aligned}
            f_{SSR}'(0) &= \lim_{t\to0} \frac{f_{SSR}(t) - f_{SSR}(0)}{t} \\
            &= \lim_{t\rightarrow 0}\frac{SSR(\bar\psi(t;\mX_1), \bar\psi(t;\mX_2))-LSSR(\mX_1,\mX_2)}{t} \\
            &= \frac{- \frac{2(T_1+T_2)}{T_3}t + o(t)}{t} \\
            &= - \frac{2(T_1+T_2)}{T_3}.        
        \end{aligned}
    \end{equation}
    Conclusively, we have that $ f_{SSR}(t) $ is derivable at $t=0$.
    
    \end{proof}

    \begin{lemma}
        \label{lemma:7}
        Given a differentiable function $ f(x) $, with $ f'(0) \ne 0$, we figure out that there is some $x^*$, such that $ f(x^*) < f(0) $.
    \end{lemma}

    \begin{proof}
        Given that $ f(0)=A $ and $ f'(0) = B \ne 0$, by the definition of derivative, we have $ \lim\limits_{h\to0} \frac{f(h)-A}{h} = B$. That is to say, $\forall \eps>0$, there exists a positive $\delta > 0$, whenever $ 0<|x|<\delta $, we have
        \begin{equation}
            \left\vert\frac{f(x)-A}{x}-B \right\vert\le\eps ,
        \end{equation} 
        then
            \begin{equation}
			-\eps |x|\le f(x) - A - Bx \le \eps |x|. 
		\end{equation}
		Let $ x^* = -\frac{|B|\delta}{2B} $, and $ \eps = \frac{|B|}2 $. We have 
        \begin{equation}
            f(x^*) \le A+Bx+\eps|x| = A-\frac{|B|\delta}{4} < A ,
        \end{equation} 
        namely $ f(x^*) < f(0) $.
    \end{proof}
    
    \subsection{Proof of Theorem \ref{thm:break}}
    Since $ f_{SSR}'(0) = -2(T_1+T_2)/T_3  \ne 0 $, by Lemma \ref{lemma:7}, there is some $ t = t^* $, such that
		\begin{equation}
			SSR(\bar\psi(t^*;\mX_1), \bar\psi(t^*;\mX_2)) < LSSR(\mX_1,\mX_2). 
		\end{equation} 
	We denote $ \tilde\varphi_1(\rvx) = \varphi(t^*;\vu^\top\rvx) $ and $ \tilde\varphi_2(\rvx) = \vv^\top\rvx $. By Lemma \ref{lemma:equivalence}, we have $ SP(\cdot) = \hat\varphi_1 \circ LN(\cdot) \circ \hat\varphi_2 $. We hence have 
    \begin{equation}
        \vv^\top\hat\mX_c = \tilde\varphi_2 (\hat\varphi_2 (LN (\hat\varphi_1 (\tilde\varphi_1 (\mX_c)))) . 
    \end{equation}
    Let $\psi = \varphi_1 \circ LN(\cdot) \circ \varphi_2 $, where $ \varphi_1 = \tilde\varphi_1 \circ \hat\varphi_1 $ and $ \varphi_2 = \hat\varphi_2 \circ \tilde\varphi_2 $. We thus have
    \begin{equation}
		SSR(\psi(\mX_1), \psi(\mX_2)) < LSSR(\mX_1, \mX_2). 
	\end{equation}
    Obviously, $\varphi_1$ and $\varphi_2$ are linear functions. Consequently, we have proved Theorem \ref{thm:break}. 

    \subsection{A Generalized Proof of Theorem \ref{thm:break}}
	
	To begin with, we also need to project $ \mX_c $ to $ \vu^\top\mX_c $. This can reach $LSSR(\mX_1, \mX_2)$, which is necessary in our discussion. More generally, we design a $n$-dimensional linear transformation $ \varphi_n(t;\cdot) : \R \to \R^n $, instead of a $2$-dimensional one in Eqn.\ref{eqn:55}. Specifically, we denote 
	\begin{equation}
		\varphi_n(t;x) = t\cdot\vw x + \vb = \mat{w_1xt +b_1 \\ \cdots \\ w_nxt+b_n}. 
	\end{equation}
 
    Considering SP on $ \varphi_n(t;x) $ with scaling=1, we denote 
    \begin{equation}
        \hat \rvx = SP(\varphi_n(t;x)) = \varphi_n(t;x) / \Vert \varphi_n(t;x) \Vert .
    \end{equation} 
    Owing to the introduce of $ t $, let $ t $ and $ \vw $ represent the direction and length of weight respectively. We thus add the constraint $ \Vert \vw \Vert_2 = 1 $ for convenience. As for the bias, if $ \vb = \bm0 $, $ \hat \rvx = \vw/(\Vert \vw \Vert_2) $ will result in $ SS(\psi(\hat\mX_1),\psi(\hat\mX_2)) = 0 $. Therefore, we require that $ \vb \ne \bm0 $. Now we are concerned about $ LSSR(\hat\mX_1, \hat\mX_2) $. 
	
	Factually, by Proposition \ref{prop:LSSR}, we need not consider all the linear functions on $ \R^n $ to get LSSR. We figure out that there must be some $ \vv\in\R^n $, such that
    \begin{equation}
        LSSR(\hat\mX_1, \hat\mX_2) = SSR( \vv^\top \hat\mX, \vv^\top \hat\mY ). 
    \end{equation}
    We give the Taylor's expansion of $ \hat\rvx $ on its each dimension. 
	
	Let $ \xi_1 = \vw^\top\vb $ and $ \xi_2 = \vb^\top\vb $. We figure out that
	\begin{equation}
		\begin{aligned}
			\frac {1} {\Vert \vw xt+\vb \Vert_2} 
			=& (1+2\xi_1xt+\xi_2x^2t^2)^{-\frac12}\\
			=& 1 - \frac12(2\xi_1xt + \xi_2x^2t^2) + \frac38 (2\xi_1xt + \xi_2x^2t^2)^2 - \frac5{16}(2\xi_1xt + \xi_2x^2t^2)^3 + o(t^3)\\
			=& 1 - \xi_1xt + (\frac32\xi_1^2-\frac12 \xi_2)x^2t^2 + (\frac32\xi_1\xi_2-\frac52\xi_1^3)x^3t^3 + o(t^3). 
		\end{aligned}
	\end{equation}
	We further obtain
	\begin{equation}
		\label{eqn:27}
		\begin{aligned}
			\hat x^{(k)} 
			=& \frac {w_kxt+b_k} {\Vert \vw xt+\vb \Vert_2}\\ =& (b_k+w_kxt)[1 - \xi_1xt + (\frac32\xi_1^2 - \frac12 \xi_2)x^2t^2 + (\frac32\xi_1\xi_2 - \frac52 \xi_1^3)x^3t^3 + o(t^3)]\\
			=&b_k + (w_k-\xi_1b_k)xt + [(\frac32\xi_1^2 - \frac12\xi_2)b_k - \xi_1w_k]x^2t^2 + [(\frac32\xi_1\xi_2-\frac52\xi_1^3)b_k + (\frac32\xi_1^2 - \frac12\xi_2)w_k]x^3t^3 + o(t^3).
		\end{aligned}
	\end{equation}
	Similarly, to simplify our calculation, we denote 
	\begin{equation}
		\hat x_{ci}^{(k)} = a_0^{(k)} + a_1^{(k)} x_{ci}t + a_2^{(k)} x_{ci}^2t^2  + a_3^{(k)} x_{ci}^3 t^3 + o(t^3)
	\end{equation}
	where $ a_0^{(k)} = b_k $, $ a_1^{(k)} = w_k-\xi_1b_k $, $ a_2^{(k)} = (\frac32\xi_1^2 - \frac12\xi_2)b_k - \xi_1w_k $ and $ a_3^{(k)} = (\frac32\xi_1\xi_2-\frac52\xi_1^3)b_k + (\frac32\xi_1^2 - \frac12\xi_2)w_k $. 
	
	Let $\vv = [v_1,\cdots,v_n]^\top $. We have
	\begin{equation}
		\begin{aligned}
			SS(\vv^\top\hat \mX_{c}) &= \sum_{i=1}^m(\vv^\top\hat\rvx_{ci}-\overline{\vv^\top\hat\rvx_c})^2\\
			&=\frac1m \sum_{i=1}^m \sum_{j=1}^m \vv^\top \hat\rvx_{ci} (\vv^\top\hat\rvx_{ci} - \vv^\top\hat\rvx_{cj})\\
			&=\frac1m \sum_{i=1}^m \sum_{j=1}^m \left(\sum_{k=1}^{n} v_k \hat x_{ci}^{(k)}\right)\cdot\left(\sum_{l=1}^{n} v_l[\hat x_{ci}^{(l)} - \hat x_{cj}^{(l)}]\right)\\
			&=\sum_{k=1}^{n} \sum_{l=1}^{n} v_k v_l \left(\frac1m \sum_{i=1}^m \sum_{j=1}^m\hat x_{ci}^{(k)} [\hat x_{ci}^{(l)} - \hat x_{cj}^{(l)}]\right).
		\end{aligned}
	\end{equation}
	Similar to calculation when we discuss the $2$-dimensional case, we have
	\begin{equation}
		\begin{aligned}
			&\frac1m\sum_{i=1}^m \sum_{j=1}^m\hat x_{ci}^{(k)} [\hat x_{ci}^{(l)} - \hat x_{cj}^{(l)}] \\
			=&\frac1m\sum_{i=1}^m \sum_{j=1}^m \left[\left(\sum_{s=0}^{3} a_s^{(k)}x_{ci}^st^s+o(t^3)\right)\cdot\left(\sum_{s=0}^{3} a_s^{(l)}(x_{ci}^s-x_{cj}^s)t^s+o(t^3)\right)\right]\\
			=&\frac1m\sum_{i=1}^m \sum_{j=1}^m \left[a_1^{(k)} a_1^{(l)} x_{ci}(x_{ci}-x_{cj})t^2 + a_2^{(k)}a_1^{(l)} x_{ci}^2(x_{ci} - x_{cj})t^3 + a_1^{(k)}a_2^{(l)} x_{ci}(x_{ci}^2-x_{cj}^2)t^3 + o(t^3)\right]\\
			=& ma_1^{(k)} a_1^{(l)}\sigma_c^2t^2 + m[a_1^{(k)} a_2^{(l)}+a_2^{(k)} a_1^{(l)}][\overline{(x_c-\bar x_{ci})^3} + 2\bar x_c\sigma_c^2]t^3 + o(t^3).
		\end{aligned}
	\end{equation}
	
	We define 
	\begin{equation}
		\theta_1=\sum_{k=1}^n\sum_{l=1}^n v_kv_la_1^{(k)} a_1^{(l)}, 
	\end{equation} 
    and 
    \begin{equation}
        \theta_2 = \sum_{k=1}^n \sum_{l=1}^n v_k v_l a_1^{(k)} a_2^{(l)} = \sum_{k=1}^n \sum_{l=1}^n v_k v_l a_2^{(k)} a_1^{(l)}. 
    \end{equation}
	We thus have that
	\begin{equation}
		\begin{aligned}
			SS(\vv^\top\hat \mX_{c}) 
			&=\sum_{k=1}^{n} \sum_{l=1}^{n} v_k v_l \left(\frac1m \sum_{i=1}^m \sum_{j=1}^m\hat x_{ci}^{(k)} [\hat x_{ci}^{(l)} - \hat x_{cj}^{(l)}]\right) \\
			&=\sum_{k=1}^{n} \sum_{l=1}^{n} v_k v_l \left(ma_1^{(k)} a_1^{(l)}\sigma_c^2t^2 + m[a_1^{(k)} a_2^{(l)}+a_2^{(k)} a_1^{(l)}][\overline{(x_c-\bar x_{ci})^3} + 2\bar x_c\sigma_c^2]t^3 + o(t^3)\right) \\
			&= m\theta_1\sigma_c^2t^2 + 2m\theta_2[\overline{(x_c-\bar x_{ci})^3} + 2\bar x_c\sigma_c^2]t^3 + o(t^3) \\
			&= \beta_{c2} t^2 +\beta_{c3}t^3 +o(t^3), 
		\end{aligned}
	\end{equation}
    where $\beta_{c2} = m\theta_1\sigma_c^2$ and $\beta_{c3} = 2m\theta_2[\overline{(x_c-\bar x_{ci})^3} + 2\bar x_c\sigma_c^2]$. 
    
	On the other hand, we obtain
	\begin{equation}
		\begin{aligned}
			SS_D(\vv^\top\hat\mX_1, \vv^\top\hat\mX_2)
			=&\frac m2(\overline{\vv^\top\hat\rvx_1}-\overline{\vv^\top\hat\rvx_2})^2\\
			=&\frac{1}{2m}\left(\sum_{i=1}^m[\vv^\top\hat\rvx_{1i}-\vv^\top\hat\rvx_{2i}]\right)^2\\
			=&\frac{1}{2m}\left(\sum_{i=1}^m\sum_{k=1}^n v_k(\hat x_{1i}^{(k)}-\hat x_{2i}^{(k)})\right)^2\\
			=&\frac{1}{2m}\left(\sum_{i=1}^m\sum_{k=1}^n v_k[a_1^{(k)}(x_{1i}-x_{2i})t+a_2^{(k)}(x_{1i}^2-x_{2i}^2)t^2+o(t^2)]\right)^2\\
			=&\frac{m}{2}\left(\sum_{k=1}^n v_k[a_1^{(k)}(\bar x_1 - \bar x_2)t + a_2^{(k)}(\overline{x_1^2}-\overline{x_2^2})t^2+o(t^2)]\right)^2\\
			=&\frac{m}{2}\sum_{k=1}^n \sum_{l=1}^n v_k v_l[a_1^{(k)}a_1^{(l)} (\bar x_1- \bar x_2)^2t^2 + [a_1^{(k)}a_2^{(l)} + a_2^{(k)}a_1^{(l)}](\bar x_1- \bar x_2)(\overline{x_1^2}-\overline{x_2^2})t^3+o(t^3)]\\
			=&\frac m2 \theta_1(\bar x_1-\bar x_2)^2t^2 + m\theta_2(\bar x_1 - \bar x_2) (\overline{x_1^2}-\overline{x_2^2})t^3+o(t^3) \\
			=&\beta_2t^2 + \beta_3t^3 + o(t^3), 
		\end{aligned}
	\end{equation}
    where $\beta_2 = \frac m2 \theta_1(\bar x_1-\bar x_2)^2$ and $\beta_3 = m\theta_2(\bar x_1 - \bar x_2) (\overline{x_1^2}-\overline{x_2^2})$.
	Therefore, we figure out that
	\begin{equation}
		\begin{aligned}
		&\beta_2(\beta_{13}+\beta_{23}) - \beta_3(\beta_{12}+\beta_{22}) \\
		=& m^2\theta_1\theta_2 (\bar x_1-\bar x_2)^2[\overline{(x_1-\bar x_{1})^3} + 2\bar x_1\sigma_1^2 + \overline{(x_2-\bar x_{2})^3} + 2\bar x_2\sigma_2^2] - m^2\theta_1\theta_2 (\bar x_1 - \bar x_2)(\overline{x_1^2} - \overline{x_2^2})(\sigma_1^2+\sigma_2^2) \\
		=& m^2\theta_1\theta_2(\bar x_1-\bar x_2)^2[\overline{(x_1-\bar x_{1})^3} +  \overline{(x_2-\bar x_{2})^3}] + m^2\theta_1\theta_2 (\bar x_1-\bar x_2) (\sigma_1^2 -\sigma_2^2) [(\bar x_1-\bar x_2)^2- (\sigma_1^2+\sigma_2^2)]\\
		=& m^2 \theta_1\theta_2 T_1 + m^2 \theta_1\theta_2 T_2, 
		\end{aligned}
	\end{equation}
	and
	\begin{equation}
		\begin{aligned}
			(\beta_{12}+\beta_{22}+\beta_2)^2 &= \left[m\theta_1 \sigma_1^2 + m\theta_1 \sigma_2^2 + \frac m2 \theta_1(\bar x_1-\bar x_2)^2\right]^2 \\
			&= \frac14 m^2\theta_1^2 [2\sigma_1^2 + 2 \sigma_2^2 + (\bar x_1-\bar x_2)^2]^2 \\
			&= \frac14 m^2\theta_1^2 T_3. 
		\end{aligned}
	\end{equation}
	We hence obtain 
	\begin{equation}
		SSR(\vv^\top\hat\mX_1,\vv^\top\hat\mX_2) = LSSR(\mX_1,\mX_2) + \frac{4\theta_2}{\theta_1}\frac{T_1+T_2}{T_3}t + o(t). 
	\end{equation}
	Similarly, $ f_{SSR}'(0)\ne0 $ means $ T_1 + T_2 \ne 0 $, we can find some $t$ and some $ \vw,\vb,\vv $, such that $ \theta_2\ne0 $, and $ SSR(\vv^\top\hat\mX_1, \vv^\top\hat\mX_2) < LSSR(\mX_1,\mX_2) $. We figure out that in the $ n $-dimensional case, $ f_{SSR}'(0)\ne0 $ is also required in our proof. 

    Hereafter, the remaining proof is nearly the same as the $2$-dimensional version. 
	
	Take our $2$-dimensional $\varphi(t;\cdot)$ as an example, $ \vw = [1,0]^\top, \vb = [0,1]^\top, \vv = [1,1]^\top $. We thus have $ s_1 = 0, s_2 = 1 $. Furthermore, we obtain $ a_1^{(1)} = 1, a_1^{(2)} = 0, a_2^{(1)} = 0, a_2^{(2)} = -\frac12 $ and then $ \theta_1 = 1, \theta_2 = -\frac12 $. As a result, we have $ 4\theta_2/\theta_1 = -2 $, which is the same as Eqn.\ref{eqn:58}. 

    \section{Proofs Related to the Algorithms}
    \label{section:proofofAlgorithms}
    
    \subsection{Proof of Proposition \ref{prop:initial}}

    \paragraph{Proposition 4.}\textit{For any input $ \mX^{(0)} $, we can find some $ \vu $, such that 
        \begin{equation}
            \varphi_1^{(1)}: \rvx_k^{(0)} \mapsto \vp_k^{(1)} = [\vu^\top\rvx_k^{(0)}, 0]^\top, 
        \end{equation}
        where $ \vp_i^{(1)} \ne \vp_j^{(1)} $ if $ \rvx_i^{(0)} \ne \rvx_j^{(0)} $. }
    
	\begin{proof} 
        In reverse, we consider to find all the $ \vu $, such that some two different points are coincident after the projection. 

        Given two points $ \rvx_i^{(0)} \ne \rvx_j^{(0)} $ in $ \mX^{(0)} $, if $ \vu $ project them into the same point, we have
        \begin{equation}
            \label{eqn:104}
            \vu^\top \rvx_i^{(0)} = \vu^\top \rvx_j^{(0)}. 
        \end{equation}We use $ \sS_2(\rvx_i^{(0)},\rvx_j^{(0)}) $ to denote the whole solution space of \Eqn\ref{eqn:104}, namely
        \begin{equation}
            \sS_2(\rvx_i^{(0)},\rvx_j^{(0)}) = \{ \vu\in\R^d:\vu^\top (\rvx_i^{(0)} - \rvx_j^{(0)}) = 0 \}. 
        \end{equation}
        Considering all the pairs of different points, we define
        \begin{equation}
            \hat\sS_2(\mX^{(0)}) = \bigcup_{ \rvx_i^{(0)} \ne \rvx_j^{(0)} } \sS_2(\rvx_i^{(0)},\rvx_j^{(0)}) 
        \end{equation}
		Since $ \rvx_i^{(0)} \ne \rvx_j^{(0)} $, we find each solution space $ \sS_2(\rvx_i^{(0)},\rvx_j^{(0)}) $ is $(d-1)$ dimensional, and the number of such sets is no more than $ m^2 $. Therefore, the union of these solution spaces\footnote{The union of finite subspaces of $d-1$ dimensional can not cover the whole space $\R^d$.} is still smaller than $\R^d$, namely $ \hat\sS_2(\mX^{(0)}) \subset \R^d $. 

        We obtain that $ \exists \rvx_i^{(0)} \ne \rvx_j^{(0)}, \vu^\top \rvx_i^{(0)} = \vu^\top \rvx_j^{(0)} $, if and only if $ \vu \in \hat\sS_2(\mX^{(0)}) $. Since $ \hat\sS_2(\mX^{(0)}) \subset \R^{d} $, we obtain 
        \begin{equation}
            \R^d/\hat\sS_2(\mX^{(0)}) \neq \emptyset. 
        \end{equation}
        Therefore, we can always find a $ \vu \in \R^d $, such that we have $ \vp_i^{(1)} \ne \vp_j^{(1)}$, for any $ \rvx_i^{(0)} \ne \rvx_j^{(0)}$.
	\end{proof}

    \subsection{Proof of Proposition \ref{prop:layer property}}




    \paragraph{Proposition 5.}\textit{For each layer, $ \varphi_l^{(1)} (2\le l \le L) $ only merges points with the same label. Nevertheless, $ \varphi_1^{(1)}, SP(\cdot) $ and $ \varphi_l^{(2)} (1\le l \le L-1) $ do not merge any points. }

    \begin{proof}

    1) According to Proposition \ref{prop:initial}, we figure out that $ \varphi_1^{(1)}$ does not merge any points. 

    \begin{figure}[h]
        \centering
        \begin{tikzpicture}
			\fill[blue, shift={(-1,2)}] (0,0) circle (3pt);
			\fill[blue, shift={(1,2)}] (0,0) circle (3pt);
			\fill[orange, shift={(-2,2)}] (0,0) circle (3pt);
			\fill[orange, shift={(2,2)}] (0,0) circle (3pt);
			\fill[blue, shift={(3,2)}] (0,0) circle (3pt);
			
			\draw [->, black] (-3,0) -- (3.5,0);
			\draw [->, black] (0,-1) -- (0,3);
            \node[] at (-0.2,2.7) {$y$};
            \node[] at (3.3,-0.2) {$x$};
            \node[] at (-0.2,-0.2) {$O$};
			\draw [] (2,0) arc [start angle=0,end angle=180, radius=2];

            \draw [->, blue, thick] (0.5,0) arc [start angle=0,end angle=116.56, radius=0.5];
            \node [blue] at (0.45,-0.35) {$\gamma_2^{(l)}$};

			\fill[blue, shift={(-0.894,1.788)}] (-4pt,-1pt) rectangle (+4pt,+1pt) (-1pt,-4pt) rectangle (+1pt,+4pt);
			\fill[blue, shift={(0.894,1.788)}] (-4pt,-1pt) rectangle (+4pt,+1pt) (-1pt,-4pt) rectangle (+1pt,+4pt);
			\fill[orange, shift={(-1.414,1.414)}] (-4pt,-1pt) rectangle (+4pt,+1pt) (-1pt,-4pt) rectangle (+1pt,+4pt);
			\fill[orange, shift={(1.414,1.414)}] (-4pt,-1pt) rectangle (+4pt,+1pt) (-1pt,-4pt) rectangle (+1pt,+4pt);
			\fill[blue, shift={(1.664,1.109)}] (-4pt,-1pt) rectangle (+4pt,+1pt) (-1pt,-4pt) rectangle (+1pt,+4pt);
			
			\draw[<-, dashed, gray, very thin] (0,0) -- (-2,2);
            \draw[<-, dashed, gray, very thin] (0,0) -- (-1,2);
            \draw[<-, dashed, gray, very thin] (0,0) -- (1,2);
            \draw[<-, dashed, gray, very thin] (0,0) -- (2,2);
            \draw[<-, dashed, gray, very thin] (0,0) -- (3,2);
			
			\node [orange] at (-2,2) [above] {$\vh_1^{(l)}$};
			\node [orange] at (2,2) [above] {$\vh_4^{(l)}$};
			\node [blue] at (-1,2) [above] {$\vh_2^{(l)}$};
			\node [blue] at (1,2) [above] {$\vh_3^{(l)}$};
			\node [blue] at (3,2) [above] {$\vh_5^{(l)}$};
			
			\node [orange] at (-1.2,1.1) {$ \rvx_1^{(l)} $};
			\node [blue] at (-0.45,1.5) {$ \rvx_2^{(l)} $};
			\node [blue] at (0.45,1.5) {$ \rvx_3^{(l)} $};
			\node [orange] at (1.0,1.0) {$ \rvx_4^{(l)} $};
			\node [blue] at (1.4,0.6) {$ \rvx_5^{(l)} $};
		\end{tikzpicture}
        \caption{A copied figure from Figure 2(b).}
    \end{figure}
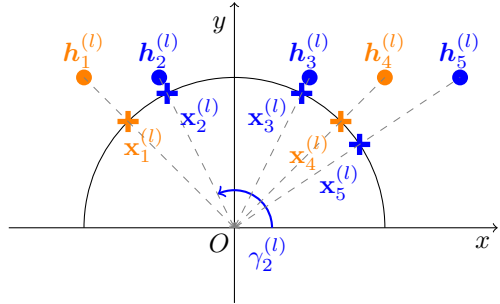

    2) Furthermore, we analyze $SP(\cdot) $. Focused on $ \gamma_k^{(l)} $ (there is an example of $\gamma_2^{(l)}$ copied from Figure \ref{fig:2b}), we figure out that     
    \begin{equation}
        \label{eqn:132}
        \gamma_k^{(l)} = \arctan \frac{2p_k^{(l)} - p_i^{(l)} - p_j^{(l)}} {p_j^{(l)} - p_i^{(l)}} ,
    \end{equation}
    where $p_i^{(l)} , p_j^{(l)}$ is defined in Algorithm~\ref{alg:PMA}. 
    We can obtain that $\gamma_k^{(l)}$ is monotonically decreasing with $ p_k^{(l)} $. 
    
    When $ \vh_{k_i}^{(l)}\ne \vh_{k_j}^{(l)}$, we have $\gamma_{k_i}^{(l)}\ne\gamma_{k_j}^{(l)}$, namely $\rvx_{k_i}^{(l)} \ne \rvx_{k_j}^{(l)}$. In other words, $SP(\cdot) $ does not merge any points.

    3) We then consider $ \varphi_l^{(2)} (1\le l \le L-1) $. Obviously, $\varphi_l^{(2)}$ is a translation transformation, which does not merge any points. 

    4) Finally, we consider $ \varphi_l^{(1)} (2\le l \le L-1) $. We first have 
    \begin{equation}
        p_k^{(l+1)} = \mat{0&1}\varphi_l^{(1)}(\rvx_k^{(l)})=\sin \gamma_k^{(l)}. 
    \end{equation}
    Accordingly, $ p_k^{(l+1)} $ is also monotonically decreasing with $ p_k^{(l)} $ when $\gamma_k^{(l)}\le \frac{\pi}{2}$. 


    Given two points from different classes denoted as $p_{k_1}^{(l)}, p_{k_2}^{(l)}$, we discuss them under three cases. 
    
    Case 1: If $p_{k_1}^{(l)}, p_{k_2}^{(l)} > p_j^{(l)}$, we have $\gamma_{k_1}^{(l)}, \gamma_{k_2}^{(l)} < \frac{\pi}{4}$. Therefore, $p_k^{(l+1)}$ is monotonically decreasing with $p_k^{(l)}$. We have
    \begin{equation}
        p_{k_1}^{(l)} \ne p_{k_2}^{(l)} \Leftrightarrow \vp_{k_1}^{(l+1)} \ne \vp_{k_2}^{(l+1)}.
    \end{equation}
    
    Case 2: If one of $  p_{k_1}^{(l)}, p_{k_2}^{(l)} $ is less than $ p_j^{(l)} $ and the other is not, then one of $  p_{k_1}^{(l+1)}, p_{k_2}^{(l+1)} $ is lager than $ \frac{\sqrt2}{2} $, while the other is not. We hence have $\vp_{k_1}^{(l+1)} \ne \vp_{k_2}^{(l+1)}$
    
    Case 3: If $ p_{k_1}^{(l)}, p_{k_2}^{(l)} $ are both less than $ p_j^{(l)} $---this case will never happen, otherwise one of them belongs to the same class with $ \vp_i^{(l)} $, resulting $ \vp_j^{(l)} $ is not the leftmost point (with the same label as $ \vp_i^{(l)} $), which contradicts the definition of $j$. 
  
    Conclusively, we find the samples from different classes will not merge by $\varphi_l^{(1)} (1\le l \le L-1)$. 

    Based on all the discussions above, we have proved Proposition \ref{prop:layer property}. 
    \end{proof}

    \subsection{Proof of Proposition \ref{prop:confusion}}

    \paragraph{Proposition 6.}\textit{Confusion refers to merging two points with different labels. If confusion happens when we project $ \mX^{(l+1)} $ onto the $y$-axis, there must be a parallelogram\footnote{The parallelogram may be degenerate. Given four points $ \rvx_1, \rvx_2, \rvx_3, \rvx_4 $, if the sum of two points is the same with that of the other two, we regard they form a parallelogram.} consisting of four different points in $ \mP^{(l)} $. }

    \begin{proof}
        If confusion happens, we will merge some two points $ \rvx_s^{(l)} $ and $ \rvx_t^{(l)} $ with different labels. According to \Eqn\ref{eqn:132}, we find that $ \sin\gamma_s^{(l)} = \sin\gamma_t^{(l)} $. Since $ \rvx_s^{(l)} $ and $ \rvx_t^{(l)} $ are different points on the unit circle, we have $ \gamma_s^{(l)} = \pi - \gamma_t^{(l)} $, namely they are symmetric about $y$-axis. Furthermore, $ \vh_s^{(l)} $ and $ \vh_t^{(l)} $ are symmetric about $y$-axis. Besides, $ \vh_i^{(l)} $ and $ \vh_j^{(l)} $ are also symmetric about $y$-axis. Since the four points are on the same line, we have $ \vh_i^{(l)} + \vh_j^{(l)} = \vh_s^{(l)} + \vh_t^{(l)} $. For $ \mH^{(l)} $ is translated from $ \mP^{(l)} $, we have $ \vp_i^{(l)} + \vp_j^{(l)} = \vp_s^{(l)} + \vp_t^{(l)} $. We hence find a parallelogram in $ \mP^{(l)} $. 
    \end{proof}

    \subsection{Proof of Proposition \ref{prop:PBA}}

    \paragraph{Proposition 7.}\textit{We can always find $ \vu_l \in \R^2 $ for Algorithm \ref{alg:PBA}, such that there is no parallelograms in $ \hat\mP^{(l)} $, and no points merged in the algorithm. }
    
	\begin{proof}
		By Algorithm~\ref{alg:PBA}, we shift the points in $ \mP^{(l)} $ up by $ 1 $, and then projects onto the unit circle $ x^2+y^2=1 $, namely
        \begin{equation}
            \tilde\vp_i^{(l)} \gets SP \left( \vp_{i}^{(l)} + \mat{0&1}^\top \right).
        \end{equation}
        We find all points in $ \tilde \mP^{(l)} $ are on the upper half circle. Obviously, any four different points in $ \tilde \mP^{(l)} $ can not form a parallelogram, for the quadrilateral has two adjacent obtuse angles. In other words, give four different points $ \tilde\vp_i, \tilde\vp_j, \tilde\vp_s, \tilde\vp_t $, we have $ \tilde\vp_i + \tilde\vp_j \ne \tilde\vp_s + \tilde\vp_t $. Besides, if $ \vp_i^{(l)} \ne \vp_j^{(l)} $, we have $ \tilde \vp_i^{(l)} \ne \tilde \vp_j^{(l)} $. 
        
        We can intuitively identify the two claims above in Figure~\ref{fig:3}. 

        Similarly, consider to merge different points together by $ \vu_l $, we can find $\vu_l$ from the set
        \begin{equation}
            \hat\sS_2(\tilde\mP^{(l)}) = \bigcup_{ \tilde\vp_i^{(l)} \ne \tilde\vp_j^{(l)} } \sS_2(\tilde\vp_i^{(l)},\tilde\vp_j^{(l)}), 
        \end{equation}
		where $ \sS_2(\tilde\vp_i^{(l)},\tilde\vp_j^{(l)}) = \{ \vu_l\in\R^2:\vu_l^\top (\tilde\vp_i^{(l)} - \tilde\vp_j^{(l)}) = 0 \} $. 
  
		We then consider to form a parallelogram. We need some $ \vu_l $, and four different points $ \tilde\vp_i, \tilde\vp_j, \tilde\vp_s, \tilde\vp_t $, such that $ \vu_l^\top \tilde\vp_i^{(l)} + \vu_l^\top \tilde\vp_j^{(l)} = \vu_l^\top \tilde\vp_s^{(l)} + \vu_l^\top \tilde\vp_t^{(l)} $. Obviously, we can find $ \vu_l $ from the set
        \begin{equation}
            \hat\sS_4(\tilde \mP^{(l)}) = \bigcup_{(i,j,s,t)\in \sI_{4}(\tilde \mP^{(l)})} \sS_4(\tilde\vp_i^{(l)}, \tilde\vp_j^{(l)}, \tilde\vp_s^{(l)}, \tilde\vp_t^{(l)}), 
        \end{equation}
        where 
        \begin{equation}
            \sS_4(\tilde\vp_i^{(l)}, \tilde\vp_j^{(l)}, \tilde\vp_s^{(l)}, \tilde\vp_t^{(l)}) = \{\vu_l\in\R^2: \vu_l^\top(\tilde\vp_i^{(l)} + \tilde\vp_j^{(l)} - \tilde\vp_s^{(l)} -\tilde\vp_t^{(l)}) = 0 \},
        \end{equation}
        and the index set
        \begin{equation}
  \sI_{4}(\tilde \mP^{(l)}) = \{ (i,j,s,t): \tilde\vp_i^{(l)}, \tilde\vp_j^{(l)}, \tilde\vp_s^{(l)}, \tilde\vp_t^{(l)} \text{ are different with each other} \}.
        \end{equation}

        Similarly, we point out that $ \hat\sS_2(\tilde\mP^{(l)}) $ consists of\footnote{$ \hat\sS_2(\tilde\mP^{(l)}) $ is a point set of finite lines, hence can not cover the whole $\R^2$.} no more than $ m^2 $ spaces of $1$-dimensional. On the other hand, since $ \tilde\vp_i + \tilde\vp_j \ne \tilde\vp_s + \tilde\vp_t $ holds for any four different points in $ \tilde\mP^{(l)} $, each $ \sS_4(\tilde\vp_i, \tilde\vp_j, \tilde\vp_s, \tilde\vp_t) $ is a $1$-dimensional space. Therefore, $ \hat\sS_4(\tilde \mP^{(l)}) $ consists of no more than $ m^4 $ spaces of $1$-dimension. 
        We hence obtain that $ \hat\sS_4(\tilde \mP^{(l)}) \cup \hat\sS_2(\tilde\mP^{(l)}) $ consists of no more than $ m^2+m^4 $ spaces of $1$-dimension, namely
        \begin{equation}
            [\hat\sS_4(\tilde \mP^{(l)}) \cup \hat\sS_2(\tilde\mP^{(l)})] \subset \R^2 . 
        \end{equation}
        
        We thus have that---there exists $  \tilde\vp_i^{(l)} \ne \tilde\vp_j^{(l)}$ subjected to $ \vu_l^\top \tilde\vp_i^{(l)} = \vu_l^\top \tilde\vp_j^{(l)} $, if and only if $ \vu_l \in \hat\sS_2(\mP^{(l)}) $. On the other hand, we figure out that---there exists four different points $ \tilde\vp_i, \tilde\vp_j, \tilde\vp_s, \tilde\vp_t $ subjected to $ \vu_l^\top \tilde\vp_i^{(l)} + \vu_l^\top \tilde\vp_j^{(l)} = \vu_l^\top \tilde\vp_s^{(l)} + \vu_l^\top \tilde\vp_t^{(l)} $, if and only if $ \vu_l \in \hat\sS_4(\mP^{(l)}) $. Since $ [\hat\sS_4(\tilde \mP^{(l)}) \cup \hat\sS_2(\tilde\mP^{(l)})] \subset \R^2 $, we obtain $ \R^2/[\hat\sS_4(\tilde \mP^{(l)}) \cup \hat\sS_2(\tilde\mP^{(l)})] \neq \emptyset $. As a result, we can always find a $ \vu_l \in \R^2/[\hat\sS_4(\tilde \mP^{(l)}) \cup \hat\sS_2(\tilde\mP^{(l)})] $ to ensure not to merge different points, and form no parallelograms in $\hat\mP^{(l)} $ as well. 
	\end{proof}

    \subsection{Discussion on a Wider LN-Net}

    We figure out that the algorithm here is suitable for both binary and multi-class classifications. Before giving the algorithm, we propose two lemmas as follows. 

    \begin{lemma}
        \label{lemma:8}
        Given $ \mX^{(l)} $ on the unit sphere, the necessary condition of $ \overline{\rvx_i^{(l)} \rvx_j^{(l)}} /\mskip-2.5mu/  \overline{\rvx_s^{(l)} \rvx_t^{(l)}} $ is that $ \angle\rvx_j^{(l)} O  \rvx_s^{(l)} =  \angle \rvx_i^{(l)} O \rvx_t^{(l)} $, where $O$ is origin of coordinates.
    \end{lemma}

    \begin{proof}
        For $\overline{\rvx_i^{(l)} \rvx_j^{(l)}} /\mskip-2.5mu/  \overline{\rvx_s^{(l)} \rvx_t^{(l)}} $, we have
        \begin{equation}
         \rvx_j^{(l)} - \rvx_i^{(l)} = k \ ( \rvx_t^{(l)} - \rvx_s^{(l)} )
        \end{equation}
        where $k \ne 0$. 
    
        Accordingly, we figure out that
        \begin{equation}
            \rvx_j^{(l)} + k \ \rvx_s^{(l)} =  \rvx_i^{(l)} + k \ \rvx_t^{(l)},
        \end{equation}
        and furthermore, 
        \begin{equation}
            (\rvx_j^{(l)})^2 + 2k \ \rvx_j^{(l)}\cdot \rvx_s^{(l)} + k^2\ (\rvx_s^{(l)})^2  =  (\rvx_i^{(l)})^2 + 2k \ \rvx_i^{(l)}\cdot\rvx_t^{(l)} + k^2\ (\rvx_t^{(l)})^2.
        \end{equation}
        Since $\rvx_i^{(l)}, \rvx_j^{(l)}, \rvx_s^{(l)}, \rvx_t^{(l)}$ are all on the unit sphere, we have $(\rvx_j^{(l)})^2=(\rvx_j^{(l)})^2=(\rvx_j^{(l)})^2=(\rvx_j^{(l)})^2=1$. Therefore, we have $\rvx_j^{(l)}\cdot\rvx_s^{(l)} =\rvx_i^{(l)}\cdot\rvx_t^{(l)}$
    
        According to the cosine theorem, we have $|\overline{ \rvx_j^{(l)}\rvx_s^{(l)}}| = \overline{| \rvx_i^{(l)}\rvx_t^{(l)}}|$. Furthermore, according to the central angle theorem, we have $\angle\rvx_j^{(l)} O  \rvx_s^{(l)} =  \angle \rvx_i^{(l)} O \rvx_t^{(l)}$.
    
    \end{proof}
    \begin{lemma}
        \label{lemma:9}
        Given $ \vp_i^{(l)}, \vp_j^{(l)}, \vp_s^{(l)}, \vp_t^{(l)} $ which are different from each other, the solution space
        \begin{equation}
            \sB_4 (\vp_i^{(l)}, \vp_j^{(l)}, \vp_s^{(l)}, \vp_t^{(l)}) = \left\{ \vb \in \R^n: \frac{(\vp_i^{(l)} + \vb)^\top ( \vp_s^{(l)} + \vb)}{\|\vp_i^{(l)} + \vb\|_2 \|\vp_s^{(l)} + \vb\|_2} = \frac{(\vp_j^{(l)} + \vb)^\top ( \vp_t^{(l)} + \vb)}{\|\vp_j^{(l)} + \vb\|_2 \|\vp_t^{(l)} + \vb\|_2} \right\}
        \end{equation}
        is contained in a hypersurface of $n-1$ dimension. 
    \end{lemma}

    \begin{proof} 
        We first loose the equation in $ \sB_4 (\vp_i^{(l)}, \vp_j^{(l)}, \vp_s^{(l)}, \vp_t^{(l)}) $ to a polynomial equation. 
        
        Ignoring the case $ \vb\in\{ -\vp_i^{(l)}, -\vp_j^{(l)}, -\vp_s^{(l)}, -\vp_t^{(l)} \} $, we can loose the equation in $ \sB_4 (\vp_i, \vp_j, \vp_s, \vp_t) $ as
        \begin{equation}
            (\vp_i^{(l)} + \vb)^\top ( \vp_s^{(l)} + \vb) \|\vp_j^{(l)} + \vb\|_2 \|\vp_t^{(l)} + \vb\|_2 = (\vp_j^{(l)} + \vb)^\top ( \vp_t^{(l)} + \vb)\|\vp_i^{(l)} + \vb\|_2 \|\vp_s^{(l)} + \vb\|_2. 
        \end{equation}
        Furthermore, we loose it again to
        \begin{equation}
            \label{eqn:b4}
            [(\vp_i^{(l)} + \vb)^\top ( \vp_s^{(l)} + \vb)]^2 \|\vp_j^{(l)} + \vb\|_2^2 \|\vp_t^{(l)} + \vb\|_2^2 = [(\vp_j^{(l)} + \vb)^\top ( \vp_t^{(l)} + \vb)]^2 \|\vp_i^{(l)} + \vb\|_2^2 \|\vp_s^{(l)} + \vb\|_2^2.
        \end{equation}

        We define
        \begin{equation}
            \sB_4' (\vp_i^{(l)}, \vp_j^{(l)}, \vp_s^{(l)}, \vp_t^{(l)}) = \{\vb:\vb \text{ satisfies \Eqn\ref{eqn:b4}.}\}. 
        \end{equation}
        
        We find that $\vb \in \sB_4' (\vp_i^{(l)}, \vp_j^{(l)}, \vp_s^{(l)}, \vp_t^{(l)})$, for each $\vb\in\sB_4 (\vp_i^{(l)}, \vp_j^{(l)}, \vp_s^{(l)}, \vp_t^{(l)})$. Since \Eqn\ref{eqn:b4} is a polynomial equation about $\vb$, its solution space $\sB_4' (\vp_i^{(l)}, \vp_j^{(l)}, \vp_s^{(l)}, \vp_t^{(l)}) $  is a hypersurface. 

        From $\sB_4 (\vp_i^{(l)}, \vp_j^{(l)}, \vp_s^{(l)}, \vp_t^{(l)})$ to $\sB_4' (\vp_i^{(l)}, \vp_j^{(l)}, \vp_s^{(l)}, \vp_t^{(l)})$, we add the four singularities $\{ -\vp_i^{(l)}, -\vp_j^{(l)}, -\vp_s^{(l)}, -\vp_t^{(l)} \}$, and we extend $\cos\angle\rvx_j^{(l)} O  \rvx_s^{(l)} =  \cos\angle \rvx_i^{(l)} O \rvx_t^{(l)}$ to $\cos^2\angle\rvx_j^{(l)} O  \rvx_s^{(l)} =  \cos^2\angle \rvx_i^{(l)} O \rvx_t^{(l)}$. 
        

        We then prove $ \sB_4' (\vp_i^{(l)}, \vp_j^{(l)}, \vp_s^{(l)}, \vp_t^{(l)}) \subset \R^n $, to ensure it is a hypersurface of $d-1$ dimension. 

        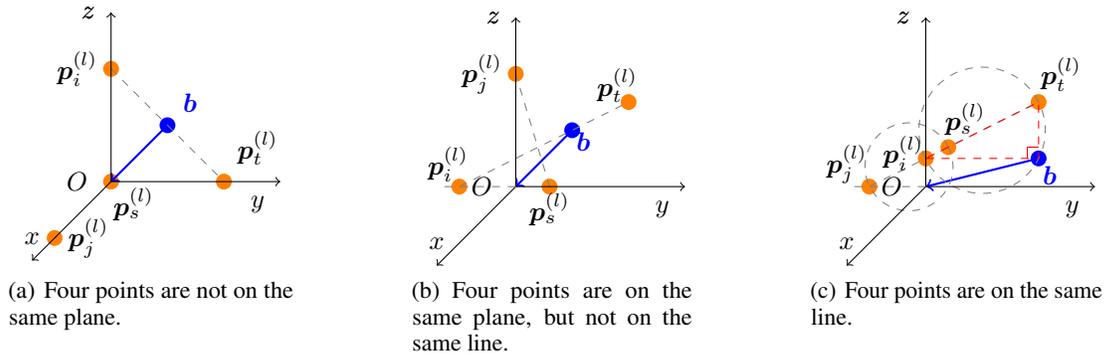
\begin{figure*}[h]
            \centering
            \subfigure[Four points are not on the same plane.]{
            \label{fig:A3-a}
			\begin{tikzpicture}[scale = 1.5]
				\fill[orange, shift={(0,0)}] (0,0) circle (2pt);
				\fill[orange, shift={(0,1)}] (0,0) circle (2pt);
				\fill[orange, shift={(1,0)}] (0,0) circle (2pt);
				\fill[orange, shift={(-0.5,-0.5)}] (0,0) circle (2pt);
                \fill[blue, shift={(0.5,0.5)}] (0,0) circle (2pt);

                \draw [->, blue, thick] (0.5,0.5) -- (0,0);
                
				\draw [dashed, gray] (0,1) -- (1,0);
				\draw [->, black] (0,0) -- (1.5,0);
				\draw [->, black] (0,0) -- (0,1.5);
                \draw [->, black] (0,0) -- (-0.7,-0.7);
                \node[] at (-0.7,-0.5) {$x$};
                \node[] at (-0.2,1.5) {$z$};
                \node[] at (1.3,-0.2) {$y$};
                \node[] at (-0.3,1) {$\vp_i^{(l)}$};
                \node[] at (1.3,0.3) {$\vp_t^{(l)}$};
                \node[] at (-0.2,-0.5) {$\vp_j^{(l)}$};
                \node[] at (0.2,-0.2) {$\vp_s^{(l)}$};
                \node[blue] at (0.7,0.7) {$\vb$};
                \node[] at (-0.2,1.5) {$z$};
                \node[] at (-0.3,0) {$O$};
			\end{tikzpicture}
		}
            \hspace{4em} 
            \subfigure[Four points are on the same plane, but not on the same line.]{
            \label{fig:A3-b}
			\begin{tikzpicture}[scale = 1.5]
				\fill[orange, shift={(0.3,0)}] (0,0) circle (2pt);
				\fill[orange, shift={(0,1)}] (0,0) circle (2pt);
				\fill[orange, shift={(1,0.75)}] (0,0) circle (2pt);
				\fill[orange, shift={(-0.5,0)}] (0,0) circle (2pt);
                \fill[blue, shift={(0.5,0.5)}] (0,0) circle (2pt);

                \draw [->, blue, thick] (0.5,0.5) -- (0,0);
                
				\draw [dashed, gray] (-0.5,0) -- (1,0.75);
                \draw [dashed, gray] (0.3,0) -- (0,1);
				\draw [->, black] (0,0) -- (1.5,0);
				\draw [->, black] (0,0) -- (0,1.5);
                \draw [->, black] (0,0) -- (-0.7,-0.7);
                \draw [gray, dashed] (0,0) -- (-0.7,0);
                \node[] at (-0.7,-0.5) {$x$};
                \node[] at (-0.2,1.5) {$z$};
                \node[] at (1.3,-0.2) {$y$};
                \node[] at (-0.6,0.2) {$\vp_i^{(l)}$};
                \node[] at (0.9,0.9) {$\vp_t^{(l)}$};
                \node[] at (-0.3,1) {$\vp_j^{(l)}$};
                \node[] at (0.3,-0.2) {$\vp_s^{(l)}$};
                \node[blue] at (0.6,0.4) {$\vb$};
                \node[] at (-0.2,1.5) {$z$};
                \node[] at (-0.3,0) {$O$};
			\end{tikzpicture}
		      }
            \hspace{4em}
            \subfigure[Four points are on the same line.]{
            \label{fig:A3-c}
			\begin{tikzpicture}[scale = 1.5]
				\fill[orange, shift={(0.2,0.35)}] (0,0) circle (2pt);
				\fill[orange, shift={(0,0.25)}] (0,0) circle (2pt);
				\fill[orange, shift={(1,0.75)}] (0,0) circle (2pt);
				\fill[orange, shift={(-0.5,0)}] (0,0) circle (2pt);
                \fill[blue, shift={(1,0.25)}] (0,0) circle (2pt);

                \draw [->, blue, thick] (1,0.25) -- (0,0);

                \draw [dashed, gray] (1,0.75) arc [start angle=26.57,end angle=386.57, radius=0.56];

                \draw [dashed, gray] (0.2,0.35) arc [start angle=26.57,end angle=386.57, radius=0.39];
                
				\draw [dashed, gray] (-0.5,0) -- (1,0.75);
                \draw [dashed, red] (0,0.25) -- (1,0.25) -- (1,0.75) -- (0,0.25);
                \draw [red] (0.9,0.25) -- (0.9,0.35) -- (1,0.35);
				\draw [->, black] (0,0) -- (1.5,0);
				\draw [->, black] (0,0) -- (0,1.5);
                \draw [->, black] (0,0) -- (-0.7,-0.7);
                \draw [gray, dashed] (0,0) -- (-0.7,0);
                \node[] at (-0.7,-0.5) {$x$};
                \node[] at (-0.2,1.5) {$z$};
                \node[] at (1.3,-0.2) {$y$};
                \node[] at (-0.2,0.3) {$\vp_i^{(l)}$};
                \node[] at (1.2,1) {$\vp_t^{(l)}$};
                \node[] at (-0.7,0.2) {$\vp_j^{(l)}$};
                \node[] at (0.35,0.6) {$\vp_s^{(l)}$};
                \node[blue] at (1.1,0.1) {$\vb$};
                \node[] at (-0.2,1.5) {$z$};
                \node[] at (-0.3,0) {$O$};
			\end{tikzpicture}
		      }
            \caption{Three cases of the four points. We figure out that $\vb$ is the shift direction and distance, and becomes the new origin when we translate $\mP^{(l)}$ to $\mH^{(l)}$.}
        \end{figure*}
        
        Case 1: Suppose the four points $ \vp_i^{(l)}, \vp_j^{(l)}, \vp_s^{(l)}, \vp_t^{(l)} $ are not on the same plane, as shown in Figure A\ref{fig:A3-a}. Choose $ \vb = -(\vp_i^{(l)} + \vp_t^{(l)})/2 $, we thus have $ \angle\vh_i^{(l)} O  \vh_t^{(l)} = \pi $. However, $ \angle \vh_j^{(l)} O \vh_s^{(l)} \in(0,\pi) $, otherwise the four points will belong to the same plane. Therefore, $\vb\notin\sB_4' (\vp_i^{(l)}, \vp_j^{(l)}, \vp_s^{(l)}, \vp_t^{(l)})$. 

        Case 2: Suppose the four points $ \vp_i^{(l)}, \vp_j^{(l)}, \vp_s^{(l)}, \vp_t^{(l)} $ are on the same plane, but not on the same line, as shown in Figure A\ref{fig:A3-b}. We can always find $-\vb$ on the \textbf{line segment} $\overline{\vp_i^{(l)}\vp_t^{(l)}} $, and ensure $-\vb$ is not on the \textbf{line} $\overline{\vp_j^{(l)}\vp_s^{(l)}} $, otherwise the four points will be on the same line. We thus have $ \angle\vh_i^{(l)} O  \vh_t^{(l)} = \pi $, but $ \angle \vh_j^{(l)} O \vh_s^{(l)} \in(0,\pi) $. Therefore, $\vb\notin\sB_4' (\vp_i^{(l)}, \vp_j^{(l)}, \vp_s^{(l)}, \vp_t^{(l)})$. 
        
        Case 3: Suppose the four points $ \vp_i^{(l)}, \vp_j^{(l)}, \vp_s^{(l)}, \vp_t^{(l)} $ are on the same line, as shown in Figure A\ref{fig:A3-c}. We can draw circles with $\overline{\vp_i^{(l)}\vp_t^{(l)}} $ and $\overline{\vp_j^{(l)}\vp_s^{(l)}} $, respectively. We can always find $-\vb$ on the previous circle, but not on the later one, otherwise they will be not different from each other. We thus have $ \angle\vh_i^{(l)} O  \vh_t^{(l)} = \frac{\pi}{2}$, but $ \angle \vh_j^{(l)} O \vh_s^{(l)} \ne \frac{\pi}{2} $. Therefore, $\vb\notin\sB_4' (\vp_i^{(l)}, \vp_j^{(l)}, \vp_s^{(l)}, \vp_t^{(l)})$. 

        Conclusively, we can always find some $ \vb \notin \sB_4' (\vp_i^{(l)}, \vp_j^{(l)}, \vp_s^{(l)}, \vp_t^{(l)})  $, then we have $ \sB_4' (\vp_i^{(l)}, \vp_j^{(l)}, \vp_s^{(l)}, \vp_t^{(l)}) \subset \R^n $. Further, $\sB_4' (\vp_i^{(l)}, \vp_j^{(l)}, \vp_s^{(l)}, \vp_t^{(l)})$ is a hypersurface of $d-1$ dimension, and $\sB_4 (\vp_i^{(l)}, \vp_j^{(l)}, \vp_s^{(l)}, \vp_t^{(l)})\subset \sB_4' (\vp_i^{(l)}, \vp_j^{(l)}, \vp_s^{(l)}, \vp_t^{(l)})$
    \end{proof}

    Here we propose the proposition of a wider LN-Net as follows. 

    \begin{proposition}
        A wider LN-Net can classify $m$ samples with any label assignment. 
    \end{proposition}
    \begin{proof}
	Similarly, we hope to merge two points from the same class, and do not merge other points meanwhile. Suppose LN acts on $\R^{n+1}$ by Lemma \ref{lemma:equivalence}, we thus use SP on $\R^n$ for convenience. Given $\mP^{(l)}\in\R^{n \times m}$ on a $n-1$ dimensional hyperplane, we consider to shift the points by $ \vb \in \R^n $ and get $\mH^{(l)}$. After that, we spherically project $\mH^{(l)}$ onto the unit sphere $ \|\rvx\|_2=1 $, represented by $ \mX^{(l+1)} $. Hereafter, we linearly project $ \mX^{(l+1)} $ onto another $ n - 1 $ dimensional hyperplane. 
	
	Different from our method on $\R^2$, we can not sort the points, it is hence much harder to design a suitable algorithm in a high dimensional space. But we can consider to merge some $\vp_i^{(l)} $ and $ \vp_j^{(l)} $ only, without merging the other points. We analyze the merging progress backward, and show how to find the projection direction and the bias $\vb$. 
	
	To get $ \mP^{(l+1)} $ from $ \mX^{(l)} $, without doubt the projection direction is along $ \overline{\rvx_i^{(l)} \rvx_j^{(l)}} $, and the target is some $ n - 1 $ dimensional hyperplane. Now we need to ensure doing so will not merge other points. Obviously, its necessary and sufficient condition is that there are no other different points $ \rvx_s^{(l)}, \rvx_t^{(l)} $, such that
    \begin{equation}
        \overline{\rvx_i^{(l)} \rvx_j^{(l)}} /\mskip-2.5mu/  \overline{\rvx_s^{(l)} \rvx_t^{(l)}}, 
    \end{equation}
    namely $ \overline{\rvx_s^{(l)} \rvx_t^{(l)}} $ is parallel to the projection direction. 
    
    According to Lemma \ref{lemma:8}, for $ \mX^{(l)} $ is on the unit sphere, the necessary condition of $ \overline{\rvx_i^{(l)} \rvx_j^{(l)}} /\mskip-2.5mu/  \overline{\rvx_s^{(l)} \rvx_t^{(l)}} $ is that---$ \angle\rvx_i^{(l)} O  \rvx_s^{(l)} =  \angle \rvx_j^{(l)} O \rvx_t^{(l)} $, where $O$ is the origin of coordinates.

    Since $ \mX^{(l)} = SP(\mH^{(l)})$, we have
    \begin{equation}
        \angle\rvx_i^{(l)} O  \rvx_s^{(l)} =  \angle \rvx_j^{(l)} O \rvx_t^{(l)} \Leftrightarrow  \angle\vh_i^{(l)} O  \vh_s^{(l)} =  \angle \vh_j^{(l)} O \vh_t^{(l)}.
    \end{equation}
    If we ensure any four different points in $ \mH^{(l)} $ to satisfy $ \angle\vh_i^{(l)} O  \vh_s^{(l)} \ne  \angle \vh_j^{(l)} O \vh_t^{(l)} $, we will not merge other points when we merge $ \rvx_i^{(l+1)} $ and $ \rvx_j^{(l+1)} $. Since $ \vh_k^{(l)} = \vp_k^{(l)} +\vb $, according to the cosine theorem, we point out that $ \angle\vh_i^{(l)} O  \vh_s^{(l)} =  \angle \vh_j^{(l)} O \vh_t^{(l)} $ is equivalent to 
	\begin{equation}
		\label{eqn:96}
		\frac{(\vp_i^{(l)} + \vb)^\top ( \vp_s^{(l)} + \vb)}{\|\vp_i^{(l)} + \vb\|_2 \|\vp_s^{(l)} + \vb\|_2} = \frac{(\vp_j^{(l)} + \vb)^\top ( \vp_t^{(l)} + \vb)}{\|\vp_j^{(l)} + \vb\|_2 \|\vp_t^{(l)} + \vb\|_2}. 
	\end{equation}

    We define
    \begin{equation}
        \sB_4 (\vp_i^{(l)}, \vp_j^{(l)}, \vp_s^{(l)}, \vp_t^{(l)}) = \left\{ \vb \in \R^n: \frac{(\vp_i^{(l)} + \vb)^\top ( \vp_s^{(l)} + \vb)}{\|\vp_i^{(l)} + \vb\|_2 \|\vp_s^{(l)} + \vb\|_2} = \frac{(\vp_j^{(l)} + \vb)^\top ( \vp_t^{(l)} + \vb)}{\|\vp_j^{(l)} + \vb\|_2 \|\vp_t^{(l)} + \vb\|_2} \right\}. 
    \end{equation}
    
    Since $ \vp_i^{(l)}, \vp_j^{(l)}, \vp_s^{(l)}, \vp_t^{(l)} $ are different from each other, the solution space of Eqn.\ref{eqn:96} about $ \vb $ is contained in a hypersurface of $n-1$ dimension, by Lemma \ref{lemma:9}. 
 
    Again, we define 
    \begin{equation}
          \hat\sB_4 (\mP^{(l)}) = \bigcup_{(i,j,s,t)\in \sI_{4}(\mP^{(l)})}  \sB_4 (\vp_i^{(l)}, \vp_j^{(l)}, \vp_s^{(l)}, \vp_t^{(l)}),
    \end{equation} 
    where
    \begin{equation}
        \sI_{4} (\mP^{(l)}) = \{ (i,j,s,t): \vp_i^{(l)}, \vp_j^{(l)}, \vp_s^{(l)}, \vp_t^{(l)} \text{ are different with each other} \}.
    \end{equation}
    We figure out that $\hat\sB_4 (\mP^{(l)})$ is contained in a union of no more than $ m^4 $ hypersurfaces of $n-1$ dimension.
	
	Besides, from $ \mP^{(l)} $ to $ \mX^{(l)} $, we can not merge any two different points. Therefore, given $ \vp_i^{(l)} \ne \vp_j^{(l)} $, we need
    \begin{equation*}
        (\vp_{i}^{(l)} + \vb)/\|\vp_{i}^{(l)} + \vb\|_2 \ne (\vp_{j}^{(l)} + \vb)/\|\vp_{j}^{(l)} + \vb\|_2.
    \end{equation*}
   Given two different points $ \vp_i, \vp_j $, we define 
    \begin{equation}
        \sB_2 (\vp_i^{(l)}, \vp_j^{(l)}) = \left\{ \vb \in \R^n: \frac{\vp_i^{(l)} + \vb}{\|\vp_i^{(l)} + \vb\|_2 } = \frac{\vp_j^{(l)} + \vb}{\|\vp_j^{(l)} + \vb\|_2 } \right\}.
    \end{equation}
    Similarly, we can prove that $ \sB_2 (\vp_i, \vp_j) $ is contained in a hypersurface of $n-1$ dimension. We find $ \hat\sB_2 (\mP^{(l)}) $ is contained in the union of no more than $ m^2 $ hypersurfaces of $n-1$ dimension, where
    \begin{equation}
        \hat\sB_2 (\mP^{(l)}) = \bigcup_{ \vp_i^{(l)} \ne \vp_j^{(l)} } \sB_2 (\vp_i^{(l)}, \vp_j^{(l)}). 
    \end{equation}

    We figure out that $ \hat\sB_2 (\mP^{(l)}) \cup \hat\sB_4 (\mP^{(l)}) $ is contained in a union of no more than $ m^2+m^4 $ hypersurfaces of $n-1$ dimension. 
    
    Therefore, we have
    \begin{equation}
        [\hat\sB_2 (\mP^{(l)}) \cup \hat\sB_4 (\mP^{(l)})] \subset \R^n
    \end{equation}
    
    Choose some $ \vb \in \R^n / [\hat\sB_2 (\mP^{(l)}) \cup \hat\sB_4 (\mP^{(l)})] $, then $ \angle\vh_j^{(l)} O  \vh_s^{(l)} =  \angle \vh_i^{(l)} O \vh_t^{(l)} $ will not holds. Furthermore, by Lemma \ref{lemma:8}, $ \overline{\rvx_i^{(l)} \rvx_j^{(l)}} /\mskip-2.5mu/  \overline{\rvx_s^{(l)} \rvx_t^{(l)}}  $ will not holds either. As a result, we can only merge $ \vp_i^{(l)} $ and $ \vp_j^{(l)} $ by projection. 

    In conclusion, we can choose to only merge two samples with the same label each step by the method above. Furthermore, we can construct an LN-Net with depth $O(m)$ to classify $m$ samples with any label assignment. Note the width of LN-Net here is wider than $3$, and we do not require the widths of each layer are equal. 
    \end{proof}

    \section{Proof of Proposition \ref{prop:Hessian}}
    \label{section:proofofhessian}

    \paragraph{Proposition 8.}\textit{Given $ g\le d/3 $, we have
        \begin{equation}
            \frac{\mathcal{H}( \psi_G(g;\cdot); \rvx )}{\mathcal{H}( \psi_L(\cdot); \rvx )} \ge 1. 
        \end{equation}
        Specifically, when $ g = d/4 $, we figure out that
        \begin{equation}
            \frac{\mathcal{H}( \psi_G(g;\cdot); \rvx )}{\mathcal{H}( \psi_L(\cdot); \rvx )} \ge \frac{d}{8} . 
        \end{equation}}

    In the proof of Proposition \ref{prop:Hessian}, we consider a single sample only. We use $ x_i $ as the $i$-th ordinate of $ \rvx $ instead of $ x^{(i)} $ in this proof, we thus use $x_i^2$ to denote the squares rather than $[x^{(i)}]^2$. 

    \subsection{Required Lemmas for the Proof}

    \begin{lemma}
        Given $ \rvx \in \R^d $, $ \mu = (x_1 + \cdots + x_d)/d $ and $ \sigma^2 = [(x_1-\mu)^2 + \cdots + (x_d-\mu)^2]/d $, we denote $ LN(\rvx) $ as $ \hat\rvx = (\rvx-\mu\bm1)/\sigma $. We point out that 
        \begin{equation}
            \mathcal{H}(\psi_L(\cdot);\rvx) = \frac{3}{\sigma^4} - \frac{6}{d\sigma^4} 
        \end{equation} 
    \end{lemma}

    \begin{proof}
        To begin with, we regard $ \hatxn{i} $ as $ \psi_{i} (\rvx) $, and then give the gradient $ \nabla_\rvx \psi_{i} (\rvx) $. Let $ s = \frac{1}{d} \sum\limits_{i=1}^d (\xn{i} - \mu)^2 $ and $ \sigma = \sqrt{s} $. We have
	\begin{equation}
		\pf{\mu}{\xn{i}} = \frac1d, \forall i,
	\end{equation}
    \begin{equation}
        \begin{aligned}
            \pf{s}{\xn{i}}&= \frac1d\pf{}{\xn{i}} \sum_{j=1}^d (\xn{j} - \mu)^2\\
			&= \frac1d\pf{}{\xn{i}}\sum_{j=1}^d \xn{j}^2 - \frac1d \pf{}{\xn{i}}d\mu^2 \\
			&=\frac2d(\xn{i}-\mu), \forall i,
        \end{aligned}
    \end{equation}
    and
    \begin{equation}
        \begin{aligned}
            \pf{\sigma}{\xn{i}} &= \frac1{2\sqrt{s}} \pf{s}{\xn{i}}\\
			&=\frac{\xn{i}-\mu}{d\sigma} = \frac{\hatxn{i}}{d}, \forall i. 
        \end{aligned}
    \end{equation}
    We thus obtain
	\begin{equation}
		\begin{aligned}
			\pf{\hatxn{i}}{\xn{i}} &= \frac1{\sigma} \pf{}{\xn{i}}(\xn{i}-\mu) + (\xn{i}-\mu)  \pf{}{\xn{i}}(\frac1\sigma)\\
			&= \frac1{\sigma}(1-\frac1d) - \frac{\hatxn{i}}{\sigma} \pf{\sigma}{\xn{i}} \\
			&= \frac1{d\sigma} (d-1-\hatxn{i}^2). 
		\end{aligned}
	\end{equation}
	While for $ j\ne i $, we have
	\begin{equation}
		\begin{aligned}
			\pf{\hatxn{i}}{\xn{j}} &= \frac1{\sigma} \pf{}{\xn{j}}(\xn{i}-\mu) + (\xn{i}-\mu)  \pf{}{\xn{j}}(\frac1\sigma)\\
			&= \frac1{\sigma}(0-\frac1d) - \frac{\hatxn{i}}{\sigma} \pf{\sigma}{\xn{j}} \\
			&= \frac1{d\sigma} (-1-\hatxn{i}\hatxn{j}).
		\end{aligned}
	\end{equation}
	
	Based above, we calculate the Hessian matrix. For each term $\displaystyle \pf{^2\hatxn{i}}{\xn{j} \partial \xn{k}} $, we figure out that there are four kinds of the second order derivative.
	
	Case $1$, $i=j=k$:
	\begin{equation}
		\begin{aligned}
			\pf{^2\hatxn{i}}{\xn{i}^2} &= -\frac1{d\sigma^2}(d-1-\hatxn{i}^2)\pf{\sigma}{\xn{i}} - \frac{2\hatxn{i}}{d\sigma}\pf{\hatxn{i}}{\xn{i}} \\
			&= - \frac1{d^2\sigma^2} (d-1-\hatxn{i}^2) \hatxn{i} - \frac{2\hatxn{i}}{d^2\sigma^2} (d-1-\hatxn{i}^2)\\
			&= \frac{1}{d^2\sigma^2}[3\hatxn{i}^3 - 3(d-1)\hatxn{i}]\\
			&= \frac{1}{d^2\sigma^2} (3\hatxn{i}^3 + 3\hatxn{i}) - \frac{3\hatxn{i}}{d\sigma^2}. 
		\end{aligned}
	\end{equation}
	
	Case $2$, only one of $j,k$ equals to $i$, assume $i=k$:
	\begin{equation}
		\begin{aligned}
			\pf{^2\hatxn{i}}{\xn{i}\partial\xn{j}} &= -\frac1{d\sigma^2}(d-1-\hatxn{i}^2)\pf{\sigma}{\xn{j}} - \frac{2\hatxn{i}}{d\sigma}\pf{\hatxn{i}}{\xn{j}} \\
			&= - \frac1{d^2\sigma^2} (d-1-\hatxn{i}^2) \hatxn{j} - \frac{2\hatxn{i}}{d^2\sigma^2} (-1-\hatxn{i}\hatxn{j})\\
			&= \frac{1}{d^2\sigma^2} [3\hatxn{i}^2\hatxn{j} + 2\hatxn{i} - (d-1)\hatxn{j}]\\
			&= \frac{1}{d^2\sigma^2} (3\hatxn{i}^2\hatxn{j} + 2\hatxn{i} +\hatxn{j}) - \frac{\hatxn{j}}{d\sigma^2}. 
		\end{aligned}
	\end{equation}
    We have that $\displaystyle\pf{^2\hatxn{i}}{\xn{j} \partial \xn{k}} = \pf{^2\hatxn{i}}{\xn{k} \partial \xn{j}}$, so the result of the other case $i=j$ has the same form with that of $i=k$. 
	
	Case $3$, $j=k$, but $i\ne j$:
	\begin{equation}
		\begin{aligned}
			\pf{^2\hatxn{i}}{\xn{j}^2} &= -\frac1{d\sigma^2}(-1-\hatxn{i}\hatxn{j})\pf{\sigma}{\xn{j}} - \frac{\hatxn{i}}{d\sigma}\pf{\hatxn{j}}{\xn{j}} - \frac{\hatxn{j}}{d\sigma}\pf{\hatxn{i}}{\xn{j}} \\
			&= - \frac1{d^2\sigma^2} (- 1 - \hatxn{i} \hatxn{j}) \hatxn{j} - \frac{\hatxn{i}}{d^2 \sigma^2}(d-1-\hatxn{j}^2) - \frac{\hatxn{j}}{d^2\sigma^2} (-1-\hatxn{i}\hatxn{j}) \\
			&= \frac{1}{d^2\sigma^2} [3\hatxn{i}\hatxn{j}^2 + 2\hatxn{j} - (d-1)\hatxn{i}]\\
			&= \frac{1}{d^2\sigma^2} (3\hatxn{i}\hatxn{j}^2 + 2\hatxn{j} +\hatxn{i}) - \frac{\hatxn{i}}{d\sigma^2}. 
		\end{aligned}
	\end{equation}
	
	Case $4$, $i,j,k$ are different from each other:
	\begin{equation}
		\begin{aligned}
			\pf{^2\hatxn{i}}{\xn{j} \partial \xn{k}} &= -\frac1{d\sigma^2}(-1-\hatxn{i}\hatxn{j})\pf{\sigma}{\xn{k}} - \frac{\hatxn{i}}{d\sigma}\pf{\hatxn{j}}{\xn{k}} - \frac{\hatxn{j}}{d\sigma}\pf{\hatxn{i}}{\xn{k}} \\
			&= - \frac1{d^2\sigma^2} (- 1 - \hatxn{i} \hatxn{j}) \hatxn{k} - \frac{\hatxn{i}}{d^2 \sigma^2}(-1-\hatxn{j}\hatxn{k}) - \frac{\hatxn{j}}{d^2\sigma^2} (-1-\hatxn{i}\hatxn{k}) \\
			&= \frac{1}{d^2\sigma^2}(3\hatxn{i} \hatxn{j} \hatxn{k} + \hatxn{i} + \hatxn{j} + \hatxn{k}). 
		\end{aligned}
	\end{equation}
	
	It is hard to calculate the operator norm of the Hessian matrix is too difficult, so we calculate the Frobenius norm instead. 
	\begin{equation}
		\label{eqn:50}
		\begin{aligned}
			\left\Vert \pf{^2\hatxn{i}}{\rvx^2} \right\Vert_F^2 =& \sum_{j=1}^d \sum_{k=1}^d \left(\pf{^2\hatxn{i}}{\xn{j} \partial \xn{k}}\right)^2 \\
			=& \sum_{j=1}^d \sum_{k=1}^d \frac{1}{d^4\sigma^4}(3\hatxn{i} \hatxn{j} \hatxn{k} + \hatxn{i} + \hatxn{j} + \hatxn{k})^2 
			+ \sum_{j\ne i} \left[\frac{\hatxn{i}^2}{d^2\sigma^4}  -2 \frac{\hatxn{i}}{d^3\sigma^4} (3\hatxn{i}\hatxn{j}^2 + 2\hatxn{j} +\hatxn{i})\right]\\
			&+ 2 \sum_{j\ne i} \left[\frac{\hatxn{j}^2}{d^2\sigma^4} - 2  \frac{\hatxn{j}}{d^3\sigma^4}  (3\hatxn{i}^2\hatxn{j} + 2\hatxn{i} +\hatxn{j})\right]
			+ \frac{9\hatxn{i}^2}{d^2\sigma^4} -2 \frac{3\hatxn{i}}{d^3\sigma^4} (3\hatxn{i}^3 + 3\hatxn{i}) \\
			=& \sum_{j=1}^d \sum_{k=1}^d \frac{1}{d^4\sigma^4}(3\hatxn{i} \hatxn{j} \hatxn{k} + \hatxn{i} + \hatxn{j} + \hatxn{k})^2 
			+ \sum_{j=1}^d \left[\frac{\hatxn{i}^2}{d^2\sigma^4} - 2 \frac{\hatxn{i}}{d^3\sigma^4} (3\hatxn{i}\hatxn{j}^2 + 2\hatxn{j} +\hatxn{i})\right]\\
			&+ 2 \sum_{j=1}^d \left[\frac{\hatxn{j}^2}{d^2\sigma^4} - 2 \frac{\hatxn{j}}{d^3\sigma^4}  (3\hatxn{i}^2\hatxn{j} + 2\hatxn{i} +\hatxn{j})\right]
			+ \frac{6\hatxn{i}^2}{d^2\sigma^4} \\
		\end{aligned}
	\end{equation}
	
	We note that 
	\begin{equation}
		\sum_{j=1}^d \hatxn{j} = 0, \sum_{j=1}^d \hatxn{j}^2 = d. 
	\end{equation}
	We thus have 
	\begin{equation}
		\label{eqn:52}
		\begin{aligned}
			& \sum_{j=1}^d \sum_{k=1}^d \frac{1}{d^4\sigma^4}(3\hatxn{i} \hatxn{j} \hatxn{k} + \hatxn{i} + \hatxn{j} + \hatxn{k})^2 \\
			=& \sum_{j=1}^d \sum_{k=1}^d \frac{1}{d^4\sigma^4}[9\hatxn{i}^2 \hatxn{j}^2 \hatxn{k}^2 + 6\hatxn{i} \hatxn{j} \hatxn{k}(\hatxn{i} + \hatxn{j} + \hatxn{k}) + (\hatxn{i} + \hatxn{j} + \hatxn{k})^2] \\
			=& \frac{9\hatxn{i}^2}{d^2\sigma^4} + 0 + \sum_{j=1}^d \sum_{k=1}^d \frac{1}{d^4\sigma^4} (\hatxn{i} + \hatxn{j} + \hatxn{k})^2\\
			=& \frac{10\hatxn{i}^2}{d^2\sigma^4} + \frac{2}{d^2\sigma^4} ,
		\end{aligned}
	\end{equation}
	
	\begin{equation}
		\label{eqn:53}
		\sum_{j=1}^d \left[\frac{\hatxn{i}^2}{d^2\sigma^4}  -2 \frac{\hatxn{i}}{d^3\sigma^4} (3\hatxn{i}\hatxn{j}^2 + 2\hatxn{j} +\hatxn{i})\right] =  \frac{\hatxn{i}^2}{d\sigma^4} - \frac{8\hatxn{i}^2}{d^2\sigma^4}, 
	\end{equation}
	and
	\begin{equation}
		\label{eqn:54}
		2 \sum_{j=1}^d \left[\frac{\hatxn{j}^2}{d^2\sigma^4} - 2 \frac{\hatxn{j}}{d^3\sigma^4}  (3\hatxn{i}^2\hatxn{j} + 2\hatxn{i} +\hatxn{j})\right] = \frac{2}{d\sigma^4} - \frac{12\hatxn{i}^2}{d^2\sigma^4} - \frac{4}{d^2\sigma^4}. 
	\end{equation}
	
	Take Eqn.\ref{eqn:52}, Eqn.\ref{eqn:53} and Eqn.\ref{eqn:54} into Eqn.\ref{eqn:50}, we obtain
	\begin{equation}
		\left\Vert \pf{^2\hatxn{i}}{\rvx^2} \right\Vert_F^2 = \frac{\hatxn{i}^2+2}{d\sigma^4} - \frac{4\hatxn{i}^2+2}{d^2\sigma^4}
	\end{equation}
	
	Now we add up all the dimensions, as LN's information of the second order
	\begin{equation}
            \label{eqn:138}
		\mathcal{H}(\psi_L(\cdot);\rvx) = \sum_{i=1}^d \left\Vert \pf{^2\hatxn{i}}{\rvx^2} \right\Vert_F^2 = \frac{3}{\sigma^4} - \frac{6}{d\sigma^4} = \frac{3}{d\sigma^4}(d-2). 
	\end{equation}
	When $ d = 2 $, we have $ \hat x_i^2 = 1 $, and $ \mathcal{H}(\psi_L(\cdot);\rvx)|_{d=2} = 0 $ naturally. 

    \end{proof}

    \begin{lemma}
        Given $ \rvx \in \R^d $, let the group number of GN be $g$. Suppose $\sigma_i^2 $ is the variance of the $i$-th group, we have that 
        \begin{equation}
            \mathcal{H}(\psi_G(g;\cdot);\rvx) = \sum_{i=1}^{g} \left( \frac{3}{\sigma_i^4} - \frac{6g}{d \sigma_i^4} \right) 
        \end{equation}
    \end{lemma}

    \begin{proof}
        We simplify $ \psi_G(g;\cdot) $ as $ \psi(\cdot) $ in the proof here. As for Group Normalization, suppose the number of groups is $g$, and $ d = g \times c $. Let $ \rvx = [ \vz_1^\top, \cdots, \vz_g^\top]^\top $, where $ \vz_i = [ z_{i1}, \cdots, z_{ic} ]^\top , (i=1,\cdots,g) $. Assume $ \rvx = [x_1, \cdots, x_d]^\top $, we denote that $ z_{ij} = x_{(i - 1) \times c + j} $.
    
    Let $ \hat\rvx = GN(\rvx) $, where $ GN(\cdot) $ denotes the Group Normalization operation. GN can be calculated by $ \mu_i = (z_{i1} + \cdots + z_{ic})/c $, $ \sigma_i^2 = [(z_{i1}-\mu_i)^2 + \cdots + (z_{ic} - \mu_i)^2]/c $, and then $ \hat z_{ij} = (z_{ij} - \mu_i)/\sigma_i $. Accordingly, we denote $ \hat\rvx = [ \hat\vz_1^\top, \cdots, \hat\vz_g^\top]^\top $, where $ \hat\vz_i = LN(\vz_i), (i=1,\cdots,g) $. To begin with ,we denote $ GN(\rvx) $ as $\psi(\rvx) = [\psi_{11}(\rvx), \psi_{12}(\rvx), \cdots, \psi_{gc}(\rvx)]$. We thus have
    \begin{equation}
        \nabla_\rvx \psi_{ij}(\rvx) = \mat{ \nabla_{\vz_1} \psi_{ij}(\rvx) \\ \vdots \\ \nabla_{\vz_g}\psi_{ij}(\rvx) }, (i=1,\cdots,g;j=1,\cdots,c). 
    \end{equation}

    We denote that $ z_{ij} = \psi_{ij}(\rvx) $. When $k \ne i$, we have $ \nabla_{\vz_k}\psi_{ij}(\rvx) = \bm0 $. When $k \ne i$, we have $ \nabla_{\vz_i}\psi_{ij}(\rvx) $ is a gradient of LN, for $ [\psi_{i1}(\rvx), \cdots, \psi_{ic}(\rvx)]^\top = LN(\vz_i) $. We can give the Hessian matrix of $ GN $, denoted as
    \begin{equation}
        \nabla_{\rvx}^2\psi_{ij}(\rvx) = \mat{\mO & & \cdots & & \mO \\ 
        & \ddots & & &  \\ 
        \vdots & & \nabla_{\vz_i}^2\psi_{ij}(\rvx) & & \vdots \\
        & & & \ddots & \\
        \mO & & \cdots & & \mO}, (i=1,\cdots,g;j=1,\cdots,c)
    \end{equation}
    By the discussion about LN above, we obtain that 
    \begin{equation}
        \| \nabla_{\vz_i}^2\psi_{ij}(\rvx) \|_F^2 = \frac{\hat z_{ij}^2 + 2}{c\sigma_i^4} - \frac{4\hat z_{ij}^2 + 2}{c^2 \sigma_i^4}.
    \end{equation}
    Obviously, we have $ \| \nabla_{\rvx}^2\psi_{ij}(\rvx) \|_F^2 = \| \nabla_{\vz_i}^2\psi_{ij}(\rvx) \|_F^2 $. Although there are many zeros in $ \nabla_{\rvx}^2\psi_{ij}(\rvx) $, for $ \sum\limits_{j=1}^c \hat x_{ij}^2 = c $, we obtain
    \begin{equation}
        \label{eqn:142}
        \begin{aligned}
            \mathcal{H}(\psi_G(g;\cdot);\rvx) &= 
            \sum_{i=1}^d \left\| \pf{^2\hat x_i}{\rvx^2} \right\|_F^2 \\
            &= \sum_{i=1}^{g} \sum_{j=1}^{c} \| \nabla_{\rvx}^2\psi_{ij}(\rvx) \|_F^2 \\
            &= \sum_{i=1}^{g} \sum_{j=1}^{c} \left( \frac{\hat x_{ij}^2 + 2}{c\sigma_i^4} - \frac{4\hat x_{ij}^2 + 2}{c^2 \sigma_i^4} \right) \\ 
            &= \sum_{i=1}^{g} \left( \frac{3}{\sigma_i^4} - \frac{6}{c \sigma_i^4} \right) 
        \end{aligned}
    \end{equation}
    \end{proof}
    
	


        \begin{lemma}
        \label{lemma:11}
        In group normalization, we have
        \begin{equation}
            \sigma^2 \geq \frac{1}{g}\sum_{i=1}^g\sigma_i^2 .
        \end{equation}        
    \end{lemma}
    \begin{proof}
    According to the definition, we have
    \begin{equation}
        \begin{aligned}
             \sigma^2 - \frac{1}{g} \sum_{i=1}^{g}\sigma_i^2 &= \frac{1}{cg} \sum_{i=1}^g \sum_{j=1}^{c} (z_{ij} - \mu)^2 - \frac{1}{cg} \sum_{i=1}^g \sum_{j=1}^{c} (z_{ij} - \mu_i)^2\\
             &= \frac{1}{cg} \sum_{i=1}^g \sum_{j=1}^{c} (z_{ij}^2 - \mu^2) - \frac{1}{cg} \sum_{i=1}^g \sum_{j=1}^{c} (z_{ij}^2 - \mu_i^2)\\
            &= \frac1{g} \sum_{i=1}^g \mu_i^2 - \mu^2.
        \end{aligned}
    \end{equation}
    Since $ c(\mu_1 + \cdots + \mu_g) = cg\mu $, we have
    \begin{equation}
             \sigma^2 - \frac{1}{g} \sum_{i=1}^{g}\sigma_i^2 = \frac1{g} \sum_{i=1}^g \mu_i^2 - \mu^2 = \frac{1}{g} \sum_{i=1}^g (\mu_i - \mu)^2 \ge 0.
    \end{equation}
    \end{proof}

    \begin{lemma}
        \label{lemma:12}
        $f(x)=\frac{1}{x^2}$ is a monotonically decreasing and convex function on $x > 0$.
    \end{lemma}
    \begin{proof}
        For $f(x)=\frac{1}{x^2}$, we have $f'(x)=-\frac{2}{x^3}<0$, namely, $f(x)$ is monotonically decreasing. Furthermore, we have $f''(x)=\frac{6}{x^4}>0$, namely, $f(x)$ is a convex function.
    \end{proof}

    \begin{lemma}
        \label{lemma:13} Given that $ \sigma_1^2,\cdots,\sigma_g^2 $ and $ \sigma^2 $ are variances in LN-G and LN respectively, we have 
        \begin{equation}
            \sum_{i=1}^g \frac{1}{\sigma_i^4} \ge \frac{g}{\sigma^4}.
        \end{equation}
    \end{lemma}
    \begin{proof}
        According to Lemma \ref{lemma:12}, we have $f(x)=\frac{1}{x^2}$ is a convex function.  By Jensen's inequality, 
        we obtain
        \begin{equation}
            \sum_{i=1}^g\frac{1}{g} f(\sigma_k^2) \ge f\left(\frac{1}{g}\sum_{i=1}^g\sigma_i^2\right)
        \end{equation}
        According to Lemma \ref{lemma:11} and Lemma \ref{lemma:12}, we have
        \begin{equation}
            \sum_{i=1}^g\frac{1}{g} f(\sigma_i^2)\ge f(\sigma^2),  
        \end{equation}
        namely 
        \begin{equation}
            \frac{1}{g}\sum_{i=1}^g \frac{1}{\sigma_i^4} \ge \frac{1}{\sigma^4}. 
        \end{equation}
    \end{proof}
    \subsection{Proof of Proposition \ref{prop:Hessian}}
    \begin{proof}
    To prove
        \begin{equation}
            \frac{\mathcal{H}( \psi_G(g;\cdot); \rvx )}{\mathcal{H}( \psi_L(\cdot); \rvx )} \ge 1, 
        \end{equation}
    we can prove \Eqn\ref{eqn:186} instead: 
        \begin{equation}
            \label{eqn:186}
            \mathcal{H}( \psi_G(g;\cdot); \rvx )-\mathcal{H}( \psi_L(\cdot); \rvx ) \ge 0. 
        \end{equation}
        According to Eqn.\ref{eqn:138}, Eqn.\ref{eqn:142} and Lemma \ref{lemma:13}, we obtain
        \begin{equation}
            \label{eqn:187}
            \begin{aligned}
                \mathcal{H}( \psi_G(g;\cdot); \rvx )-\mathcal{H}( \psi_L(\cdot); \rvx ) &= \sum_{i=1}^g\left(\frac{3}{\sigma_i^4} - \frac{6}{c\sigma_i^4} \right) - \frac{3}{d\sigma^4}(d-2) \\
                &= 3\left(1-\frac{2}{c}\right)\sum_{i=1}^g\frac{1}{\sigma_i^4} - \frac{3}{\sigma^4}\left(1-\frac{2}{d}\right) \\
                &\ge \frac{3}{\sigma^4}\left(g-\frac{2g}{c} -1 + \frac{2}{d}\right) \\
                &= \frac{3}{d\sigma^4} (d-2g-2)(g-1).
            \end{aligned}
        \end{equation}
        
        When $ g \ge 2 $, we have $ d \ge 6 $. Therefore, we obtain 
        \begin{equation}
            d - 2g -2 = d - \frac{2d}{c} -2 \ge \frac 13d - 2 \ge 0
        \end{equation}

        According to \Eqn\ref{eqn:187}, we give the necessary condition for equality in
        \Eqn\ref{eqn:186}. One of the cases is $g=1$ obviously. The other case is $d=2g+2$ --- but we note that $g|d$, we hence have $g|2$. Namely $g=2,d=6$ is the only other case for equality. 
        
        Therefore, we have proved
        \begin{equation}
            \frac{\mathcal{H}( \psi_G(g;\cdot); \rvx )}{\mathcal{H}( \psi_L(\cdot); \rvx )} \ge 1. 
        \end{equation}
    As for the case $ g=d/4 $, we have that 
        \begin{equation}
            \begin{aligned}
                \mathcal{H}( \psi_G(g;\cdot); \rvx ) &\ge \frac{3}{\sigma^4}\left(g-\frac{2g}{c} \right) \\
                & = \frac{3}{\sigma^4}\left(g-\frac{2g^2}{d} \right) \\
                &= \frac{6}{d \sigma^4}\left(-g^2 +\frac{d}2 g \right) \\
                &= \frac{6}{d \sigma^4}\left( \frac{d^2}{16} - (g - \frac d4)^2 \right).
            \end{aligned}
        \end{equation}
        When $ g = d/4 $, the right term reaches its maximum, where we have 
        \begin{equation}
            \mathcal{H}( \psi_G(g;\cdot); \rvx ) \ge \frac{3d}{8\sigma^4}. 
        \end{equation}
        On the other hand, we have that
        \begin{equation}
            \mathcal{H}(\psi_L(\cdot);\rvx) = \frac{3}{\sigma^4} - \frac{6}{d\sigma^4} \le \frac{3}{\sigma^4}. 
        \end{equation}
        As a result, we obtain
        \begin{equation}
            \frac{\mathcal{H}( \psi_G(g;\cdot); \rvx )}{\mathcal{H}( \psi_L(\cdot); \rvx )} \ge \frac{d}{8} . 
        \end{equation}
    \end{proof}

    \subsection{$\mathcal{H}$ about ReLU}

    \REVISE{
    We conduct additional analyses to compare the nonlinearity of ReLU and LN during the phase of rebuttal. ReLU is defined as $\max(0,x)$, which is not differentiable strictly. To compare ReLU with LN, we consider to introduce the Dirac function $\delta(x)$ as ReLU's second-order derivative, namely $\nabla^2 ReLU(x)=\delta(x)$. We know that $\int_I\delta(x)dx=1$ and $\int_I f(x)\delta(x)dx=f(0)$. To apply the integral, we introduce the expectation, and assume $\rvx\sim N(0,\pmb I)$ is $d$-dimensional. Since we do not know how to calculate $\int_I f(x)\delta^2(x)dx$, we remove the square sign in $\mathcal{H}$. Specifically, we define $\bar{\mathcal{H}}(f;\rvx)$ as }
    \begin{equation}
        \bar{\mathcal{H}}(f;\rvx)=\sum_{i=1}^d\mathbb{E}_{\rvx}\left\Vert\frac{\partial^2 y_i}{\partial \rvx^2}\right\Vert_F,
    \end{equation}

    \REVISE{like Eqn.20 in our paper, and $y_i$ is defined similarly.}
    
    \REVISE{Based on the assumptions above, we have that}
    
    \begin{equation}
        \bar{\mathcal{H}}(relu(\cdot);\rvx) =\frac{d}{\sqrt{2\pi}}=O(d),
    \end{equation}
    
    \REVISE{and}
    
    \begin{equation}
        \bar{\mathcal{H}}(\psi_{L}(\cdot);\rvx) =\sum_{i=1}^d\mathbb{E}_{\rvx}\frac{1}{d\sigma^2}\sqrt{d(\hat x_i^2+2)-(4\hat x_i^2+2)}=O(\sqrt {d}).
    \end{equation}
    
    \REVISE{Furthermore, we have}
    
    \begin{equation}
        \bar{\mathcal{H}}(\psi_G(g;\cdot);\rvx) =g\cdot O(\sqrt c)=O(\sqrt{dg}).
    \end{equation}
    
    \REVISE{Note that we removed the square sign in $\mathcal{H}$, and there is a square root sign in $\bar{\mathcal{H}}$.}
    
    \REVISE{We hope the analysis above can help compare ReLU with LN, to some extent.}

    \section{Experiments}
    \subsection{Details of Experiments on Comparison of Representation Capacity by Fitting Random Labels.}
    \label{section:experiments}
    In this section, we provide the details of experimental setup in comparing the representation capacity by fitting random labels, as stated in Section~\ref{sec:Experiments}.

    \subsubsection{Dataset with Random Labels} 
    We conduct the random label datasets based on CIFAR-10 and MNIST, referred to as CIFAR-10-RL and MNIST-RL. In particular, for each sample of these datasets, we randomly assign a class label to this sample and save all the samples as a dataset. Even though the labels are random, the label assignment is certain once the dataset is conducted. Therefore, it is meaningful to compare the results of different methods by fitting random labels.


    \paragraph{MNIST-RL is more challenging.}
    Here, we highlight that MNIST-RL is more challenging in training a classifier for fitting the labels, compare to CIFAR-10-RL.
    Let $X_c$ represents examples belong to class $c$. It is clear that the features in $X_c$ are very close for the normal MNIST dataset. For example, all the digits of "0" are very similar in representation, they all have rounded curves. However, if we use the random label (the MNIST-RL dataset), the samples in $X_c$ will have different labels. In this case, the network will need to map $X_c$ --- which is very close in representation --- to different labels. As a result, we need more powerful model to fit MNIST-RL and is more difficult to train.


    \begin{figure*}[b]
        \vspace{-0.1in}
        \centering
        \hspace{0.05in}	
            \subfigure[]{
            \label{fig:5a}
            \includegraphics[height=4.5cm]{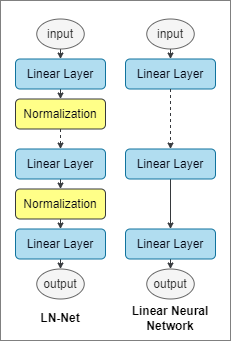}
        }
        \centering
        \hspace{0.05in}	
            \subfigure[]{
            \centering
            \label{fig:5b}
            \includegraphics[height=4.5cm]{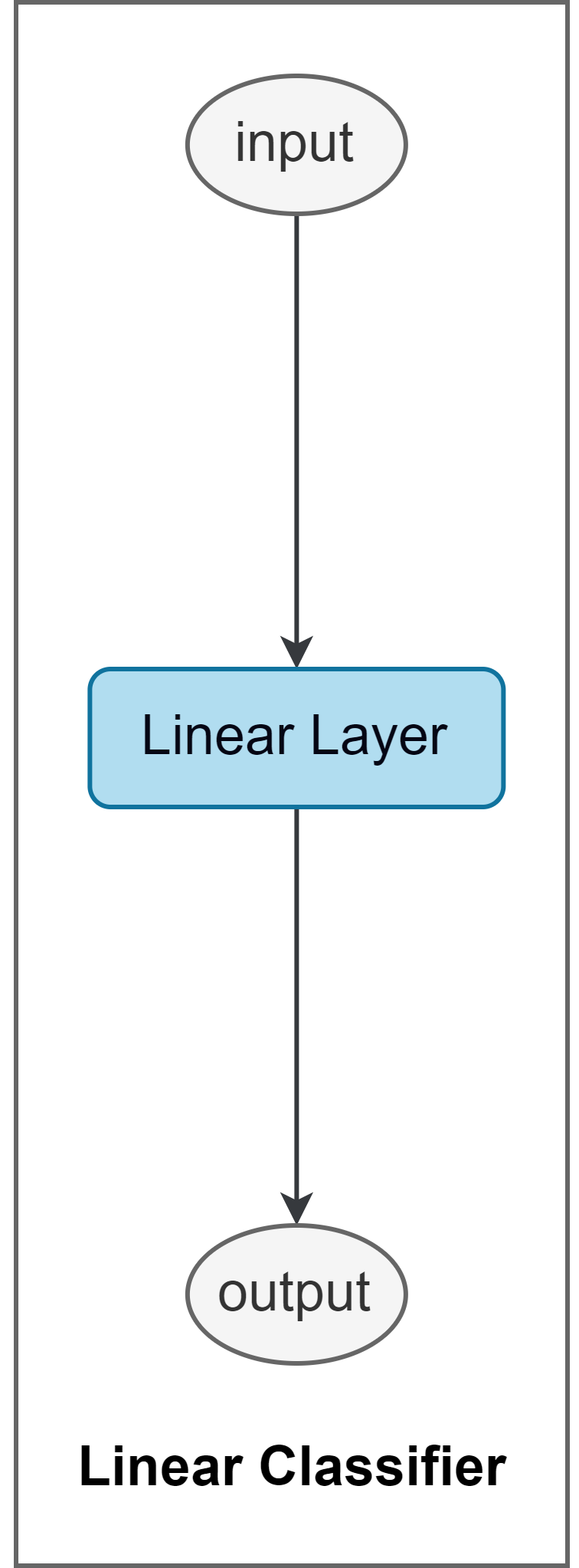}
        }
        \centering
        \hspace{0.05in}	
            \subfigure[]{
            \centering
            \label{fig:5c}
            \includegraphics[height=4.5cm]{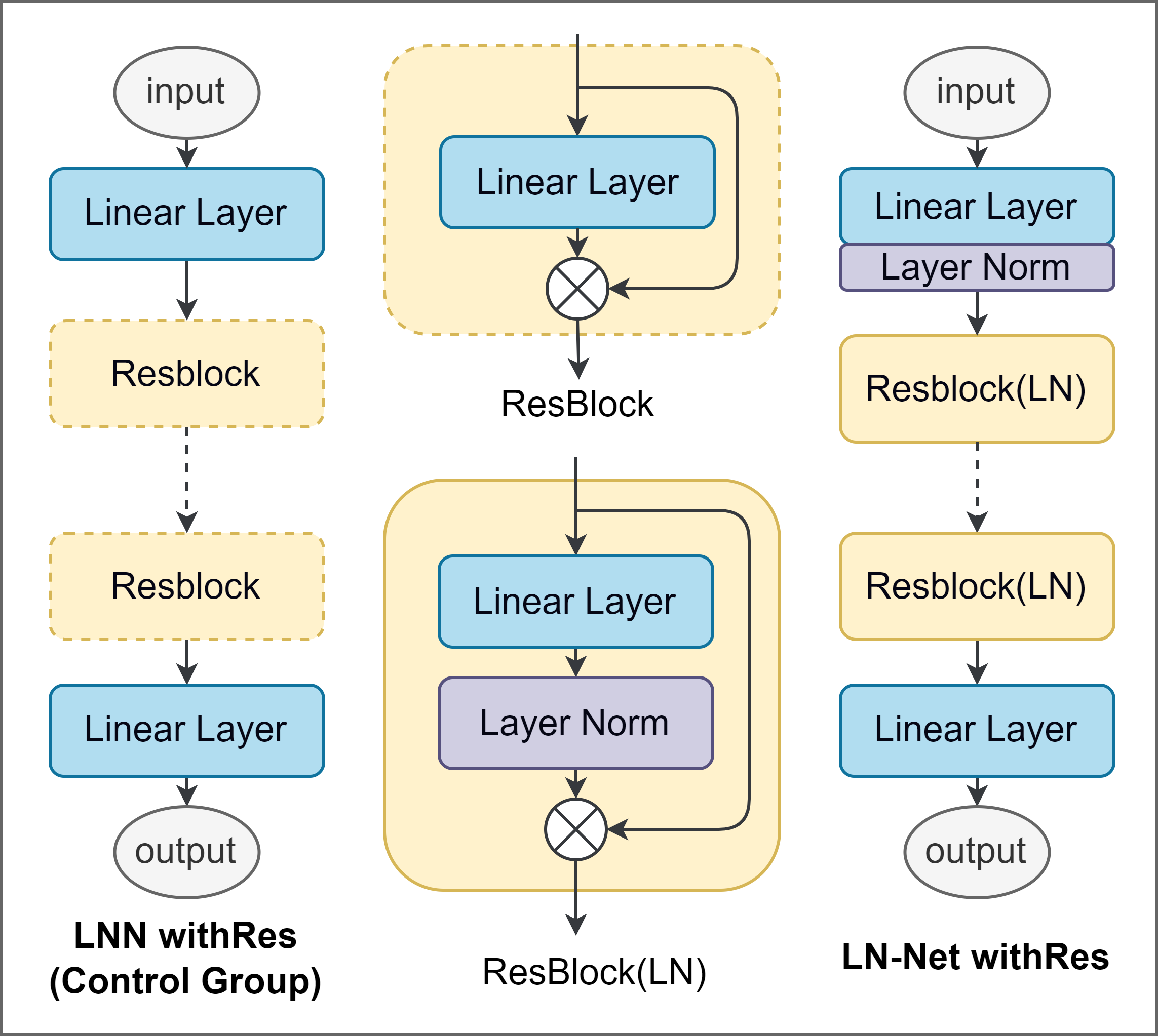}
        }
            \centering
        \hspace{0.05in}	
            \subfigure[]{
            \centering
            \label{fig:5d}
            \includegraphics[height=4.5cm]{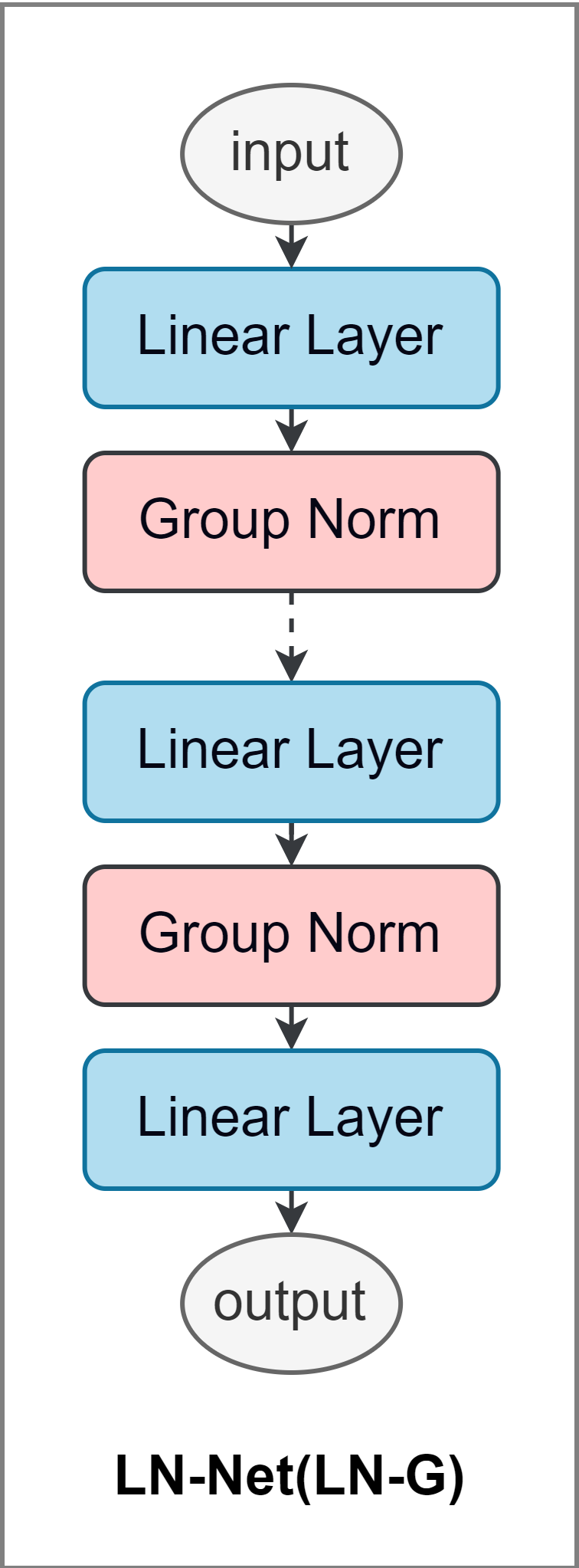}
        }
            \centering
        \vspace{-0.13in}
        \caption{Schematic representation of the networks used in the experiment. (a) Original LN-Net and Linear Neural Network (LNN). (b) Linear classifier. (c) LN-Net and LNN using residual connections. (d) LN-Net using LN-G.}
        \label{fig:5}
        \vspace{-0.12in}
    \end{figure*}

    \subsubsection{Details on Verifying Nonlinearity of LN}
    In this part, we use various  configurations of hyper-parameters to train our models, aiming at reducing the effect from the optimization.  
    We first sufficiently train a linear classifier (Figure~\ref{fig:5} (b)), as the baseline, which provides the (nearly) upper bound accuracy of linear classifier.  We then compare the results under linear neural network and LN-Net with residential structure for better optimization as shown in Figure~\ref{fig:5} (c). We vary the depths ranging in $\{2, 4, 6, 8, 10, 12  14\}$, and each hidden layer has a dimension of 256.
    
  \paragraph{Training protocols.}
	For the training of liner classifier, we apply both SGD optimizer with momentum (0.1) and Adam optimizer with betas $(0.9, 0.999)$. We train the model for 150 epochs and use a learning rate schedule with a decay 0.5 per 20 epochs. We search the batch sizes ranging in $\{128, 256\}$, the initial learning rates ranging in $\{0.001, 0.003, 0.005, 0.008, 0.05, 0.08, 0.1, 0.15\}$ and 5 random seeds,  and report the best accuracy from these configurations of hyper-parameters. 
    For the training of linear neural networks and LN-Nets, we follow the settings in training linear classifier, except that: 1) we use a batch size of 128 and a fixed random seed; 2) we search the initial learning rates ranging in $\{0.01, 0.03, 0.05, 0.08, 0.1\}$ for SGD and the initial learning rates ranging in  $\{0.001, 0.003, 0.005, 0.008, 0.05, 0.08, 0.1, 0.15 \}$ for Adam. 

    \paragraph{Results.}
    In Figure~\ref{fig:Res-CIFAR10} of the main paper, we show the best result of linear classifier as black dashed line, which is $18.51\%$ on CIFAR-10-RL and $15.38\%$ on MNIST-RL. We also provide the detailed results for linear neural network and LN-Net, shown in Table~\ref{table:1}.

   \begin{table*}[h]
       \centering
          \caption{The result of linear neural network and LN-Net model on classification task on CIFAR-10-RL and MNIST-RL. The bold numbers refer to those outperform linear classifier. We can see layer normalization breaks the bound of linearity.}
                 \label{table:1}
           \begin{tabular}{lcccc}
           \hline
           dataset & \multicolumn{2}{c}{RL-CIFAR-10}  & \multicolumn{2}{c}{RL-MNIST}     \\ \hline
           depth   & Linear+Res & LN+Res           & Linear+Res & LN+Res           \\ \hline
           2       & 17.37\%    & \textbf{20.45\%} & 14.71\%    & 14.45\%          \\
           4       & 17.00\%    & \textbf{27.97\%} & 14.54\%    & 15.26\%          \\
           6       & 16.97\%    & \textbf{39.24\%} & 14.29\%    & 15.28\%          \\
           8       & 17.02\%    & \textbf{39.39\%} & 14.32\%    & \textbf{15.76\%} \\
           10      & 16.98\%    & \textbf{31.12\%} & 13.89\%    & \textbf{18.26\%} \\
           12      & 16.91\%    & \textbf{50.48\%} & 13.35\%    & \textbf{18.47\%} \\
           14      & 15.19\%    & \textbf{55.58\%} & 13.98\%    & \textbf{19.44\%} \\ \hline
           best    & 17.37\%    & \textbf{55.58\%} & 14.71\%    & \textbf{19.44\%}\\ \hline
           \end{tabular}
       \vspace{-0.15in}
   \end{table*}   
    

    \subsubsection{Details on Amplifying the Nonlinearity using Group}
     We use the origin LN-Net and replace LN with LN-G, as shown in Figure~\ref{fig:5} (d). For the configuration of networks, we fix the number of neurons as 256 and vary the depths ranging in $ \{1, 2, 4, 6, 8, 10, 12, 14\}$. We vary the group numbers of LN-G ranging in $\{2, 4, 8, 16, 32, 64, 128\}$. 
    \paragraph{Training protocols.}
    We use the same training protocols as the experiment above, except that we only use SGD optimizer with fixed momentum of 0.1 and search the initial learning rate ranging in $ \{0.01, 0.03, 0.05, 0.1\}$. 
    
    \paragraph{Results.}
       We provide the detailed results of CIFAR-10-RL in Table~\ref{table:GN_in_CIFAR} and MNIST-RL in Table~\ref{table:GN_in_MNIST} for linear neural network and LN-Net.
       %
    

    \begin{table*}[h]
        \centering
            \caption{The performance of LN-Net with LN-G on CIFAR-10-RL. The rows of the table represent the model depth and the columns represent the group number of LN-G in the model. The percentage is the best accuracy of model under such setting. The bold number refers to the best accuracy among all group numbers under such depth.}
                    \label{table:GN_in_CIFAR}
        \begin{tabular}{lcccccccc}
        \hline
        CIFAR & 1                & 2       & 4       & 6        & 8        & 10       & 12       & 14       \\ \hline
        2        & 20.51\%          & 29.09\% & 52.17\% & 60.70\%  & 67.21\%  & 71.45\%  & 74.10\%  & 68.53\%  \\
        4        & 26.63\%          & 45.19\% & 72.41\% & 84.08\%  & 91.36\%  & 94.02\%  & 95.76\%  & 96.76\%  \\
        8        & 35.02\%          & 60.65\% & 91.74\% & 98.57\%  & 99.72\%  & 99.94\%  & 99.99\%  & 99.96\%  \\
        16       & 46.42\%          & 79.71\% & 99.58\% & 99.99\%  & \textbf{100.00\%} & \textbf{100.00\%} & \textbf{100.00\%} & \textbf{100.00\%} \\
        32 & 59.89\% & \textbf{93.67\%} & \textbf{99.96\%} & \textbf{100.00\%} & \textbf{100.00\%} & \textbf{100.00\%} & \textbf{100.00\%} & \textbf{100.00\%} \\
        64       & \textbf{69.40\%} & 91.62\% & 99.44\% & 99.66\%  & 96.58\%  & 88.20\%  & 77.22\%  & 44.48\%  \\
        128      & 26.48\%          & 14.66\% & 12.28\% & 10.38\%  & 10.23\%  & 10.26\%  & 10.37\%  & 10.22\%  \\ \hline
        best     & 69.40\%          & 93.67\% & 99.96\% & 100.00\% & 100.00\% & 100.00\% & 100.00\% & 100.00\% \\ \hline
        \end{tabular}
    \end{table*}

    \begin{table*}[h]

        \centering
        \caption{The performance of LN-Net with LN-G on MNIST-RL. The rows of the table represent the model depth and the columns represent the group number of LN-G in the model. The percentage is the best accuracy of model under such setting. The bold number refers to the best accuracy among all group numbers under such depth.}
               \label{table:GN_in_MNIST}
        \begin{tabular}{lcccccccc}
        \hline
        MNIST & 1       & 2       & 4       & 6       & 8       & 10      & 12      & 14      \\ \hline
        2        & 14.53\% & 18.25\% & 26.83\% & 27.76\% & 27.96\% & 27.56\% & 30.39\% & 30.81\% \\
        4        & 14.77\% & 20.98\% & 33.35\% & 40.67\% & 50.00\% & 53.52\% & 57.44\% & 58.78\% \\
        8        & 15.61\% & 25.38\% & 46.48\% & 64.51\% & 74.91\% & 81.34\% & 85.98\% & 89.97\% \\
        16       & 19.13\% & 32.43\% & 66.59\% & 86.20\% & 92.16\% & 94.03\% & 95.32\% & 95.25\% \\
        32 & 24.92\%          & 47.08\%          & \textbf{82.34\%} & \textbf{92.40\%} & \textbf{94.47\%} & \textbf{95.56\%} & \textbf{95.68\%} & \textbf{95.96\%} \\
        64 & \textbf{33.95\%} & \textbf{54.00\%} & 70.61\%          & 68.63\%          & 56.89\%          & 42.89\%          & 13.43\%          & 10.21\%          \\
        128      & 10.22\% & 10.17\% & 10.16\% & 10.22\% & 10.30\% & 10.25\% & 10.32\% & 10.31\% \\ \hline
        best     & 33.95\% & 54.00\% & 82.34\% & 92.40\% & 94.47\% & 95.56\% & 95.68\% & 95.96\% \\ \hline
        \end{tabular}
    \end{table*}

   \subsection{More Results of CNN without Activation Functions}
   \label{section:CNN-extension}
   As stated in Section~\ref{sec:CNN_EX}, we conduct more experiments on different networks, including the results on VGGs, and the 20-layer ResNet with the original configuration of channel number.
    
 \paragraph{Results on VGGs.}
 Following the experimental setup shown in Section ~\ref{sec:CNN_EX}, we also conduct experiments on CIFAR-10 classification using different normalization methods in the VGG-style networks (the network architecture used is ResNet-20, but with the residual connections removed.) with ReLU activation removed, where the group number $g$ ranging in \{2, 4, 8, 16, 32, 64\}. The experimental results of different normalization methods are shown in the Table ~\ref{table:vggs-na}. The results of different groups of GN and LN-G-Position are shown in the Figure ~\ref{fig:vgg}. We have the similar observations as in the ResNet-20 shown in the main paper. 
    \begin{figure}[h!]
        \begin{minipage}[c]{.43\linewidth}
            \subfigure[Training.]{
                \begin{minipage}[c]{.43\linewidth}
                    \centering
                    \includegraphics[width=4cm]{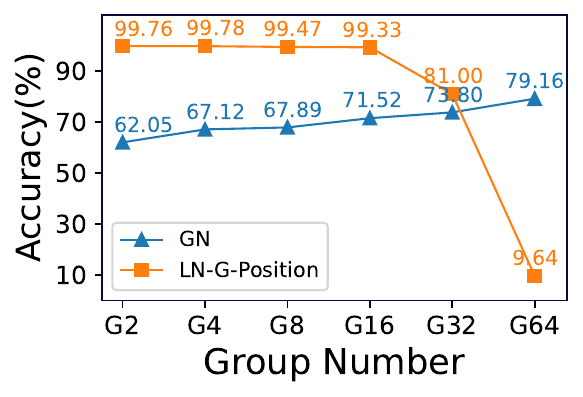}
                \end{minipage}
            }
            \hspace{0.15in}		
            \subfigure[Test.]{
                \begin{minipage}[c]{.43\linewidth}
                    \centering
                    \includegraphics[width=4cm]{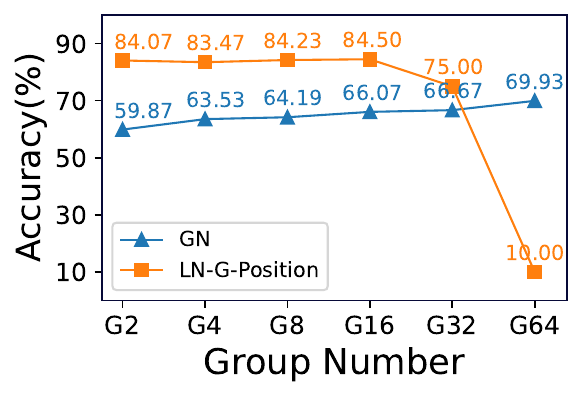}
                \end{minipage}
            }	
             \vspace{-0.12in}
            \caption{Results of the variants of LN-G (GN and LN-G-Position)  when using different group number. The experiments are conducted on CIFAR-10 dataset using a 20-layer VGG-style network without ReLU activation. We show (a) the training accuracy and (b) the test accuracy. In the x-axis, G2 refers to a group number of 2.}
            \label{fig:vgg}
        \end{minipage}
        \hspace{0.45in}
        \begin{minipage}[c]{.43\linewidth}
            \vspace{-0.25in}
            \subfigure[Training.]{
                \begin{minipage}[c]{.43\linewidth}
                    \centering
                    \includegraphics[width=4cm]{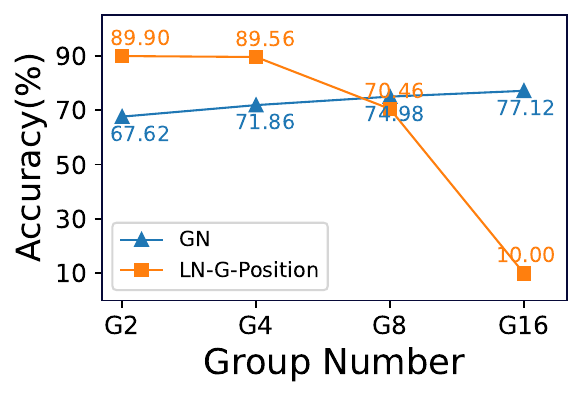}
                \end{minipage}
            }
            \hspace{0.15in}	
            \subfigure[Test.]{
                \begin{minipage}[c]{.43\linewidth}
                    \centering
                    \includegraphics[width=4cm]{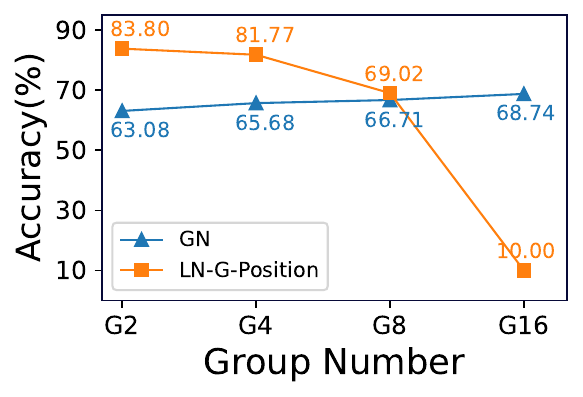}
                \end{minipage}
            }	
             \vspace{-0.12in}
            \caption{Results of the variants of LN-G (GN and LN-G-Position)  when using different group number. The experiments are conducted on CIFAR-10 dataset using ResNet-20-Original without ReLU activation. We show (a) the training accuracy and (b) the test accuracy.
            }
            \label{fig:resnet_ori}
        \end{minipage}
    \end{figure}
    \begin{table}[h!]
    	    \vspace{-0.12in}
        \centering
        \caption{Comparison of different normalization methods on CIFAR-10 using VGGs-NA (VGGs without ReLU activation).}
                \label{table:vggs-na}
        \begin{tabular}{c|c|c}
          \hline
          Normalization methods   & Train Acc(\%) & Test  Acc(\%) \\ \hline
          IN  & 9.76  & 10  \\
          BN  & 39.41 & 39.52 \\
          LN  & 51.51 & 51.06 \\ 
          GN  & 79.16 & 69.93 \\
          LN-G-Position & 99.33 & 84.5 \\ \hline
        \end{tabular}
    \end{table}
 \paragraph{Results on original ResNet-20}
  Following the experimental setup shown in Section ~\ref{sec:CNN_EX}, We also conduct experiments on the original ResNet-20-NA (with ReLU removed), where the group number $g$ ranging in \{2, 4, 8, 16\}. The experimental results of different normalization methods are shown in the Table~\ref{table:resnet20-ori}. The results of different groups of GN and LN-G-Position are shown in the Figure~\ref{fig:resnet_ori}. We have also the similar observations as in the ResNet-20 shown in the main paper. 
     \begin{table}[h!]
        \centering
        \caption{Comparison of different normalization methods on CIFAR-10 using ResNet-20-original-NA (the ResNet-20 using original configuration of channel numbers without ReLU activation).}
                \label{table:resnet20-ori}
        \begin{tabular}{c|c|c}
          \hline
          Normalization methods   & Train Acc(\%) & Test  Acc(\%) \\ \hline
          IN  & 10  & 10  \\
          BN  & 36.16 & 39.34 \\
          LN  & 61.12 & 58.69 \\ 
          GN  & 77.12 & 68.74 \\
          LN-G-Position & 89.9 & 83.8 \\ \hline
        \end{tabular}
    \end{table}

\end{document}

%% file: tool/symbol.tex

\usepackage{amsmath}
\usepackage{amsfonts}
\usepackage{bm}
\usepackage{esint}









\def\eqref#1{equation~\ref{#1}}









\def\1{\bm{1}}



\def\rx{{\textnormal{x}}}


\def\rvx{{\mathbf{x}}}





\def\vb{{\bm{b}}}

\def\vh{{\bm{h}}}

\def\vp{{\bm{p}}}
\def\vq{{\bm{q}}}

\def\vu{{\bm{u}}}
\def\vv{{\bm{v}}}
\def\vw{{\bm{w}}}

\def\vy{{\bm{y}}}
\def\vz{{\bm{z}}}



\def\mC{{\bm{C}}}

\def\mH{{\bm{H}}}
\def\mI{{\bm{I}}}

\def\mM{{\bm{M}}}
\def\mN{{\bm{N}}}
\def\mO{{\bm{O}}}
\def\mP{{\bm{P}}}
\def\mQ{{\bm{Q}}}

\def\mU{{\bm{U}}}
\def\mV{{\bm{V}}}
\def\mW{{\bm{W}}}
\def\mX{{\bm{X}}}
\def\mY{{\bm{Y}}}

\DeclareMathAlphabet{\mathsfit}{\encodingdefault}{\sfdefault}{m}{sl}
\SetMathAlphabet{\mathsfit}{bold}{\encodingdefault}{\sfdefault}{bx}{n}



\def\sB{{\mathbb{B}}}

\def\sD{{\mathbb{D}}}

\def\sI{{\mathbb{I}}}
\def\sJ{{\mathbb{J}}}

\def\sN{{\mathbb{N}}}

\def\sP{{\mathbb{P}}}

\def\sS{{\mathbb{S}}}








\newcommand{\R}{\mathbb{R}}






\def\ie{\emph{i.e.}}

\newcommand{\X}[1]{\mX_{#1}}










\newcommand{\Eqn}{Eqn.}




\newcommand{\eps}[0]{\varepsilon}
\newcommand{\pf}[2]{\frac{\partial #1}{\partial #2}}
\newcommand{\mat}[1]{\begin{bmatrix}#1\end{bmatrix}}

\newcommand{\xn}[1]{{x}_{#1}}
\newcommand{\hatxn}[1]{\hat{x}_{#1}}